\documentclass[sts]{imsart}

\RequirePackage{amsthm,amsmath,amsfonts,amssymb}
\RequirePackage[numbers,sort&compress]{natbib}
\RequirePackage[colorlinks,citecolor=blue,urlcolor=blue]{hyperref}
\RequirePackage{graphicx}

\usepackage{jcd}
\usepackage{conformal_style}
\startlocaldefs
\theoremstyle{plain}

\theoremstyle{remark}

\newcommand{\lebesgue}{\lambda}  

\newcommand{\Cjack}{\what{C}^{\textup{jk-mmx}}_{m,\alpha,\delta}}
\newcommand{\Csplit}{\what{C}^{\textup{split}}_{m,\alpha,\delta}}
\newcommand{\Chcp}{\what{C}^{\textup{hcp}}_{m,\alpha}}
\newcommand{\Chjk}{\what{C}^{\textup{hjk+}}_{m,\alpha}}
\newcommand{\Cresizedsplit}{\what{C}^{\textup{resized}}_{m,\alpha,\delta}}

\newcommand{\Bsplit}{B^{\textup{split}}_{n,\tau}}
\newcommand{\Bsplitn}{B^{\textup{split}}_{n,\tau(n)}}

\newcommand{\alg}{\mathsf{A}}
\newtheorem{intassumption}{Assumption}

\endlocaldefs

\begin{document}

\begin{frontmatter}
\title{Predictive Inference in Multi-environment Scenarios}
\begin{aug}
\author[A]{\fnms{John C. }~\snm{Duchi}\ead[label=e1]{jduchi@stanford.edu}},
\author[B]{\fnms{Suyash }~\snm{Gupta}\ead[label=e2]{suyash028@gmail.com}},
\author[C]{\fnms{Kuanhao}~\snm{Jiang}\ead[label=e3]{kuanhaojiang@g.harvard.edu}}
\and
\author[D]{\fnms{Pragya}~\snm{Sur}\ead[label=e4]{pragya@fas.harvard.edu}}

\address[A]{John C. Duchi is Associate Professor, Department of Statistics \& Department of Electrical Engineering, Stanford University, Stanford 94305, USA\printead[presep={\ }]{e1}.}

\address[B]{Suyash Gupta is Senior Applied Researcher, LinkedIn, Sunnyvale 94085, USA\printead[presep={\ }]{e2}.}

\address[C]{Kuanhao Jiang is Ph.D. Student, Department of Statistics, Harvard University, Cambridge 02138, USA\printead[presep={\ }]{e3}.}

\address[D]{Pragya Sur is Assitant Professor, Department of Statistics,
Harvard University, Cambridge 02138, USA\printead[presep={\ }]{e4}.}

\end{aug}

\begin{abstract}
We address the challenge of constructing valid confidence intervals and sets
in problems of prediction across multiple environments.
We investigate two types of coverage suitable for these problems, extending
the jackknife and split-conformal methods to show how to obtain
distribution-free coverage in such non-traditional, potentially hierarchical
data-generating scenarios.
We demonstrate a novel
resizing method to adapt to problem difficulty, which applies both to
existing approaches for predictive inference and the
methods we develop; this reduces prediction set sizes using limited
information from the test environment, a key to the methods' practical
performance, which we evaluate through neurochemical sensing and species
classification datasets.
Our contributions also include extensions for
settings with non-real-valued responses, a theory of consistency for
predictive inference in these general problems, and
insights on the limits of conditional coverage
\end{abstract}


\end{frontmatter}


\section{Introduction}
\label{sec:intro}

In the predictive inference problem, a statistician observes a training
sample $\{(X_i, Y_i)\}_{i = 1}^n$ of size $n$ and wishes to predict the
unknown value of $Y_{n+1}$ at a test point $X_{n + 1}$, where in the
classical setting, $\{(X_i,Y_i)\}_{i=1}^{n+1}$ are exchangeable random
variables.  Vovk et al.'s \emph{conformal
prediction}~\citep{VovkGaSh05} addresses this problem, even in finite sample
and distribution-free settings, constructing a prediction band $\what{C}$
such that $\what{C}(X_{n+1})$ covers $Y_{n+1}$ with a desired probability
on average over the draw $(X_{n + 1}, Y_{n + 1})$.

In contemporary problems, however, the statistician rarely observes data
from a single set of identically distributed training examples,
instead receiving data implicitly or explicitly arising from
multiple environments.
For instance, a neuroscientist investigating diseases
of the nervous system may use multiple electrodes to measure
neurotransmitter levels, with a goal to predict these levels at future time
points. Variations---whether in voltage potentials, experimental conditions,
build of the electrode, or otherwise---yield data from different electrodes
that follow distinct distributions~\cite{Loewinger856385}.
To understand the
neurobiological underpinnings of decision making, the statistician must
leverage information from multiple electrodes to develop a prediction
model that alleviates spurious electrode-to-electrode variations.
Additionally, even in
cases in which one tries to exactly replicate data generating methodology,
distribution shift means that prediction methods lose
substantial accuracy~\cite{RechtRoScSh19,TaoriDaShCaReSc20}.
We investigate and develop methodology for constructing prediction intervals
for such multi-environment problems.

\subsection{Problem Setting}

Data from multiple environments should improve predictions on a target only
if the training and test data share common characteristics.
We begin by considering a framework of hierarchical sampling
\cite{LoewingerPa22, LeeBaWi23, DunnWaRa23}, where one assumes that data
from different environments arise from a common hierarchical
model.
Let $\mc{P}$ be a set of distributions on $\mc{X}
\times \mc{Y}$, where $\mc{Y}$ is the outcome/response space and $\mc{X}$ is
the space of covariates, and $\mu$ denote a probability distribution on
$\mc{P} \times \N$.
Consider a sequence of exchangeable pairs generated from $\mu$,
\begin{equation}
  \label{eqn:exchangeable-generation}
  (P_{XY}^{1},n_1),\ldots,
  (P_{XY}^{m+1},n_{m+1}),
\end{equation}
where in training we observe i.i.d.\ samples
\begin{equation}
  \label{eqn:iid-within-environment}
  (X^i, Y^i) \defeq
  \{X_j^i,Y_j^i\}_{j=1}^{n_i},
  ~~~
  (X^i_j, Y_j^i) \simiid P_{XY}^i,
\end{equation}
for $j \in [n_i]$ and
$i \in [m]$, treating $P^i_{XY}$ as the $i$th environment.  In the test,
we observe an i.i.d.\ sample $\{X_j^{m+1}\}_{j=1}^{n_{m+1}}$ generated from
the marginal distribution on $X$ according to the $(m + 1)$st
$P_{XY}^{m+1}$, and we wish to construct confidence sets for the unknown
responses $\{Y_j^{m+1}\}_{j=1}^{n_{m+1}}$ with valid coverage.
We depart
from the traditional predictive inference literature in the sense that
individual observations need not be exchangeable.
Instead, hierarchical sampling assumes that
within each environment, the observations are
exchangeable, and the environments themselves are exchangeable as well.
This
weaker assumption allows addressing scenarios where the data may exhibit
variations across different environments, but defining valid
coverage requires careful consideration.

We consider two coverage notions
for multi-environment settings.
The first considers properties close to those conformal prediction
provides~\cite{VovkGaSh05, AngelopoulosBa23}, relying on the
i.i.d.\ sampling from $P^i$ in the
model~\eqref{eqn:exchangeable-generation}--\eqref{eqn:iid-within-environment}:
we seek $\what{C}$ covering a single new example with a
prescribed probability.
\begin{definition}
  \label{definition:marginal-new-coverage}
  A confidence set mapping $\what{C} : \mc{X} \toto \mc{Y}$
  provides $1 - \alpha$ hierarchical coverage in
  the setting~\eqref{eqn:exchangeable-generation} if
  for the single observation $(X_1^{m + 1}, Y_1^{m + 1}) \sim P^{m+1}_{XY}$,
  \begin{equation}
    \label{eqn:marginal-new-coverage}
    \P\left(Y_1^{m + 1} \in \what{C}(X_1^{m + 1})\right)
    \ge 1 - \alpha.
  \end{equation}
\end{definition}
\noindent
Dunn et al.~\citep{DunnWaRa23} and Lee et al.~\cite{LeeBaWi23} both adopt the
guarantee~\eqref{eqn:marginal-new-coverage} as a notion of coverage.

Instead of this marginal guarantee over a single observation from the new
environment $m + 1$ and the i.i.d.\ within environment
setting~\eqref{eqn:iid-within-environment}---and given that we expect to
collect multiple observations from each environment---it is also natural to
consider coverage notions over entire new samples.
Thus, we consider data generation where only the full samples are
exchangeable: let $\mu$ be a probability distribution on $\environments \times
\N$, where $\environments$ is a collection of environments; then we receive
exchangeable samples
\begin{equation}
  \label{eqn:exchangeable-arbitrary-sample}
  (\env_i, n_i) \sim \mu
  ~~
  \mbox{and} ~~
  (X^i, Y^i) \defeq \{X_j^i, Y_j^i\}_{j = 1}^{n_i},
\end{equation}
where $(X^i, Y^i)$ is generated
via a Markov kernel from $\mc{E} \times \N$ to $(\mc{X} \times \mc{Y})^{n_i}$.
This makes no assumptions about the relationships between individual
observations $(X_j^i, Y_j^i)$, so we consider a stronger coverage property:
\begin{definition}
  \label{definition:valid-coverage}
  A confidence set mapping $\what{C} : \mc{X} \toto \mc{Y}$ provides
  distribution-free level $(\alpha, \delta)$-coverage in
  the setting~\eqref{eqn:exchangeable-arbitrary-sample} if
  \begin{equation}
    \label{eqn:coverage-def}
    \P\bigg(\frac{1}{n_{m+1}}
    \sum_{j = 1}^{n_{m+1}} \indic{Y_j^{m + 1} \in \what{C}(X_j^{m + 1})}
    \ge 1 - \alpha \bigg) \ge 1 - \delta.
  \end{equation}
\end{definition}
\noindent
That is, $\what{C}$ covers a
$1 - \alpha$ fraction of observed examples in the new environment
with probability at least $1 - \delta$.

\subsection{Main Contributions}

Our main contributions include the following.
\begin{enumerate}[1.]
\item We introduce the multi-environment jackknife and split conformal
  methods (Algorithms \ref{alg:multi-env} and \ref{alg:multi-env-split},
  respectively). Theorems \ref{theorem:main-coverage}
  and \ref{theorem:main-coverage-split} establish that these provide
  distribution-free level $(\alpha, \delta)$-coverage~\eqref{eqn:coverage-def} when
  the response space $\mc{Y} = \mathbb{R}$.
  These algorithms
  straightforwardly extend
  (see Algorithms \ref{alg:nested-multi-env} and
  \ref{alg:nested-multi-env-split}) to general response spaces $\mc{Y}$,
  and they continue to provide
  $(\alpha, \delta)$-coverage (see Theorems
  \ref{theorem:nested-coverage} and \ref{theorem:nested-coverage-split}
  in Sec.~\ref{sec:general-confidence}).
\item We investigate the behavior of algorithms targeting $(\alpha,
  \delta)$-coverage~\eqref{eqn:coverage-def} alongside previous work on
  hierarchical conformal inference~\cite{LeeBaWi23}, evaluating methods on
  neurochemical sensing~\cite{Loewinger856385} and species
  classification~\cite{KohSaMaXiZhBaHuYaPhBeLeKuPiLeFiLi21,beery2020iwildcam}
  data.
  Our experiments (Sec.~\ref{sec:real_data})
  reveal that the multi-environment jackknife tends to
  yield smaller confidence sets than the split-conformal methodology when
  there are fewer training environments. Conversely, the multi-environment
  split conformal method demonstrates better performance when the number of
  training environments is large.
\item The experiments indicate that multi-environment algorithms can be
  conservative, and to mitigate this, we propose and evaluate
  a resizing method
  (Algorithm \ref{alg:resized-nested-multi-env-split} in
  Sec.~\ref{sec:resizing_algos}) to reduce the size of prediction sets given
  access to a limited amount of information from the test environment.
  This strategy appears to be useful for predictive inference more
  generally, including to hierarchical conformal prediction
  (HCP)~\cite{LeeBaWi23}, where it yields notable set size reductions.
\item We develop new consistency theory (see Theorems
  \ref{theorem:general-consistency} and
  \ref{theorem:general-consistency-split} in
  Sec.~\ref{sec:consistency-results}) for multi-environment predictive
  inference, showing how jackknife and split
  conformal methods produce consistent confidence sets.
\item We also touch briefly (see
  Section~\ref{sec:environment-conditional-coverage}) on
  environment-conditional coverage; in analogy to standard single
  environment settings~\cite{Vovk12, BarberCaRaTi21a}, we show that
  conditional coverage is impossible, extending these ideas to reflect that
  environments may be discrete. Approximate environment-conditional
  coverage can, however, hold.
\end{enumerate}

\subsection{Related Work}

Standard predictive inference methods include
split-conformal~\cite{VovkGaSh05, LeiWa14, CauchoisGuDu21, RomanoSeCa20, CauchoisGuAlDu20, CauchoisGuAlDu22, Gupta22} and
modified jackknife procedures~\cite{BarberCaRaTi21}.  The current paper
extends these to multi-environment problems. Split conformal prediction
separates the data into a training and a calibration set, using the training
data to fit a model and the calibration data to set a threshold for
constructing prediction intervals. Since it only splits the data once, the
method may sacrifice statistical efficiency for computational
gains. Addressing this issue, jackknife-style procedures use all available
data for training and calibration by fitting leave-one-out models,
increasing computational cost for accuracy. Both methods require
exchangeability of the entire observed data to ensure valid coverage, while
multi-environment methods work under the weaker assumption that
within (but not across) environments, observations are exchangeable,
and the environments themselves are exchangeable. 

Recognizing the challenges inherent in collecting data, implicitly or
explicitly, across multiple environments, a recent literature considers
conformal prediction under hierarchical models, assuming the
multi-environment
setting~\eqref{eqn:exchangeable-generation}--\eqref{eqn:iid-within-environment}.
Among these, Dunn et al.~\citep{DunnWaRa23} and Lee et al.~\citep{LeeBaWi23}
study conformal prediction under hierarchical sampling and propose methods
satisfying the marginal coverage
guarantee~\eqref{eqn:marginal-new-coverage}, as well as a few other
distribution-free guarantees, which may be a satisfying coverage guarantee
for many applications.
The $(\alpha, \delta)$-coverage condition~\eqref{eqn:coverage-def} requires
coverage for multiple examples in the test environment simultaneously
without assuming i.i.d.\ sampling within each environment, and
it is unclear if existing multi-level conformal approaches satisfy it.


\section{Methods for regression}
\label{sec:basic-method}

To introduce our basic methods, we assume the target space $\mc{Y} = \R$.
In
this case, we wish to return a confidence set $C : \mc{X} \toto \R$, where
typically $C(x)$ is an interval.  We define a fitting algorithm $\alg$ to be
a function that takes a collection of samples as input and outputs an
element of $\mc{F} \subset \mc{X} \to \mc{Y}$.
To describe our algorithms formally, we introduce two quantile-type
mappings, where for $v \in \R^n$ we let $v_{(1)} \le
v_{(2)} \le \cdots \le v_{(n)}$ be its order statistics.

\begin{definition}[Quantile mappings]
  \label{def:quantiles}
  For $v \in \R^n$, 
  \begin{align*}
    \quantplus(\{v_i\}) & \defeq
    v_{(\ceil{(1 - \alpha)(n + 1)})} \\
    \quantminus(\{v_i\}) & \defeq v_{(\floor{\alpha(n + 1)})}
    = -\quantplus(\{-v_i\}).
  \end{align*}
\end{definition}
\noindent
That is, $\quantplus$ gives the $\ceil{(1 - \alpha)(n+1)}$th smallest
value of its argument, and similarly for $\quantminus$.

\subsection{Multi-environment Jackknife-minmax}

We first introduce a multi-environment version of
Barber et al.'s jackknife-minmax~\cite{BarberCaRaTi21}. The
idea is simple: we repeatedly fit a predictor $\what{f}_{-i}$ to all
environments \emph{except} environment $i$, then evaluate residuals on
environment $i$ to gauge the variability of predicting on a new
environment.
We define $\what{f} = \alg((X^1, Y^1), \ldots, (X^m, Y^m))$ to
be the predictor we would fit given each sample
$(X^i, Y^i) = \{(X_j^i, Y_j^i)\}_{j = 1}^{n_i}$,
and consider the leave-one-out predictors
\begin{equation}
  \label{eqn:loo-predictor}
  \what{f}_{-i}
  \defeq \alg \left(\left\{(X^k, Y^k)\right\}_{k \neq i, k \le m}
  \right) 
\end{equation}
From these, we construct the leave-one-out residuals
for each example $j = 1, \ldots, n_i$ in environment $i$,
letting
\begin{equation*}
  R^i_j = |Y^i_j -\what{f}_{-i}(X^i_j)|
  ~~ \mbox{for} ~ j \in [n_i],
  ~ i \in [m].
\end{equation*}

We then pursue a blocked confidence set construction. Within
each environment, we let $S^i_{1 - \alpha}
= \quantplus[n_i,\alpha](\{R_j^i\}_{j=1}^{n_i})$ be the
$1 - \alpha$ quantile of the residuals for predicting
in environment $i$ using $\what{f}_{-i}$. To obtain
an interval that is likely to cover, we use quantiles
of these residuals \emph{across} environments, as the environments
are exchangeable. We therefore construct intervals of the form
\begin{align*}
  C(x) \defeq & \left[f_{\textup{low}}(x) - \quantplus[m,\delta](\{S_{1-\alpha}^i\}_i), \right. \\
    & \left. \qquad
    f_{\textup{high}}(x) + \quantplus[m,\delta](\{S_{1-\alpha}^i\}_i) \right],
\end{align*}
where it remains to choose $f_{\textup{low}}$ and
$f_{\textup{high}}$ to obtain valid coverage.
Algorithm~\ref{alg:multi-env} achieves this,
using the minimum and maximum values of the held-out predictions
to construct the confidence set.

\begin{center}
  \algbox{
    \label{alg:multi-env}
    \textbf{Multi-environment Jackknife-minmax:} the regression case
  }{
    \textbf{Input:} samples $\{X^i_j, Y^i_j\}_{j = 1}^{n_i}$,
    $i = 1, \ldots, m$, confidence levels $\alpha, \delta$ \\
    
    \textbf{For} $i = 1, \ldots, m$,

    \hspace{1em} \textbf{set}
    $\what{f}_{-i}$ to the leave-one-out predictor~\eqref{eqn:loo-predictor}

    \hspace{1em} \textbf{construct} residuals
    \begin{equation*}
      R^i_j = |Y_j^i - \what{f}_{-i}(X_j^i)|,
      ~~ j = 1, \ldots, n_i,
    \end{equation*}
    \hspace{1em} and residual quantiles
    \begin{equation*}
      S^i_{1 - \alpha}
      = \quantplus[n_i,\alpha] \left(R_1^i, R_2^i,
      \ldots, R_{n_i}^i\right).
    \end{equation*}
    
    \textbf{Return} confidence interval mapping
    \begin{align*}
      \Cjack(x) \defeq
      &\Big[\min_{i \in [m]} \what{f}_{-i}(x) - \quantplus[m,\delta]\left(
        \{S^i_{1-\alpha}\}_{i=1}^m\right),\\
       & \qquad \max_{i \in [m]} \what{f}_{-i}(x) + \quantplus[m,\delta]\left(
        \{S^i_{1-\alpha}\}_{i=1}^m\right)
        \Big].
    \end{align*}
  }
\end{center}

\begin{theorem}
  \label{theorem:main-coverage}
  The multi-environment confidence mapping
  $\Cjack$
  Algorithm~\ref{alg:multi-env} returns
  provides level $(\alpha, \delta)$-coverage~\eqref{eqn:coverage-def}.
\end{theorem}
\noindent
See Appendix~\ref{sec:proof-main-coverage} for the proof. In words, with
probability at least $1 - \delta$, the prediction intervals from Algorithm
\ref{alg:multi-env} cover at least $(1 - \alpha) \times 100\%$ of the
examples in the test environment.

\begin{remark}
  In Algorithm \ref{alg:multi-env}, it may appear that taking the minimum
  and maximum of the held-out predictions could yield conservative
  prediction intervals; intuitively, suitably corrected $\delta$ and
  $(1-\delta)$ quantiles (as in the jackknife+~\cite{BarberCaRaTi21})
  of $\widehat{f}_{-i}(x) \pm S^{i}_{1-\alpha}$
  should yield a confidence set satisfying the
  guarantee~\eqref{eqn:coverage-def}.
  The arbitrariness in the process~\eqref{eqn:exchangeable-arbitrary-sample}
  and strength of guarantee~\eqref{eqn:coverage-def} appear to preclude
  this, and experimentally, this approach fails to provide $(\alpha,
  \delta)$-coverage; see Appendix~\ref{app:jkquantile}.
\end{remark}

\subsection{A multi-environment split conformal method}

We also introduce a multi-environment version of split conformal
inference. Our algorithm partitions the environment index set $\{1,\hdots,m
\}$ into subsets $D_1$ and $D_2$ uniformly.
We use the data in environments
indexed by $D_1$ to fit a model $\widehat{f}_{D_1}=\alg(\{X^j,Y^j \}_{j \in
  D_1})$ with which we construct residuals for each example $j =
1,\hdots, n_i$ in each environment $i \in D_2$. Then for each environment in
$D_2$, we construct the $(1-\alpha)$-th quantile of its $n_i$
residuals. This yields a set $\{\scorerv^i_{1-\alpha}\}_{i \in D_2}$ of quantiles.
To obtain a likely-to-cover interval, we consider quantiles of
these quantiles, yielding algorithm~\ref{alg:multi-env-split} below.

\begin{center}
  \algbox{
  \label{alg:multi-env-split}
    Multi-environment Split Conformal Inference: the regression case
  }{
    \textbf{Input:} samples $\{X^i_j, Y^i_j\}_{j = 1}^{n_i}$,
    $i \in [m]$, confidence levels $\alpha, \delta$, split ratio $\gamma$ \\

    Randomly partition $[m]$ into $D_1$ and $D_2$ with $ \frac{|D_1|}{m} = \gamma$.

    \textbf{set} $$ \widehat{f}_{D_1}=\alg(\{X^i,Y^i \}_{i \in D_1}). $$

    \textbf{For} $i \in D_2$, construct residuals
    \begin{equation*}
      R_j^i=\left|Y_j^i-\widehat{f}_{D_1}\left(X_j^i\right)\right|, \quad j=1, \ldots, n_i, 
    \end{equation*}
    \hspace{1em} and quantiles
    \begin{equation*}
    S_{1-\alpha}^i=\widehat{q}_{n_i, \alpha}^{+}\left(R_1^i, R_2^i, \ldots, R_{n_i}^i\right).
    \end{equation*}
    \textbf{return} confidence interval mapping 
    \begin{align*}
      \Csplit(x)
      & \defeq
      \left[ \widehat{f}_{D_1}(x)-\widehat{q}_{m, \delta}^{+}\left(\left\{S_{1-\alpha}^i\right\}_{i \in D_2}\right), \right.\\
        &\left. \qquad\qquad
        \what{f}_{D_1}(x)+\widehat{q}_{m, \delta}^{+}\left(\left\{S_{1-\alpha}^i\right\}_{i \in D_2}\right)\right] .
    \end{align*}
  }
\end{center}

\begin{theorem}
  \label{theorem:main-coverage-split}
  The multi-environment confidence mapping
  $\Csplit$
  Algorithm~\ref{alg:multi-env-split}
  returns provides level $(\alpha, \delta)$-coverage~\eqref{eqn:coverage-def}.
  If the observations $Y_j^i$ are a.s.\ distinct,
  \begin{align*}
    & \P\bigg[\sum_{j=1}^{n_{m+1}}
      \indic{Y_j^{m+1} \in \Csplit\left(X_j^{m+1}\right)} \geq \\
      & \qquad \ceil{(1-\alpha)(n_{m+1}+1)}
    \bigg] \leq 1 - \delta + \frac{1}{m(1- \gamma)+1}.
  \end{align*}
\end{theorem}
\noindent
See Appendix~\ref{sec:proof-main-coverage-split} for the proof.

Both algorithms provide valid coverage, as
Theorems~\ref{theorem:main-coverage} and \ref{theorem:main-coverage-split}
demonstrate.
Multi-environment split conformal fits
the model once, and is therefore computationally attractive.
We expect this method to be less conservative as it takes no maxima or
minima over predictions coming from multiple models.
In contrast, the jackknife+ uses (almost) the entire sample set to fit the
model and construct residual quantiles.
Thus, we expect it to outperform Alg.~\ref{alg:multi-env-split} when there
are few training environments.
We investigate these points via real data experiments in Section
\ref{sec:real_data}.

\subsection{Methods to achieve marginal coverage
  in hierarchical predictive inference}

As we discuss earlier, Dunn et al. and Lee et al. provide predictive
inference methods for the hierarchical (multi-environment)
setting~\eqref{eqn:exchangeable-generation}, focusing on guarantees that
provide coverage for a single new
observation~\eqref{eqn:marginal-new-coverage}, as in
Definition~\ref{definition:marginal-new-coverage}.  Focusing on the more
recent paper~\cite{LeeBaWi23}, we review (with a correction to avoid an
accidental infinite quantile~\cite[App.~A.2.1]{LeeBaWi23}) their
hierarchical jackknife+ and hierarchical conformal prediction algorithms. In
the procedures, we let $\pointmass_z$ denote a point mass at $z$, and for a
distribution $P$ on $\R$, define the left quantile mapping $Q_\alpha(P)
\defeq \inf\{t \mid P(Z \le t) \ge \alpha\}$ and (non-standard) right
quantile $Q_\alpha^{\textup{r}}(P) \defeq \sup\{t \mid P(Z \le t) < \alpha\}$.

\begin{center}
  \algbox{
    \label{alg:hcp_jackknife+}
    Hierarchical Jackknife+~\cite{LeeBaWi23}
  }{
    \textbf{Input:} samples $\{X^i_j, Y^i_j\}_{j = 1}^{n_i}$,
    $i \in [m]$, level $\alpha$\\

    \textbf{For} $i = 1, \ldots, m$,

    \hspace{1em}\textbf{set}
    $\what{f}_{-i}$ to the leave-one-out predictor~\eqref{eqn:loo-predictor}
    
    \hspace{1em}\textbf{construct} residuals
    \begin{equation*}
      R^i_j = |Y_j^i - \what{f}_{-i}(X_j^i)|,
      ~~ j = 1, \ldots, n_i.
    \end{equation*}
    
    \textbf{Return} confidence interval mapping
    \begin{equation*}
      \Chjk(x)\defeq
      \left[ \text{low}(x),  \text{high}(x)\right],
    \end{equation*}
    where
    \begin{align*}
      \textup{low}(x)
      & \defeq
      Q_\alpha^{\textup{r}}
      \bigg(\sum_{i = 1}^m \sum_{j = 1}^{n_i}
      \frac{\pointmass_{\what{f}_{-i}(x) - R_j^i}}{
        (m + 1) n_i}
      + \frac{\pointmass_{-\infty}}{m+1}\bigg), \\
      \textup{high}(x)
      & \defeq
      Q_\alpha
      \bigg(\sum_{i = 1}^m \sum_{j = 1}^{n_i}
      \frac{\pointmass_{\what{f}_{-i}(x) + R_j^i}}{
        (m + 1) n_i}
      + \frac{\pointmass_{-\infty}}{m+1}\bigg).
    \end{align*}
  }
\end{center}

\begin{center}
  \algbox{
  \label{alg:hcp_split}
  Hierarchical Conformal Prediction~\cite{LeeBaWi23}}{
    \textbf{Input:} samples $\{X^i_j,
    Y^i_j\}_{j = 1}^{n_i}$, $i \in [m]$, confidence level $\alpha$, split
    ratio $\gamma$\\

    Randomly partition $[m]$ into $D_1$ and $D_2$ with $ \frac{|D_1|}{m} = \gamma$.
    
    \textbf{set} $\widehat{f}_{D_1}=\alg(\{X^i,Y^i \}_{i \in D_1})$.

    \textbf{For} $i \in D_2$, construct residual quantiles
    \begin{equation*}
      R_j^i=\left|Y_j^i-\widehat{f}_{D_1}\left(X_j^i\right)\right|, \quad j=1, \ldots, n_i
    \end{equation*}

    \textbf{set}
    \begin{align*}
      T = Q_{1-\alpha}\bigg(\sum_{i=m \gamma + 1}^{m}
      \sum_{j=1}^{n_k} \frac{\pointmass_{R_j^i}}{\left(|D_1|+1\right) n_i}
      + \frac{\pointmass_{+\infty}}{|D_1|+1}\bigg).
    \end{align*}
    
    \textbf{return} confidence interval mapping $$ \Chcp(x):=\left[ \widehat{f}_{D_1}(x)-T,  \widehat{f}_{D_1}(x)+T\right] .$$
  }
\end{center}

These algorithms construct prediction intervals for a single observation
$(X_1^{m+1}, Y_1^{m+1})$ in the test environment $m + 1$,
guaranteeing marginal coverage~\eqref{eqn:marginal-new-coverage}:
\begin{corollary}[Lee et al. \cite{LeeBaWi23}, Theorems 1 and 5]
  The jackknife+ mapping $\Chjk$ that
  Algorithm~\ref{alg:hcp_jackknife+} returns provides
  $1 - 2 \alpha$ hierarchical coverage,
  and the conformal mapping $\Chcp$ that
  Algorithm~\ref{alg:hcp_split} returns provides
  $1 - \alpha$ hierarchical coverage~\eqref{eqn:marginal-new-coverage}.
\end{corollary}

These results are not completely comparable to $(\alpha, \delta)$-coverage guarantees
(Definition~\ref{definition:valid-coverage}).  We do so somewhat
heuristically in our experiments in Section~\ref{sec:compare_hcp}, where for
values of $\alpha \in (0, 1)$---the coverage guarantee within an
environment---we may vary $\delta$ to compare performance of the methods.
Previewing our results, it appears that both the hierarchical jackknife+ and
split-conformal methods generate prediction sets with comparable size and
coverage properties to the multi-environment methods in
Algorithms~\ref{alg:multi-env} and~\ref{alg:multi-env-split}.


\section{General confidence sets and extensions}
\label{sec:general-confidence}

To this point, we have described our algorithms for real-valued predictions,
where confidence intervals $C(x) = [a, b]$ are most practicable. Here, we
generalize the algorithms beyond regression, where the
target space may not be the real line and the prediction sets may be
asymmetric. We first present the general formulation and abstract algorithms.
Subsequently, we specialize our construction to
demonstrate implementation in a few cases of interest: (i) when we
represent general target spaces $\mc{Y}$ and confidence sets by labels $y$
that suffer small loss under a prediction $f(x)$, i.e., $\{y \in \mc{Y} \mid
\loss(y, f(x)) \le \tau\}$, and (ii) for quantile regression-type
approaches, which allow asymmetric confidence sets in regression
problems~\cite{RomanoPaCa19}.

\subsection{General nested confidence sets}

We begin with our most general formulation.
Here, we treat confidence sets
themselves as the objects of interest (adopting the
interpretation~\cite{GuptaKuRa22}), rather than any particular prediction
method $\what{f}$, and assume that confidence sets are indexed by a
threshold $\tau$ and nested in that
\begin{equation*}
  C_\tau(x) \subset C_{\tau + \delta}(x) ~~ \mbox{for~all~} \delta \ge 0.
\end{equation*}
We also require the sets are
right-continuous (which is tacitly implicit in the
work~\cite{GuptaKuRa22}), so that
\begin{equation*}
  C_\tau(x) = \bigcap_{\delta > 0} C_{\tau + \delta}(x)
  ~~ \mbox{for~all~} \tau.
\end{equation*}
We assume now that the algorithm $\alg$ returns a collection of
confidence
set mappings $\{\what{C}_\tau\}_{\tau \in \R}$, where
each $\what{C}_\tau : \mc{X} \toto \mc{Y}$ is a set-valued function.
To see how this generalizes the initial Algorithm~\ref{alg:multi-env},
note that we may write
\begin{equation*}
\begin{aligned}
  \widehat{C}_\tau(x) &= \left[f(x) - \tau, f(x) + \tau\right] \\
  \text{or} \quad \widehat{C}_\tau(x) &= \left[f_{\text{low}}(x) - \tau, f_{\text{high}}(x) + \tau\right].
\end{aligned}
\end{equation*}
Assuming $\alg$ can perform this calculation, the immediate
extension of Algorithm~\ref{alg:multi-env} follows.

\begin{center}
  \algbox{
    \label{alg:nested-multi-env}
    Multi-environment Jackknife-minmax via nested
    confidence sets
  }{
    \textbf{Input:} samples $\{X^i_j, Y^i_j\}_{j = 1}^{n_i}$,
    $i \in [m]$, levels $\alpha, \delta$,
    predictive set algorithm $\alg$ \\
    
    \textbf{For} $i = 1, \ldots, m$, \textbf{set}
    \begin{align*}
      \{\what{C}^{-i}_\tau\}_{\tau \in \R}
      = \alg\left(\{X^k, Y^k\}_{k \neq i, k \le m}\right)
    \end{align*}
    \hspace{1em} and construct residuals
    \begin{equation*}
      R^i_j = \inf \left\{\tau \mid
      Y_j^i \in \what{C}^{-i}_\tau(X_j^i) \right\},
      ~~ j = 1, \ldots, n_i,
    \end{equation*}
    \hspace{1em}
    and quantiles
    \begin{equation*}
     \scorerv^i_{1 - \alpha}
      = \quantplus[\alpha] \left(R_1^i, R_2^i,
      \ldots, R_{n_i}^i\right).
    \end{equation*}
    \textbf{return} confidence set mapping
    \begin{equation*}
      \Cjack(x) \defeq
      \bigcup_{i \in [m]} \what{C}_{\what{\tau}}^{-i}(x)
      ~~ \mbox{for} ~~
      \what{\tau} = \quantplus[\delta]\left(\left\{S^i_{1-\alpha}\right\}_{i=1}^m
      \right).
    \end{equation*}
  }
\end{center}

\begin{theorem}
  \label{theorem:nested-coverage}
  The multi-environment confidence mapping $\Cjack$ 
  Algorithm~\ref{alg:nested-multi-env} returns
  provides level $(\alpha, \delta)$-coverage~\eqref{eqn:coverage-def}.
\end{theorem}
\noindent
The proof of Theorem~\ref{theorem:nested-coverage} mimics
that of Theorem~\ref{theorem:main-coverage}, and we
present it in Section~\ref{sec:proof-nested-coverage}.

Similarly, the extension of Algorithm \ref{alg:multi-env-split} follows.

\begin{center}
  \algbox{
  \label{alg:nested-multi-env-split}
    Multi-environment Split Conformal via nested confidence sets
  }{
    \textbf{Input:} samples $\{X^i_j, Y^i_j\}_{j = 1}^{n_i}$,
    $i \in [m]$, confidence levels $\alpha, \delta$, split ratio $\gamma$, predictive set algorithm $\alg$ \\

    Randomly partition $[m]$ into $D_1$ and $ D_2$
    with $\frac{|D_1|}{m} = \gamma$.\\

    \textbf{set} $\{\what{C}^{D_1}_\tau\}_{\tau \in \mathbb{R}} = \alg( \{(X^{i}, Y^{i})\}_{i \in D_1} )$.\\

    \textbf{For} $i \in D_2$, construct residuals
    \begin{equation*}
      R_j^i=\inf \left\{\tau \mid Y_j^i \in \widehat{C}^{D_1}_\tau\left(X_j^i\right)\right\}, \quad j=1, \ldots, n_i, 
    \end{equation*}
    \hspace{1em}
    and quantiles
    \begin{equation*}
      \scorerv_{1-\alpha}^i=
      \quantplus[\alpha]\left(R_1^i, R_2^i, \ldots, R_{n_i}^i\right).
     \end{equation*}
    \textbf{Return} confidence set mapping
    \begin{equation*}
      \Csplit(x) \defeq \what{C}^{D_1}_{\what{\tau}}(x)
      ~~ \mbox{for} ~~
      \what{\tau} = \quantplus[\delta]\left(\left\{S_{1-\alpha}^i\right\}_{i \in D_2}
      \right).
    \end{equation*}
  }
\end{center}

\begin{theorem}
  \label{theorem:nested-coverage-split}
  The multi-environment confidence mapping
  $\Csplit$
  Algorithm~\ref{alg:nested-multi-env-split} returns
  provides level $(\alpha, \delta)$-coverage~\eqref{eqn:coverage-def}.
  If the scores $S_{1-\alpha}^i$ are a.s.\ distinct, then
  \begin{align*}
    & \P\bigg[\sum_{j=1}^{n_{m+1}} 1\left\{Y_j^{m+1} \in \Csplit\left(X_j^{m+1}\right)\right\} \geq\\
      & \qquad \left\lceil(1-\alpha)\left(n_{m+1}+1\right)\right\rceil\bigg] \leq 1 - \delta + \frac{1}{m(1- \gamma)+1}. 
   \end{align*}
  
\end{theorem}
\noindent
See Appendix~\ref{sec:proof-nested-coverage-split} for the proof.

\subsection{Specializations and
  examples of the nested confidence set approach}

We specialize the general nested prediction set
Algorithms~\ref{alg:nested-multi-env} and~\ref{alg:nested-multi-env-split}
to a few special cases where implementation is direct and natural.

\subsubsection{General loss functions and targets}
\label{sec:general-losses}

In extension to the preceding section, we consider the following
scenario:
we have targets $y \in \mc{Y}$, covariates $x \in \mc{X}$, and
prediction functions
$f \in \mc{F} \subset \mc{X} \to \R^\preddim$. Then
for a loss $\loss : \mc{Y} \times \R^\preddim \to \R_+$,
we consider predictive sets of the form
\begin{equation*}
  C_\tau(x) = \left\{y \in \mc{Y} \mid \loss(y, f(x)) \le \tau \right\}
\end{equation*}
where, for now, $f \in \mc{F}$ and $\tau \in \R$ are left implicit.
These are nested, allowing application of
Theorem~\ref{theorem:nested-coverage} and
Algorithm~\ref{alg:nested-multi-env}.
A slight specialization allows easier presentation:
define the residual losses on environment $i$ by
\begin{equation*}
  R\sups{i}_j
  \defeq \loss\left(Y\sups{i}_j, \what{f}_{-i}(X\sups{i}_j)\right),
\end{equation*}
where as previously
$\what{f}_{-i}$ is the leave-one-out predictor
$\what{f}_{-i} = \alg((X^k, Y^k)_{k \neq i})$.
Setting $S\sups{i}_{1 - \alpha} =
\quantplus[\alpha](\{R\sups{i}_j\}_j)$,
the nested union in Algorithm~\ref{alg:nested-multi-env}
is exactly
 \begin{align*}
 & \Cjack(x)
  \defeq\\
  & \left\{y \in \mc{Y} \mid \min_{k \le m}
  \loss(y, \what{f}_{-k}(x)) \le
  \quantplus[\delta]\left(\{S\sups{i}_{1-\alpha}\}_{i=1}^m\right)\right\}.
\end{align*}

\begin{corollary}
  The loss-based set $\Cjack$ provides $(\alpha, \delta)$-coverage~\eqref{eqn:coverage-def}.
\end{corollary}

\subsubsection{Quantile regression}
\label{sec:quantile}

Romano et al. \cite{RomanoPaCa19} highlight how moving beyond symmetric confidence sets
to use quantile-based regressoin functions allows more accurate and tighter
confidence bands even for $\R$-valued responses $Y$.
Algorithm~\ref{alg:nested-multi-env} and
Theorem~\ref{theorem:nested-coverage} let us adapt their technique to obtain
quantile-type confidence sets in multi-environment settings.  Imagine we
have two algorithms fitting lower and upper predictors
\begin{equation*}
  \what{l}_{-i} = \alg\low\left((X^k, Y^k)_{k \neq i}\right),
  ~~~
  \what{u}_{-i} = \alg\high\left((X^k, Y^k)_{k \neq i}\right),
\end{equation*}
where we leverage the idea that the methods
target that $Y$ lies in $[\what{l}_{-i}(x),
  \what{u}_{-i}(x)]$ with a prescribed probability $1 - \alpha$.  To
construct nested confidence sets from $\what{l}, \what{u}$, we set
\begin{equation*}
  \what{C}^{-i}_\tau(x) = \left[\what{l}_{-i}(x) - \tau,
    \what{u}_{-i}(x) + \tau \right].
\end{equation*}
Specializing the generic construction
in Algorithm~\ref{alg:nested-multi-env} to this case,
set the residuals
\begin{equation*}
  R_j^i \defeq \max\left\{\what{l}_{-i}(X_j^i) - Y_j^i,
  Y_j^i - \what{u}_{-i}(X_j^i)\right\},
\end{equation*}
which by inspection satisfies
\begin{equation*}
  R_j^i = \inf\left\{\tau \in \R \mid
  \what{l}_{-i}(X_j^i) - \tau
  \le Y_j^i \le
  \what{u}_{-i}(X_j^i) + \tau \right\}.
\end{equation*}
We then construct $S^i_{1-\alpha} = \quantplus[\alpha](\{R_j^i\}_{j =
  1}^{n_i})$ (as in Algorithm \ref{alg:nested-multi-env}), and
setting $\what{\tau}
=\quantplus[\delta]\left(\{S^k_{1-\alpha}\}_{k = 1}^m\right)$,
the
multi-environment jackknife-minmax confidence set becomes
\begin{equation*}
  \Cjack(x) \defeq
  \bigcup_{i = 1}^m
  \left[\what{l}_{-i}(x)
    - \what{\tau},
    \what{u}_{-i}(x) + \what{\tau}\right].
\end{equation*}
Theorem~\ref{theorem:nested-coverage} shows this set provides
valid coverage:
\begin{corollary}
  The lower/upper set $\Cjack$ provides $(\alpha, \delta)$-coverage~\eqref{eqn:coverage-def}.
\end{corollary}

\section{Real Data Examples}
\label{sec:real_data}

With the ``basic'' setting of the preceding sections, we move to real-data
experiments to evaluate the methods; these methods motivate some of the
extensions and further theory to come in the sequel.

\subsection{Neurochemical Sensing}
\label{exp:neurochemical_sensing}

We apply our algorithms to the prediction of neurotransmitter concentration levels, with a specific focus on dopamine.
Estimating dopamine levels in awake, functioning humans at a relatively high
frequency is notoriously challenging, but because dopamine governs critical
human behavior, understanding stimuli that maintain healthy dopamine levels
is of crucial importance~\cite{doi:10.1073/pnas.1513619112}.

Scientists now have access to extensive lab generated multi-environment data that can aid in improving  human dopamine level predictions~\cite{LoewingerPa22,BANG2020999}.
Scientists expose electrodes to various known dopamine concentrations and collect measurements of currents that pass through the electrodes at different voltage potentials (say $p$ different potentials).
From each electrode, the scientist obtains a matrix, where each row records
the different current levels in the electrode, when one changes its
potential over a set of $p$ values while exposing it to a specific dopamine
concentration level.
For each electrode exposed to a certain concentration
level, the scientist collects multiple $p$-dimensional measurements
across different time points.
By changing the concentration over $\ell \ge 1$ distinct levels and
collecting $t \gg 1$ observations at each level, the
scientist obtains $n=\ell \times t$ observations, resulting in an
$n \times p$ covariate matrix from each electrode.
The $t$ measurements
corresponding to each level may exhibit weak correlations, but since
state-of-the-art scientific work in the area \cite{LoewingerPa22} treats
these to be independent, we adopt this convention.
Row outcomes thus measure the dopamine level that generated that row's
current values.
Variations in electrode construction and the experimental
setup under which measurements are obtained mean that the data from different
electrodes follow different distributions.

To map this application to our setting, we consider each electrode to be one of our environments.
We use data from multiple electrodes to train our algorithms, hoping that such multi-environment learning would create robust prediction models that generalize better when applied in a different context, e.g.\ while predicting dopamine levels on the human brain.  
The training data comprises 15 environments corresponding to $15$
electrodes. Each includes roughly 20,000 observations~\cite{LoewingerPa22,
  Loewinger856385}.  of current measured in nanoamps (nA) collected at 1000
discrete voltage potentials.
For each observation, the outcome is a measurement of nanomolar (nM) dopamine concentration.
The outcomes lie in $[0, 2000]$, so we intersect our predictions with this
range before producing intervals.

In each experiment, we select 5 environments at random for training and use
the rest for testing. We train Algorithms \ref{alg:multi-env} and
\ref{alg:multi-env-split} and examine their
coverage on the test data under the $(\alpha, \delta)$-coverage notion
in Definition \ref{definition:valid-coverage}. We use ridge regression for
the base model $\what{f}$ and leave-one-out cross validation for choosing
the ridge parameter, repeating the experiment $100$ times and plotting the
average coverage and the average set length (defined below).

For the $k$th of 100 experiments, we let $ \{e_{k,i}\}_{1
  \leq i \leq 5}$ and $ \{e_{k,i}\}_{6 \leq i \leq 15}$ denote the train and
test environments, respectively.  For $k \in [100], i \in [15]$,
we use $n_{k,i}$ to denote the sample size in
environment $e_{k,i}$ and $\{X_j^{k,i}, Y_{j}^{k,i}\}_{1 \leq j
  \leq n_{k,i} } $ the observations. We define
\begin{equation*}
  A_{j}^{k, i } \defeq \indic{Y_j^{k,i} \in
    \what{C}\left(X_j^{k,i}\right)}
\end{equation*}
to indicate whether, in experiment $k$, the confidence set
covers the outcome in the $j$th sample in environment $i$.

We say a test environment is covered if the fraction of covered samples in the environment is at least $1 - \alpha$. Theorems~\ref{theorem:nested-coverage} and \ref{theorem:nested-coverage-split} show that
we expect to cover at least $1 - \delta$ fraction of the test environments.
We define ``empirical $1 - \delta$'' as the fraction of test environments
covered across our experiments, ``empirical set length'' as the average
length of constructed confidence sets averaged over the test environments,
and ``empirical $1 - \alpha$" as the average fraction of covered samples
over the covered test environments: letting $\wb{A}^{k,i}
= \frac{1}{n_{k,i}} \sum_{j = 1}^{n_{k,i}} A_j^{k,i}$, these become
\begin{align*}
  &  ``\text{Empirical $1 - \delta$}" := \\
  &  \quad \frac{1}{1000} \sum_{k=1}^{100} \sum_{i=6}^{15} \indicbigg{ \sum_{j=1}^{n_{k,i}} A_{j}^{k, i } \geq
    \ceil{(1-\alpha)\left(n_{k,i}+1\right)}},\\
  & ``\text{Empirical $1 - \alpha$}" := \\
  & \frac{\sum_{k=1}^{100} \sum_{i=6}^{15}
    \indic{\sum_{j=1}^{n_{k,i}} A_{j}^{k, i } \geq
      \ceil{(1-\alpha)\left(n_{k,i}+1\right)}}
    \wb{A}^{k,i}}{
    \sum_{k=1}^{100} \sum_{i=6}^{15} \indic{ \sum_{j=1}^{n_{k,i}} A_{j}^{k, i } \geq\left\lceil(1-\alpha)\left(n_{k,i}+1\right)\right \rceil }} ,\\
  & ``\text{Empirical Set Length}" := \\
  & \qquad \frac{1}{1000} \sum_{k=1}^{100} \sum_{i=6}^{15}    \frac{ 1}{n_{k,i}}\sum_{j=1}^{n_{k, i}}
  \left| \what{C}\left(X_j^{k, i}\right) \right|.
\end{align*}

\subsubsection{Influence of Input \texorpdfstring{$\delta$}{delta}}\label{subsec:inputdelta}

To examine the influence of the input $\delta$ on the performance of our
algorithms, we set $\alpha = 0.05$, and the split ratio to be 0.5 for
Algorithm \ref{alg:multi-env-split}. We vary the values of $\delta$ and
display the results in Figure \ref{fig:vary_delta}. The plots show that both
multi-environment split conformal and jackknife-minmax produce valid
coverage. Multi-environment split conformal tends to generate more
conservative prediction intervals than jackknife-minmax.
The relationship between the empirical $1- \alpha$ and the input
$1-\delta$ is non-monotone, because an increase in the input $1-
\delta$ tends to increase the set length for both algorithms, which in turn
may increase the empirical $1- \alpha$.
On the other hand, we may achieve higher $1 - \delta$
by including more environments with low
coverage while decreasing empirical
$(1-\alpha)$. 

\begin{figure}
  \centering
  \includegraphics[width=9cm]{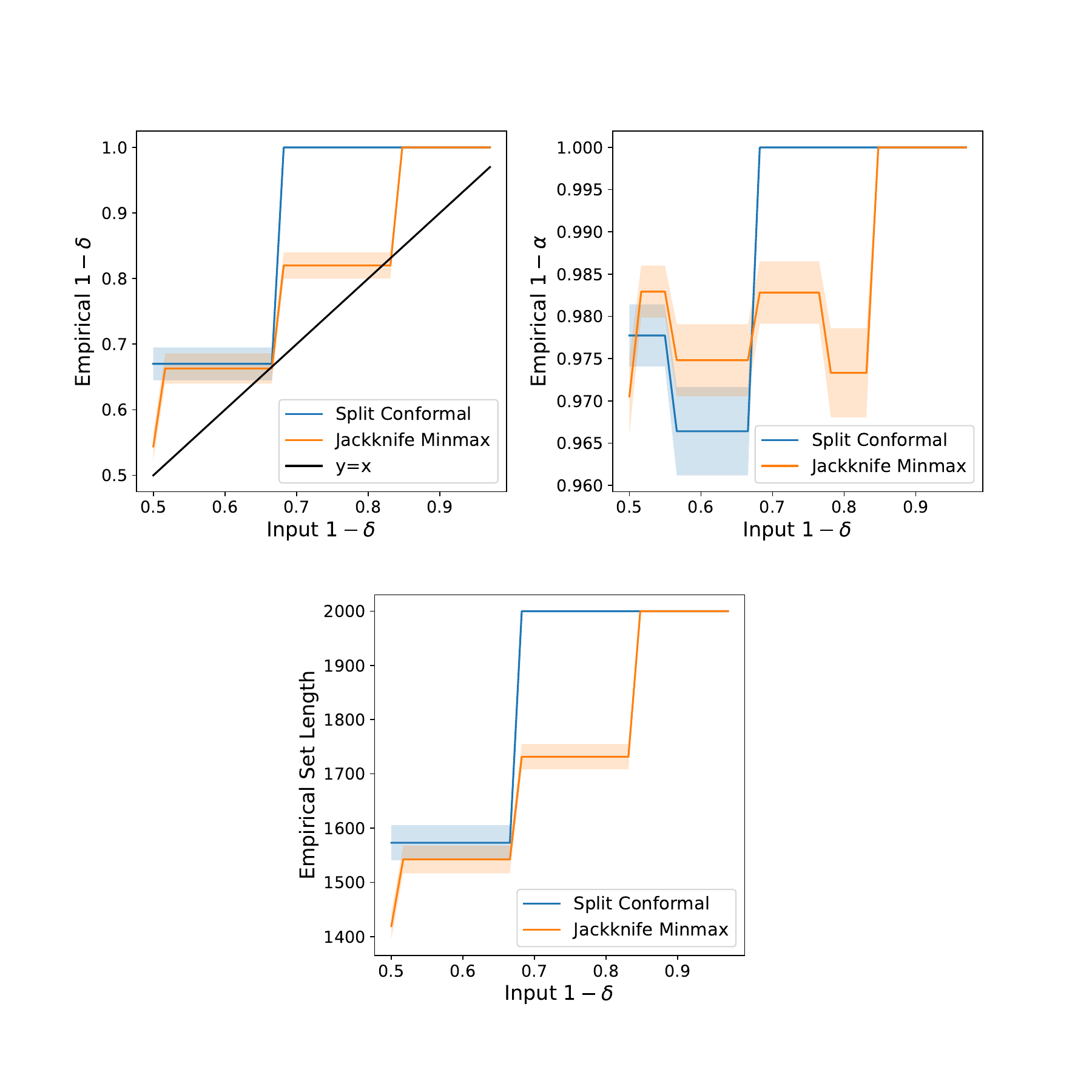}
  \vspace{-1cm}
\caption{\label{fig:vary_delta}
Influence of input $\delta$ on the performance of conformal algorithms applied to the neurochemical sensing data, with $\alpha = .05$. The plots show the empirical $1-\delta$, empirical $1-\alpha$, and empirical set length for both the split conformal and jackknife-minmax algorithms with various input $\delta$.}
\end{figure}

\subsubsection{Influence of Input \texorpdfstring{$\alpha$}{alpha}}\label{subsec:alpha}

To examine the influence of the input $\alpha$ on the performance of our conformal algorithms, we set $\delta = 0.33$, and the split ratio to be 0.5 for Algorithm \ref{alg:multi-env-split},
choosing $\delta=0.33$ as because for smaller values, multi-environment conformal outputs [0, 2000] as the prediction interval regardless of the input $\alpha$.
We vary the $\alpha$ values and display the results in Figure \ref{fig:vary_alpha}.
We observe that for both algorithms, the empirical $1- \alpha$ tends to increase as the input $1- \alpha$ increases. However, the relationship between the empirical $1- \delta$ and the input $1-\alpha$ is less clear. Two factors influence this relationship. As $1- \alpha$ increases, the set length of conformal intervals will increase so that each sample is more likely to be covered. Nonetheless, as the input $1-\alpha$ increases, more samples in each test environment need to be covered, and the fraction of environments satisfing the condition may decrease.

\begin{figure}
\centering
\includegraphics[width=9cm]{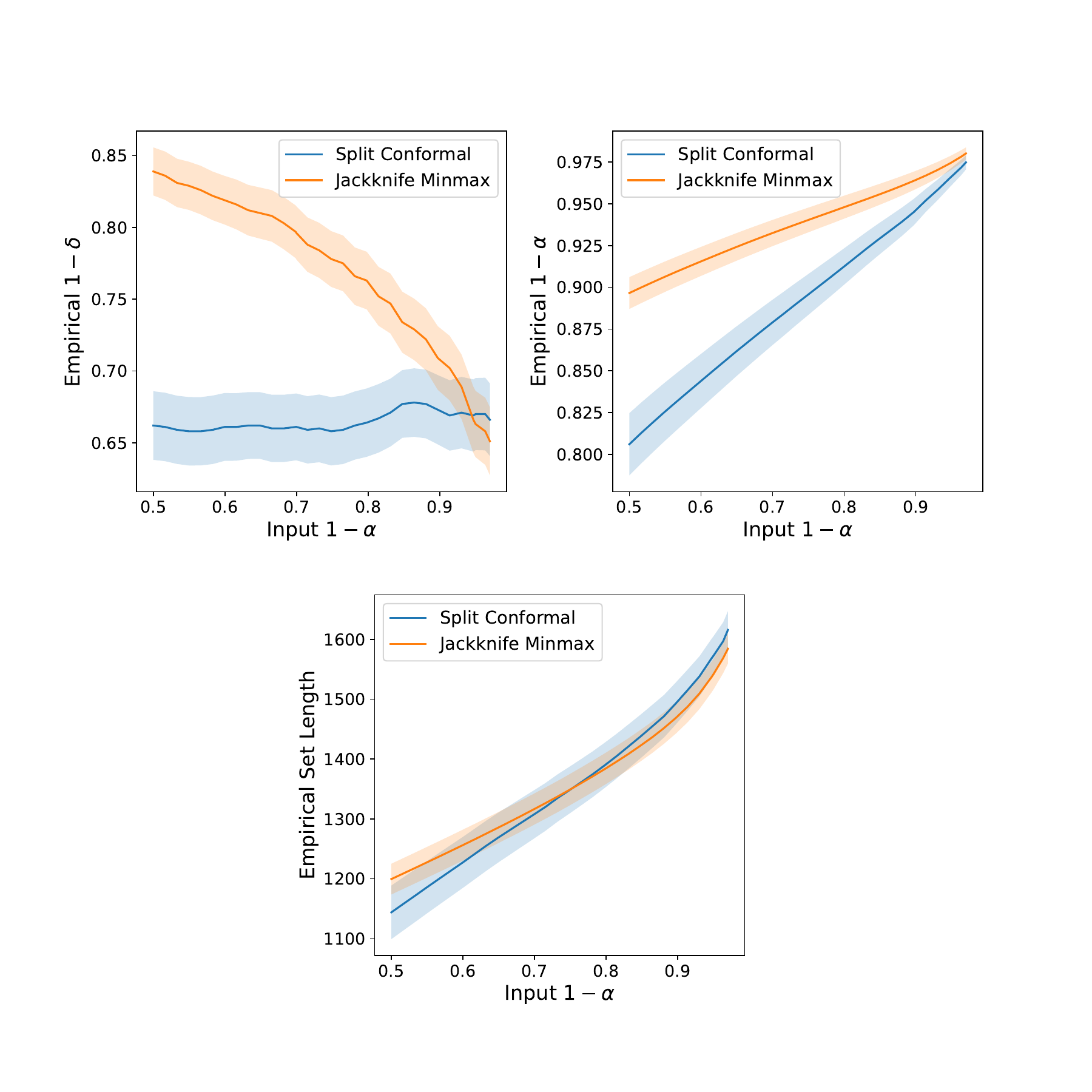}
  \vspace{-1cm}
\caption{\label{fig:vary_alpha}
  Influence of input $\alpha$ on the performance of conformal algorithms the neurochemical sensing data. For these experiments, $\delta$ is set to be 0.33. The plots show the empirical $1-\delta$, empirical $1-\alpha$, and empirical set length for both the split conformal and jackknife-minmax algorithms with various input $\alpha$.}
\end{figure}

\subsection{Species Classification}
\label{exp:species_classification}

We next apply our algorithms in the context of species classification.  To
monitor wildlife biodiversity, ecologists use camera traps---heat or
motion-activated cameras placed in the wild---which exhibit variation in
illumination, color, camera angle, background, vegetation, and relative
animal frequencies, so we consider each camera trap an
environment.
Ecologists seek to use existing camera trap shots to train
machine learning models that classify wildlife species accurately in new
deployments~\cite{KohSaMaXiZhBaHuYaPhBeLeKuPiLeFiLi21}.

The covariates of this species classification data are 2D images, and the
targets are species of animals present in the
images~\cite{beery2020iwildcam}. We pre-process the data by removing
environments with at most 100 observations, removing labels that
appear in less than 5 percent of the remaining environments. After the
pre-processing, we obtain 219 environments and 57 labels. On average, each
environment consists of 874 images. We then randomly choose 50 environments
for training and keep the remaining 169 for testing.
For the base model, we use a ResNet-50 model pretrained
on ImageNet using a learning rate of $3 \cdot 10^{-5}$ and no $\ell_2$-regularization~\cite{He2015DeepRL}. Since the pretrained model takes in images of size $448 \times 448$, we rescale the inputs to the same size. We repeat the experiment 20 times, and then plot the average coverage and the average set length.

\subsubsection{Influence of Input \texorpdfstring{$\delta$}{delta}}
We examine the performance of our algorithms as the input $\delta$ varies (Figure \ref{fig:iwilds_vary_delta}) similar to Section \ref{subsec:inputdelta} earlier.  We observe that the performance is now flipped, the multi-environment jackknife-minmax is now more conservative. With an increased number of training environments, the multi-environment split conformal outperforms jackknife-minmax also in terms of empirical set length.

\begin{figure}
\centering
\includegraphics[width=9cm]{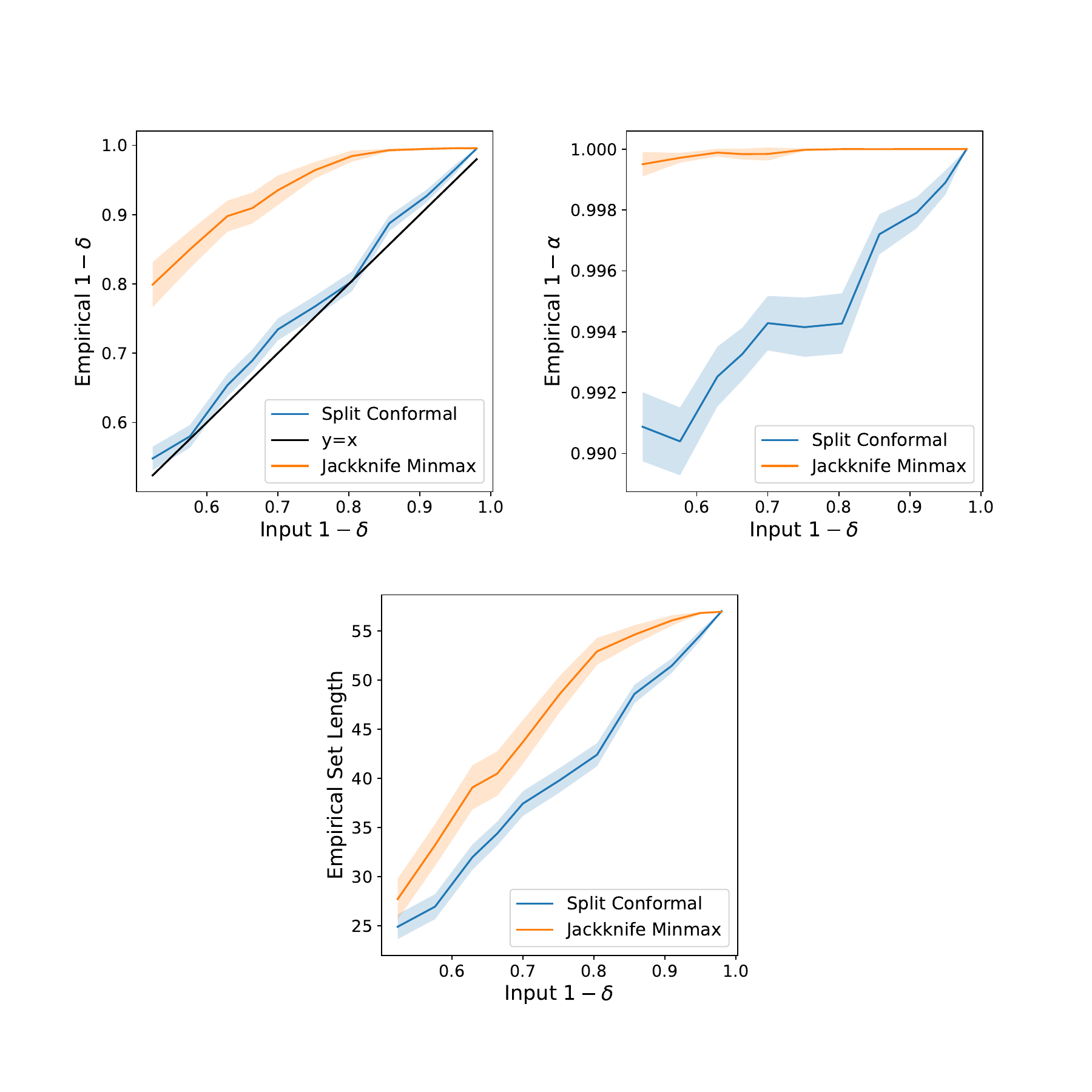}
  \vspace{-1cm}
\caption{\label{fig:iwilds_vary_delta}
  Influence of input $\delta$ on the performance of conformal algorithms applied to the species classification data, with $\alpha = .05$. The plots show the empirical $1-\delta$, empirical $1-\alpha$, and empirical set length for both the split conformal and jackknife-minmax algorithms with various input $\delta$.}
\end{figure}

\subsubsection{Influence of Input \texorpdfstring{$\alpha$}{alpha}}
We examine the performance of our algorithms as the input $\alpha$ varies (Figure \ref{fig:iwilds_vary_alpha}) similar to Section \ref{subsec:alpha}.  We observe that the empirical $1-\alpha$ and empirical set length both increase as the input $1-\alpha$ increases. Interestingly, the conformal sets output by jackknife-minmax are not much larger than those output by split conformal under this setting.

\begin{figure}
\centering
\includegraphics[width=9cm]{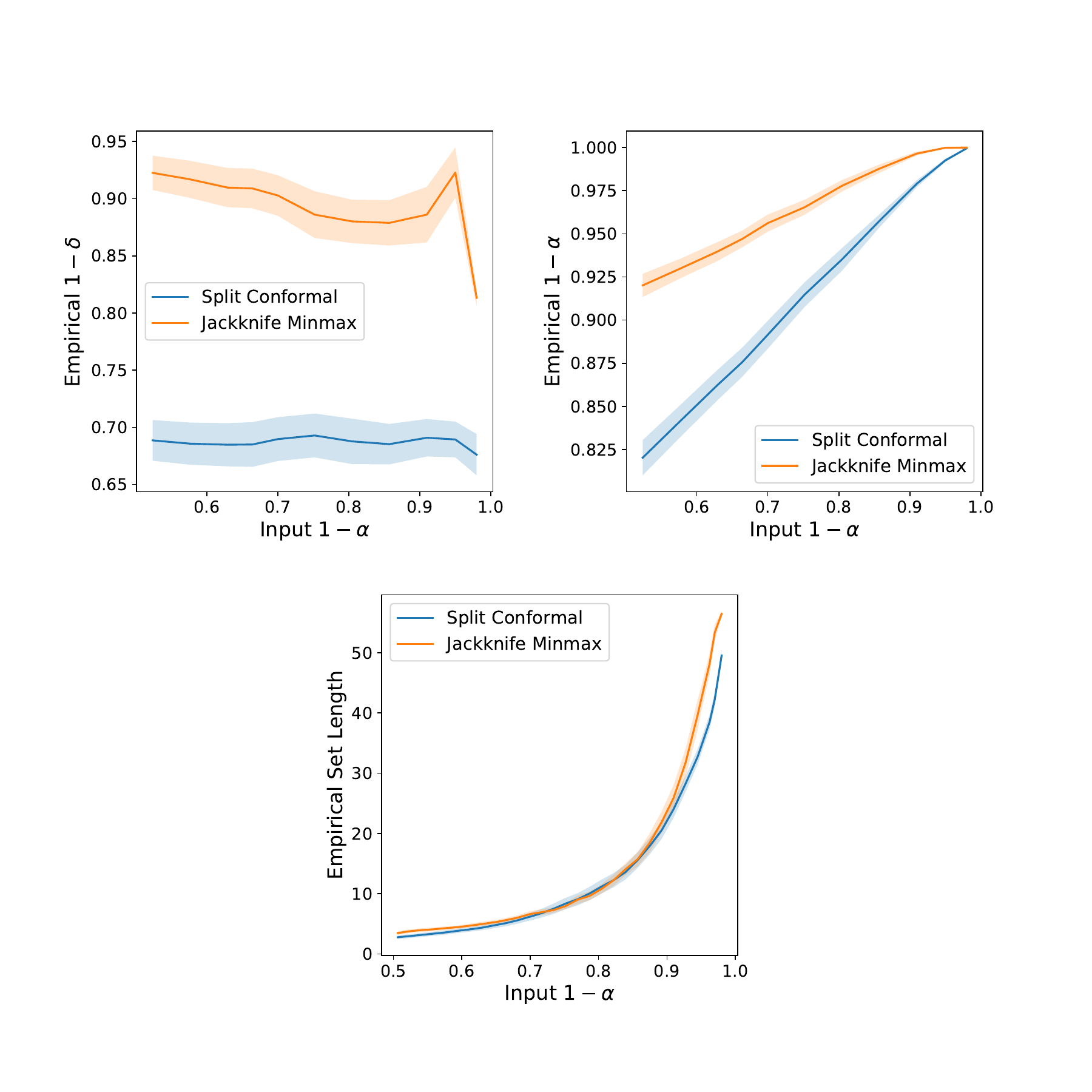}
  \vspace{-1cm}
\caption{\label{fig:iwilds_vary_alpha}
  Influence of input $\alpha$ on the performance of conformal algorithms the species classification data. For these experiments, $\delta$ is set to be 0.33. The plots show the empirical $1-\delta$, empirical $1-\alpha$, and empirical set length for both the split conformal and jackknife-minmax algorithms with various input $\alpha$.}
\end{figure} 

\subsection{Comparison with Hierarchical
Conformal Prediction and Jackknife+} \label{sec:compare_hcp} To conclude, we
provide a comparison of our algorithms with HCP (Alg.~\ref{alg:hcp_split})~\cite{LeeBaWi23}. Setting up
this comparison is non-trivial since HCP provides  different
coverage guarantees. Nonetheless, since they work under similar
hierarchical models, we believe a comparison to be instructive. To set this
up, for each fixed value of $\alpha$, we find the largest $\delta$ such that
the fraction of overall test samples covered by multi-environment split
conformal exceeds that of HCP. We then compare the performance of
multi-environment split conformal with parameters $\alpha, \delta$ and HCP
with parameter $\alpha$. We apply the two methods on the neurochemical
sensing and species classification data. We display the results in Figures
\ref{fig:bio_compare_Hj} and \ref{fig:iwilds_compare_HCP}, respectively.

The $\delta$ selection procedure means that multi-environment split
conformal and jackknife-minmax produce slightly larger prediction sets than
HCP and hierachical jackknife+, respectively.
Moreover, we observe that for all the conformal algorithms considered, coverage (i.e. empirical $1-\delta$ and empirical $1- \alpha$) generally increase as the input $1 - \alpha$ increases.
However, as Figure~\ref{fig:iwilds_compare_HCP} shows, this relation does not always hold.
As the input $1-\alpha$ becomes larger, the average set size of the produced prediction sets also becomes larger, which tends to increase the empirical $1-\delta$.
On the other hand, as the input $1-\alpha$ becomes larger, fewer environments will have at least $1-\alpha$ fraction of samples covered, which tend to decrease the empirical $1-\delta$,
making relation between empirical $1-\delta$ and input $1-\alpha$ potentially
non-monotone.
The empirical $1-\alpha$ and input $1-\alpha$ similarly have a non-monotonic
relationship.

In sum, with an appropriate choice of $\delta$, we observe that the multi-environment split conformal and jackknife-minmax produce prediction sets with similar size and coverage properties as the HCP and hierarchical jackknife+, respectively.

\begin{figure}
\centering
\includegraphics[width=9cm]{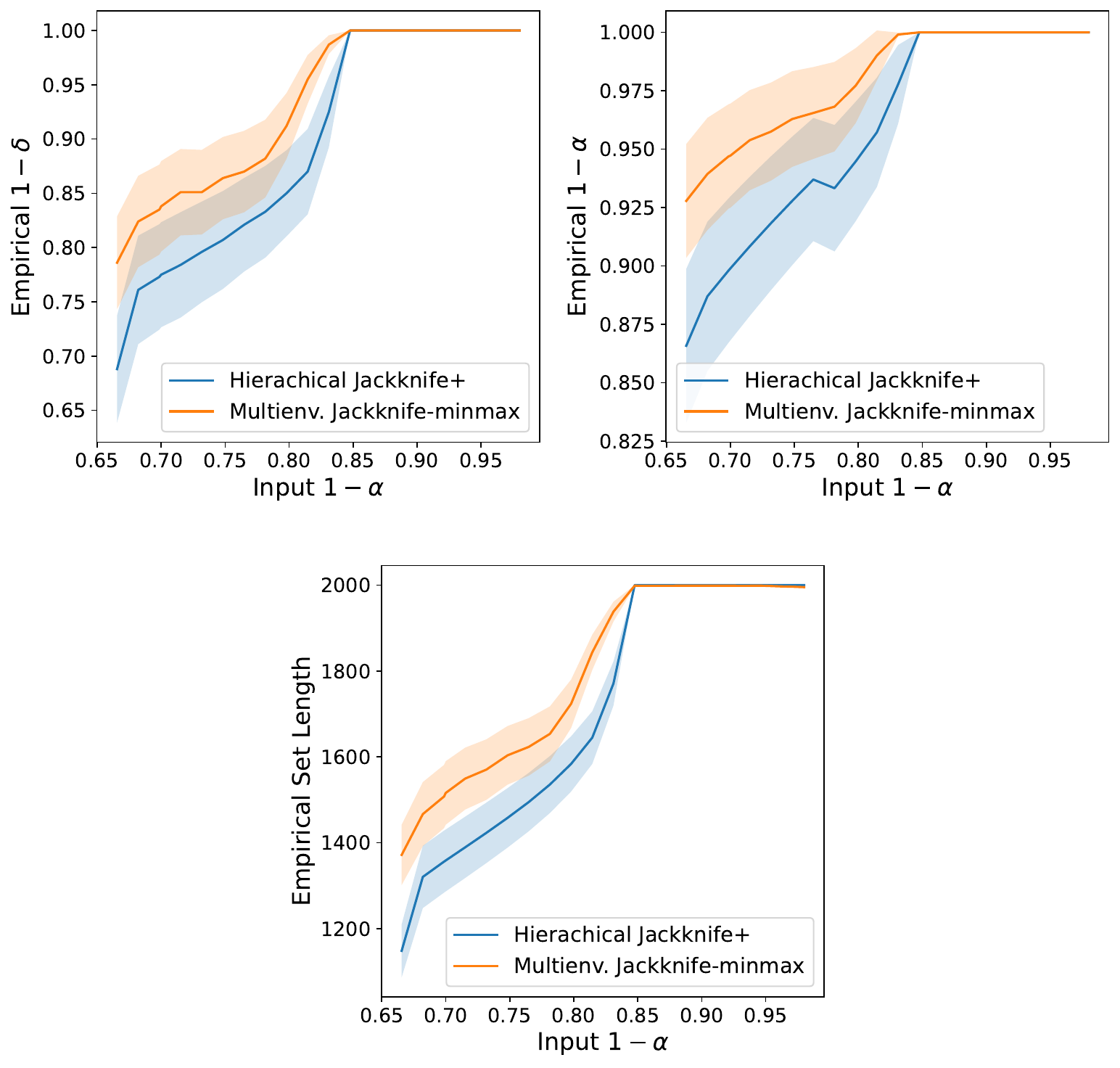}
\vspace{-.2cm}
\caption{\label{fig:bio_compare_Hj}
  Performance of multi-environment jackknife-minmax and hierarchical jackknife+ applied to the neurochemical sensing data. Multi-environment jackknife-minmax takes in the parameters $\alpha, \delta$, and hierarchical jackknife+ takes in the parameter $\alpha$. For each value of $\alpha$, we find the largest $\delta$ such that the fraction of test samples covered by multi-environment split conformal exceeds that of hierarchical jackknife+.}
\end{figure}

\begin{figure}
\centering
\includegraphics[width=9cm]{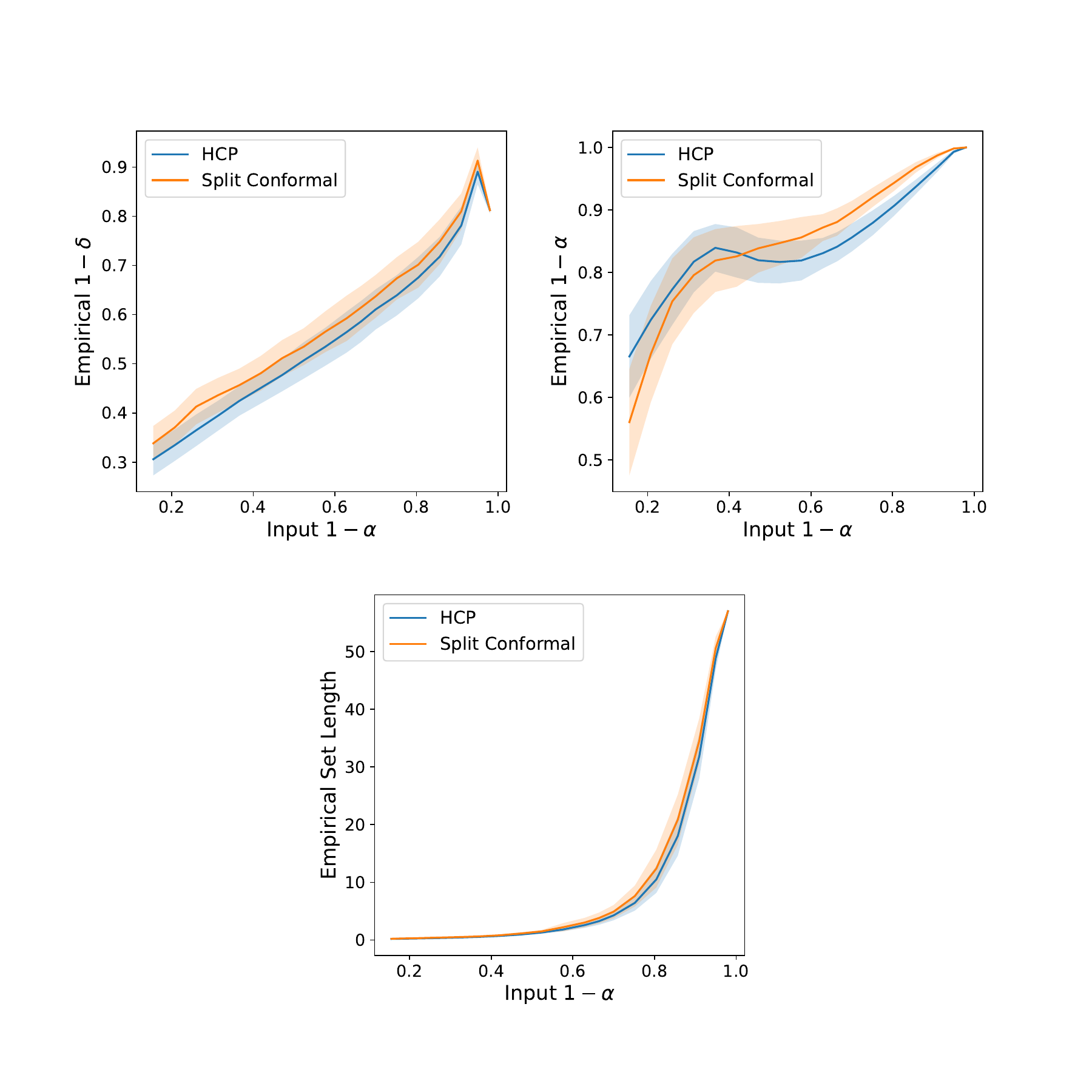}
\vspace{-1cm}
\caption{\label{fig:iwilds_compare_HCP} Performance of multi-environment split conformal and HCP applied to the species classification data. Multi-environment split conformal takes in the parameters $\alpha, \delta$, and HCP takes in the parameter $\alpha$. For each value of $\alpha$, we find the largest $\delta$ such that the fraction of test samples covered by multi-environment split conformal exceeds that of HCP.}
\end{figure}

\section{Resizing residuals to reduce interval lengths}
\label{sec:resizing_algos}

Our experimental results in Section~\ref{sec:real_data} make clear that
naive application of multi-environment algorithms often produces wide
prediction intervals. The bottom plots in Figures \ref{fig:vary_delta} and
\ref{fig:iwilds_vary_delta} show that the average size of the confidence
sets are particularly large when the input parameter $\delta$ is small.
This conservativeness is a natural consequence of the constructions of the
multi-environment confidence sets, just as Romano et al. \cite{RomanoPaCa19} note that
naive symmetric and uniform prediction intervals (i.e., those of the form
$\what{C}(x) = [f(x) - \tau, f(x) + \tau]$) can be overly conservative as
they must typically cover even values $X$ for which $Y$ is highly
non-symmetric or has high variance.  In our context, this presents when an
environment $i$ has residual quantile $S_{1-\alpha}^i$ much larger than the
rest---it is an outlier environment. For small $\delta$, such outlier
environments mostly determine the $1 - \delta$ quantile of the training
environments' scores $S^i_{1-\alpha}$, and consequently, these outlier
training environments govern the size of the confidence set for test samples
regardless of whether the test environment is outlying.

To mitigate this issue, we scale the residual quantiles by a resizing factor
so that the adjusted quantiles (i) remain exchangeable, and (ii) are
supported on a similar range for \emph{all} environments. We compute the $1
- \delta$ quantile of these adjusted residual quantiles. Finally, we
multiply this $1 - \delta$ quantile by the test environment's resizing
factor to construct valid confidence sets. This way, the constructed
confidence set length remains small if the test environment is not an
outlier. Since the probability of any environment being an outlier is small,
the size of the constructed confidence set tends to be small on average.

The question therefore turns to finding an accurate strategy for
constructing these resizing factors.  One natural approach is---if they are
available---to incorporate environmental covariates.  If environmental
covariates $e \in \mc{E}$ are available for each environment, then we can
estimate resizing factors using them.
(One could also incorporate
these into the predictor $f : \mc{X} \times \mc{E} \to \mc{Y}$ via an
expanded covariate $(X, e)$.)  Alternatively, we show that given access to a
limited amount of labeled data from the test environment, one can indeed
construct suitable resizing factors. We describe our strategy in the context
of the multi-environment split conformal algorithm.

Suppose we observe $L^{m+1}$ labeled random examples from the test
environment. As before, we partition the training environments into two sets
$D_1$ and $D_2$. Using data from environments in $D_1$, we construct a
nested confidence set $\{\what{C}_\tau\}_{\tau \in \mathbb{R}} $. We
randomly select $L^{m+1}$ samples from each training environment and compute
residuals for these samples using the $\{\what{C}_\tau\}_{\tau \in
  \mathbb{R}} $ confidence sets. We then pick some quantile $\alpha_0$ close
to 0, and compute the $1 - \alpha_0$ quantile of these residuals as the
resizing factor for each environment. As a result, the resizing factors for
outlier environments tend to be large as desired. Empirically, we find that
$\alpha_0 = 0.05$ works well for a reasonable range of input $\alpha$.
One can adapt this approach to any of multi-environment jackknife-minmax,
hierarchical conformal prediction, or hierarchical jackknife+;
Algorithm~\ref{alg:resized-nested-multi-env-split} shows how to
do so for multi-environment split conformal prediction.

\begin{center}
  \algbox{
  \label{alg:resized-nested-multi-env-split}
    Resized Multi-environment Split Conformal
  }{
    \textbf{Input:} samples $\{X^i_j, Y^i_j\}_{j = 1}^{n_i}$,
    $i \in [m]$, labeled samples $\{X_j^{m+1}, Y_j^{m+1}\}_{j \in L^{m+1}}$, confidence levels $\alpha, \delta$, split ratio $\gamma$, predictive set algorithm $\alg$, resizing quantile $\alpha_0$\\

    Randomly partition $[m]$ into $D_1$ and $D_2$ with $\frac{|D_1|}{m} = \gamma$.
    \textbf{set} $\{\what{C}^{D_1}_\tau\}_{\tau \in \mathbb{R}} = \alg( \{(X^{i}, Y^{i})\}_{i \in D_1} )$

    \textbf{For} $i \in D_2$,
    \begin{enumerate}[1.]
    \item Compute residuals $$  R_j^i=\inf \left\{\tau \mid Y_j^i \in \widehat{C}^{D_1}_\tau\left(X_j^i\right)\right\}, \quad j=1, \ldots, n_i. $$

    \item Randomly select $ |L^{m+1}| $ samples in environment $i$, and denote the set of selected samples as $ L^i $. Compute the resizing factor $$ s^i:=\widehat{q}_{|L^i|, \alpha_0}^{+}\left(\{ R_j^i \}_{j \in L^{i}}\right) ,$$ and resized residual quantiles
    \begin{equation*}
      S_{1-\alpha}^i=\widehat{q}_{n_i, \alpha}^{+}\left( \left\{R_{j}^i/{s^i} \right\}_{j \in [n_i] \setminus L^i} \right).
    \end{equation*}
    \end{enumerate}
    
    \textbf{compute} the resizing factor for the test environment $$ s^{m+1} :=  \widehat{q}_{\left|L^{m+1}\right|, \alpha_0}^{+}\left(\left\{R_j^i\right\}_{j \in L^{m+1}}\right). $$
    
    \textbf{Return} confidence interval mapping with $\widehat{\tau}=s^{m+1} \cdot \widehat{q}_\delta^{+}(\{S_{1-\alpha}^i\}_{i \in D_2}),$ $$ \Cresizedsplit(x):= \what{C}^{D_1}_{\what{\tau}}(x) .$$
  }
\end{center}

As long as the resizing factors $s^i, i \in D_2 \cup \{m+1\}$ are
exchangeable and independent of samples $(X^i, Y_i), i \in 
Y_j^i\}_{j=1}^{n_i}, i \in D_1$, the proof of
Theorem~\ref{theorem:main-coverage-split} immediately
extends to show the following result:

\begin{theorem}
  \label{theorem:resized-nested-coverage-split}
  The multi-environment confidence mapping $\Cresizedsplit(x)$ that
  Algorithm~\ref{alg:resized-nested-multi-env-split} returns provides level
  $(\alpha, \delta)$-coverage~\eqref{eqn:coverage-def}, that is,
  the event
  \begin{align*}
    E_m & \defeq \bigg\{
    \sum_{j \in [n_{m+1}]
      \setminus L^{m + 1}} \indic{Y^{m+1}_j \in \Cresizedsplit(X^{m+1}_j)}
      \\
      & \qquad\qquad ~ \ge \ceil{(1-\alpha)(n_{m+1}- |L^{m+1}|+1)} \bigg\}
  \end{align*}
  satisfies $\P(E_m) \ge 1 - \delta$.  If the quantiles
  $\scorerv^i_{1-\alpha}$ are distinct with probability 1, then $\P(E_m) \le 1 -
  \delta + \frac{1}{m(1 - \gamma) + 1}$.
\end{theorem}

As our experimental
results motivate the
resizing in Alg.~\ref{alg:resized-nested-multi-env-split},
we turn now to empirical investigation.



\subsection{Resized Multi-environment Split Conformal}

We apply the Algorithm~\ref{alg:resized-nested-multi-env-split} to both the neurochemical sensing and the species classification datasets,
demonstrating that the resized split conformal algorithm indeed reduces the
average length of conformal intervals.
In both examples, we set the resizing quantile $\alpha_0$ to be 0.05.

We first duplicate the dopamine sensing experiment in Section
\ref{exp:neurochemical_sensing}, except we randomly sample 30 labeled data
(of the approximately 20,0000 observations) in each environment to construct
the resizing factors.
Figure~\ref{fig:bio_resizing_0.001} presents coverage
results with respect to the unlabeled data in each test environment.

The left plot shows the relation between the empirical set length and the input $1-\alpha$. The blue and the orange curves correspond to the results of split conformal with and without resizing, respectively. The green curve corresponds to the oracle case where we know the $1 - \alpha_0$ residual quantile of all the unlabeled data in each test environment, and we use the $1 - \alpha_0$ residual quantile as the resizing factor.
We observe that when the input $1- \alpha$ is close to 1, the empirical set length is large for all three methods, as the resizing method cannot reduce the average set length by much.
For moderate $1-\alpha$, on the other hand, the resizing methods can reduce the average set length.
The right plot shows the relation between empirical $1-\delta$ and input
$1-\alpha$, demonstrating that the resizing methods provide valid coverage
and corroborating Theorem \ref{theorem:resized-nested-coverage-split}.

\begin{figure}
\centering
\includegraphics[width=9cm]{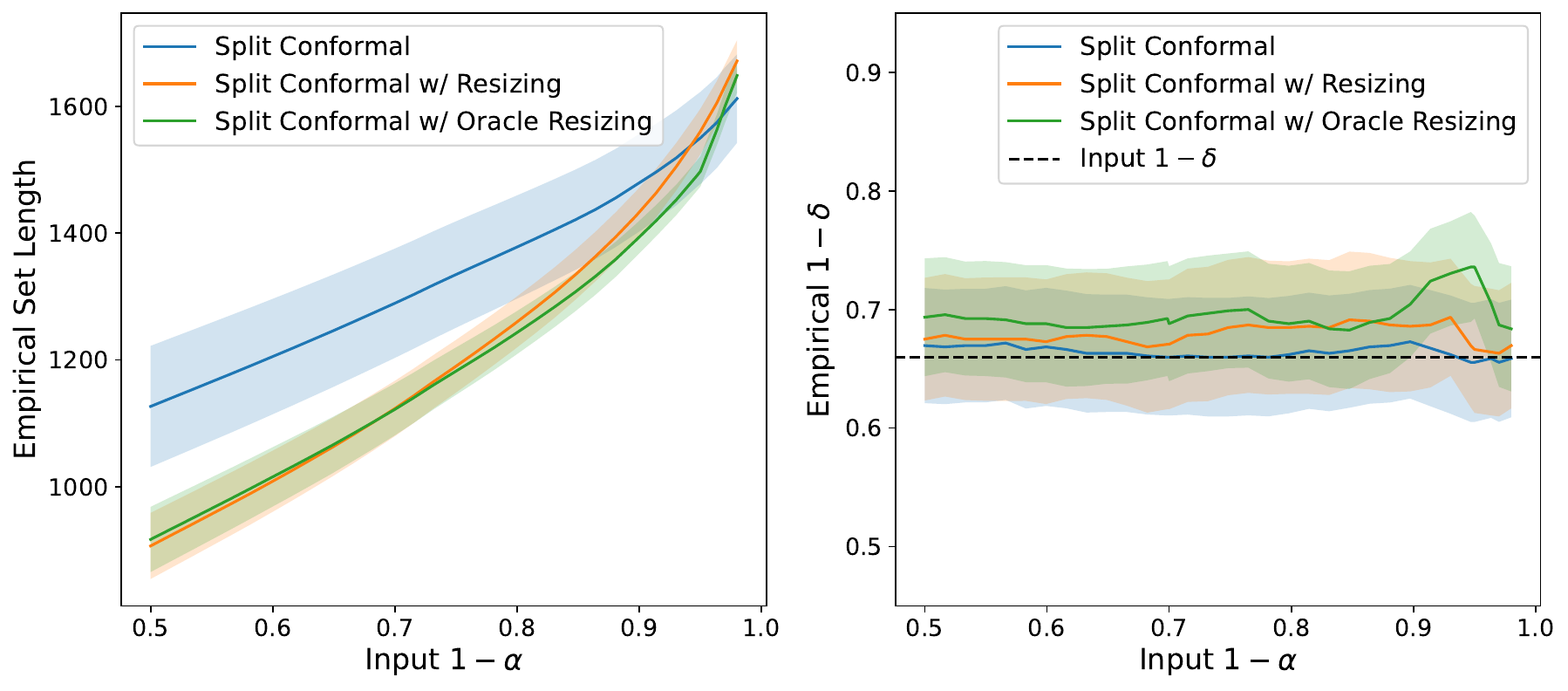}
\caption{Performance of the split conformal algorithm with and without resizing applied to the neurochemical sensing data. For these experiments, $\delta$ is set to be 0.33. The left plot shows the
relation between empirical set length and input $1-\alpha$, while the right plot shows the relation between empirical $1 -  \delta$ and input $1 - \alpha$.}
\label{fig:bio_resizing_0.001}
\end{figure} 

For the species classification data, we use 20 randomly selected labeled samples in each environment to construct the resizing factors. The experimental settings are the same as in \ref{exp:species_classification}, except that (i) we fix the input $\delta$ to be 0.05, and vary the input $\alpha$, and (ii) we randomly sample 20 labeled data in each environment to construct the resizing factors. We display the results in Figure~\ref{fig:iwilds_0.1}. Again, we observe that the resizing method reduces the average set length without breaking our coverage guarantees.

\begin{figure}
\centering
\includegraphics[width=9cm]{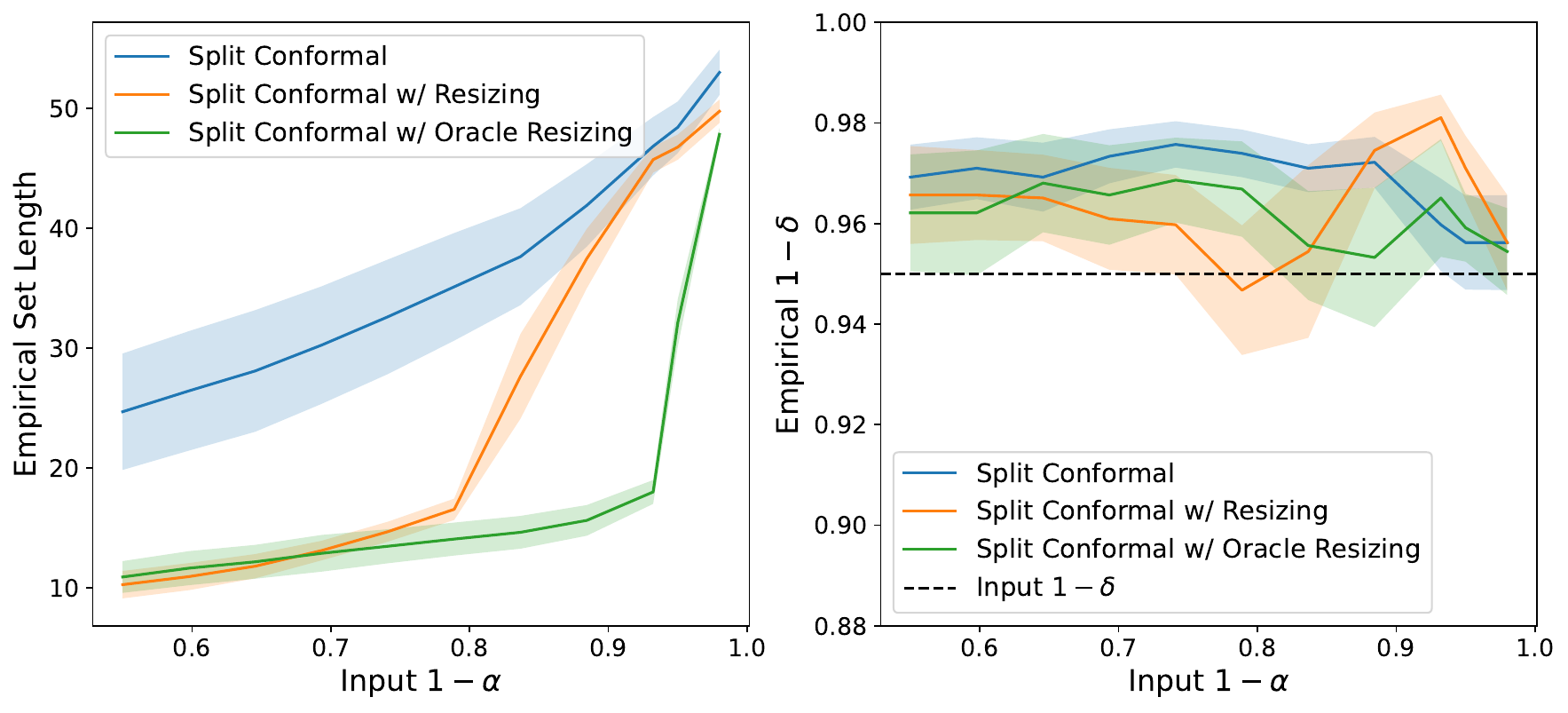}
\caption{Performance of the split conformal algorithm with and without resizing applied to the species classification data. For these experiments, $\delta$ is set to be 0.05. The left plot shows the
relation between empirical set length and input $1-\alpha$, while the right plot shows the relation between empirical $1 -  \delta$ and input $1 - \alpha$.}
\label{fig:iwilds_0.1}
\end{figure}

\subsection{Resized Hierarchical Conformal Prediction (HCP)}
To illustrate the utility of the resizing technique beyond our algorithms,
we study its effects on Lee et al.'s HCP algorithm~\citep{LeeBaWi23} (see
Alg.~\ref{alg:hcp_split}).  Since Lee et al.~design HCP to construct
conformal intervals for a single sample in each test environment, it may not
be practical to use additional labeled samples from the test environments to
construct resizing factors. For illustration purposes, we consider the
oracle setting where we know the $1 - \alpha_0$ quantile of the
residuals for observations in the test environment, using
as the resizing factor. For both the species classification data
and the neurochemical sensing data, we vary the value of the input $\alpha$
and compare the performance of HCP and resized HCP algorithms, displaying
the results in Figures~\ref{fig:bio_resized_hcp} and
\ref{fig:iwilds_resized_hcp}.
The oracle resizing method reduces the average size of the HCP
prediction intervals without breaking the coverage guarantees, highlighting
that such adaptive resizing methods may be essential for strong performance
of hierarchical and multi-environment predictive inference.

\begin{figure}
\centering
\includegraphics[width=9cm]{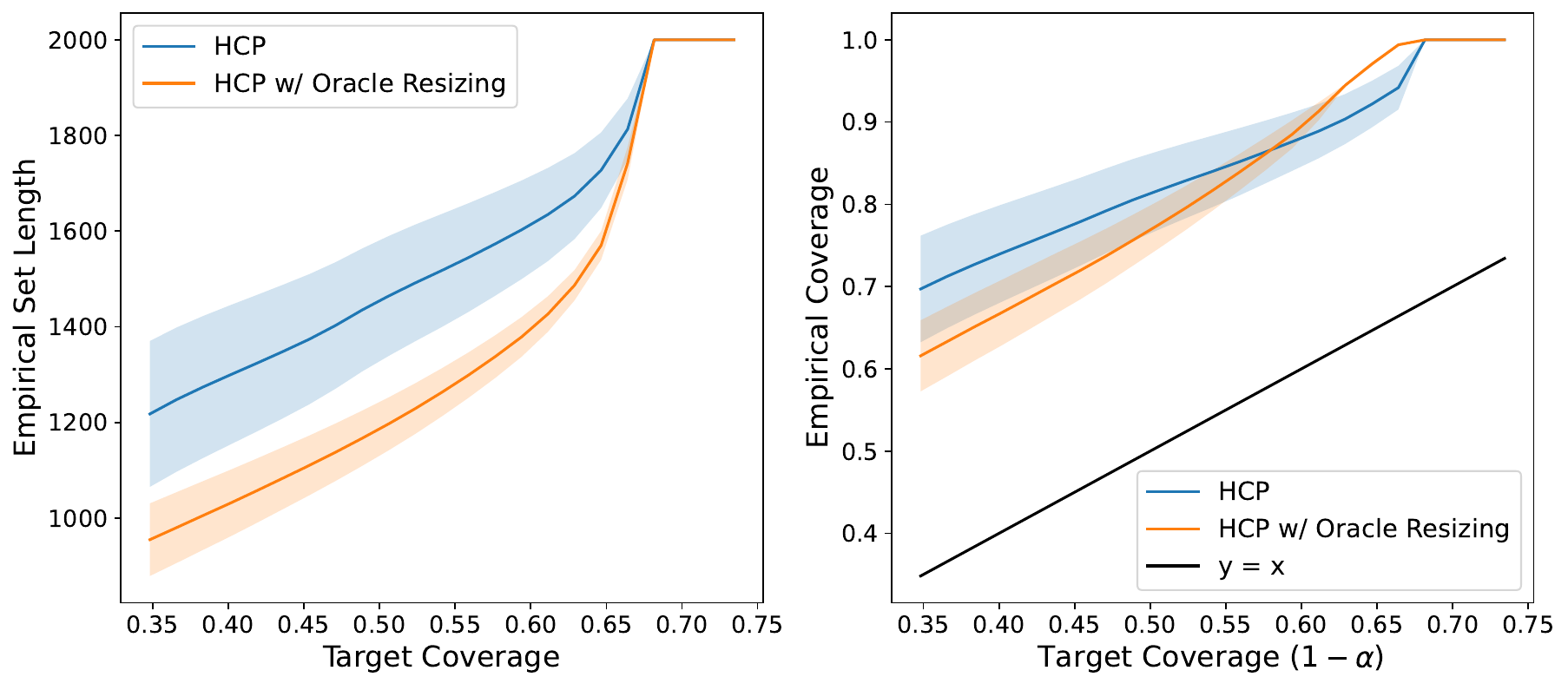}
\caption{Performance of HCP and resized HCP applied to the neurochemical sensing data. Left plot shows the average set size over all test samples, and right plot shows the fraction of all test samples covered by their conformal sets.}
\label{fig:bio_resized_hcp}
\end{figure} 

\begin{figure}
\centering
\includegraphics[width=9cm]{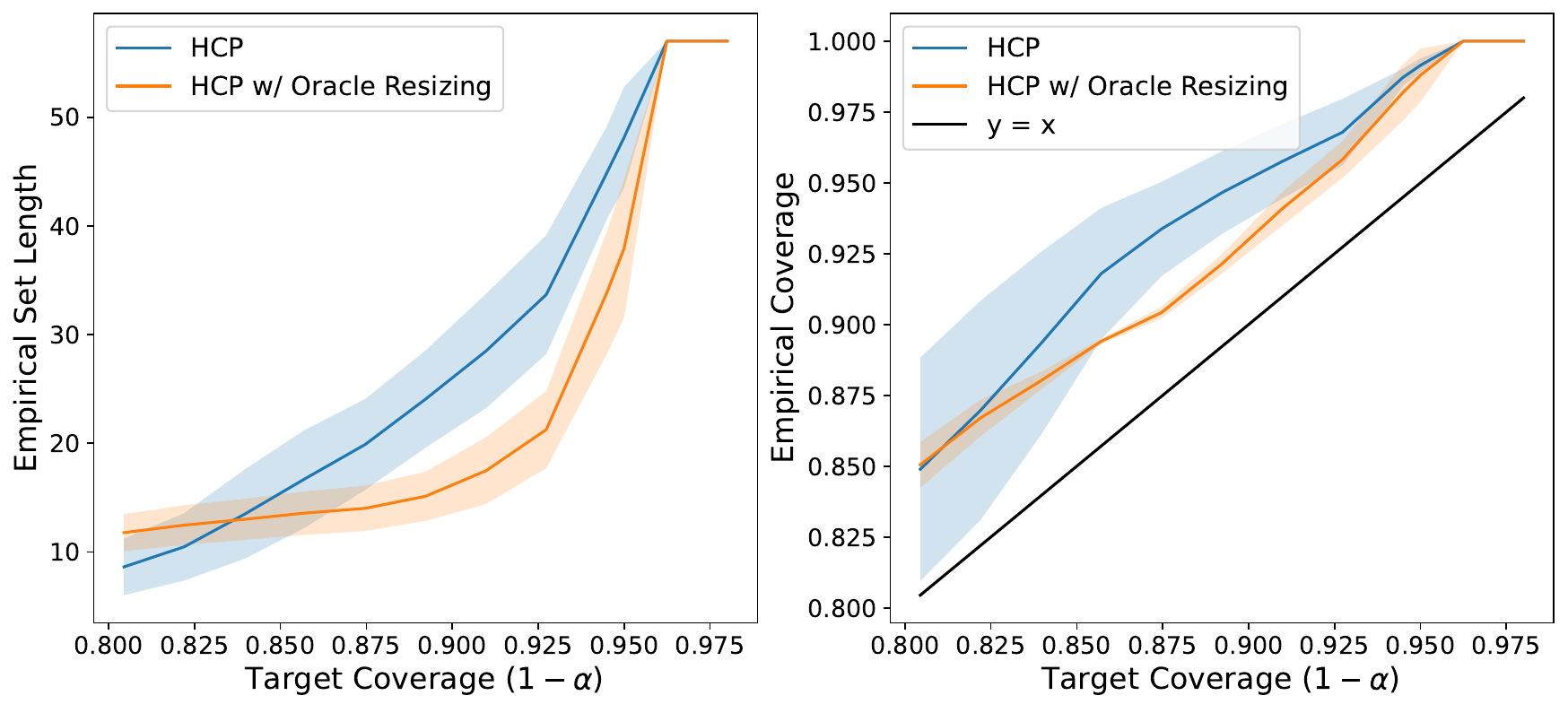}
\caption{Performance of HCP and resized HCP applied to the species classification data. Left plot shows the average set size over all test samples, and right plot shows the fraction of all test samples covered by their conformal sets.}
\label{fig:iwilds_resized_hcp}
\end{figure}


\section{General Consistency Results}
\label{sec:consistency-results}

In this section, we develop a theory of consistency for the
nested confidence sets in Section~\ref{sec:general-confidence}.
We define consistency
via convergence to a particular limiting confidence set mapping.
When this confidence set is in some sense optimal, then our
methods are consistent.
To that end, we shall assume there is a
fixed collection $\{C_\tau\}$ of nested confidence sets,
and we let $\env$ be a random environment drawn from the collection
of possible environments $\environments$,
letting $\P(\cdot \mid \env)$ denote the induced distribution
conditional on $\env$. 
For a fixed $\alpha, \delta \in (0, 1)$, we define
\begin{align*}
  &\tau\opt(\delta,\alpha)\\
  &\defeq \inf\left\{\tau ~ \mbox{s.t.}~
  \P\left(\P\left(Y \in C_\tau(X) \mid \env\right) \ge 1 - \alpha
  \right) \ge 1 - \delta \right\}.
\end{align*}
Thus, for the given confidence mappings $C = \{C_\tau\}$, the value
$\tau\opt(\delta, \alpha)$ yields the smallest confidence set providing
$(\delta,\alpha)$ coverage (Definition~\ref{definition:valid-coverage}) at a
population level.  Then for Algorithms~\ref{alg:nested-multi-env} and \ref{alg:nested-multi-env-split} to be
consistent, we must have the resulting confidence sets $\Cjack$ and $\Csplit$ approach
$C_{\tau\opt}$ in an appropriate sense.

\subsection{Conditions and preliminary definitions}
\label{sec:consistency_prelim_def}

We require a few more definitions, providing examples as we state them.
For a given nested confidence family $C = \{C_\tau\}$, we
define the coverage threshold
\begin{equation}
  \label{eqn:coverage-threshold}
  \tau(x, y, C) \defeq \inf\left\{\tau \in \R \mid y \in C_\tau(x)\right\}
\end{equation}
to be the smallest value $\tau$ such that $C_\tau(x)$ covers $y$.\\

\begin{example}[Coverage thresholds]
  The coverage threshold~\eqref{eqn:coverage-threshold}
  has straightforwardly computable values for most of the
  ``standard'' cases of confidence sets we consider.
  For the interval-based confidence set
  $C_\tau(x) = [f(x) - \tau, f(x) + \tau]$, evidently
  $\tau(x, y, C) = |f(x) - y|$.
  For the lower/upper sets we use
  in quantile regression (recall Section~\ref{sec:quantile}),
  where
  $C_\tau(x) = [l(x) - \tau, u(x) + \tau]$, we have
  $\tau(x, y, C) = \max\{l(x) - y, y - u(x)\}$,
  so that $\tau(X_i^j, Y_i^j, \what{C}^{-i})
  = R_j^i$ is simply the residual.
  For the generic loss-based set
  $C_\tau(x) = \{y \mid \loss(y, f(x)) \le \tau\}$
  (recall Section~\ref{sec:general-losses}), it
  is immediate that $\tau(x, y, C) = \loss(y, f(x))$,
  and so once again
  we have equality with the residuals
  $R_j^i = \tau(X_i^j, Y_i^j, \what{C}^{-i})$.
\end{example}

To discuss convergence of the estimators in general, we must address
modes of quantile convergence, which we do at a generic level
via convergence in distribution. To that end,
for an $\R$-valued random variable $Z$, let
\begin{equation*}
  \law(Z \mid P)
\end{equation*}
denote the induced probability law of $Z$ under
the probability distribution $P$. 
For example, letting
$\what{P}^i$ be the empirical distribution
of $(X_j^i, Y_j^i)_{j = 1}^{n_i}$,
then
the empirical distribution
of the values $\tau(X_j^i, Y_j^i, \what{C}^{-i})$ becomes
$\law(\tau(X, Y, \what{C}^{-i}) \mid \what{P}^i)$.
We recall the
\emph{bounded Lipschitz metric} between distributions $P$ and $Q$,
\begin{equation*}
  \bl{P - Q} \defeq \sup_{\linf{f} \le 1,
    \lipnorm{f} \le 1} \int f (dP - dQ),
\end{equation*}
noting that for any distribution $P$ on $\R^d$, if $\what{P}_n$ denotes the
empirical distribution of $Z_i \simiid P$, $i = 1, \ldots, n$, then
$\bls{\what{P}_n - P} \cas 0$ (see,
e.g., van der Vaart and Wellner \cite[Chs.~1.10--1.12]{VanDerVaartWe96}).

An essentially standard lemma relates convergence in
the bounded Lipschitz metric to quantiles;
for completeness we include a proof
in Appendix~\ref{sec:proof-quantiles-from-bl}.
\begin{lemma}
  \label{lemma:quantiles-from-bl}
  Let $Q$ be a distribution on $\R$ with
  $\alpha$-quantile $\quant_\alpha(Q)$ such that
  if $Z \sim Q$, then
  \begin{equation*}
    Q(Z \le \quant_\alpha(Q) - u) < \alpha
    ~~ \mbox{and} ~~
    Q(Z \le \quant_\alpha(Q) + u) > \alpha
  \end{equation*}
  for all $u > 0$. Then the quantile mapping $\quant_\alpha$ is
  continuous
  at $Q$ for the bounded-Lipschitz metric: for each $\epsilon > 0$,
  there exists $\delta > 0$ such that $\bl{P - Q} \le \delta$ implies
  $|\quant_\alpha(P) - \quant_\alpha(Q)| \le \epsilon$.
\end{lemma}

With these definitions, we make a few assumptions on the convergence of the
estimated families of nested confidence sets.  We will then
revisit the major examples of confidence sets we
have considered---the symmetric sets in the multi-environment
jackknife-minmax (Section~\ref{sec:basic-method}), the loss-based sets in
Section~\ref{sec:general-losses}, and the quantile-type sets in
Section~\ref{sec:quantile}---and show natural sufficient conditions for the
assumptions to hold.

The first two relate to consistency and continuity of the nested confidence
sets.

\stepcounter{assumption}
\begin{intassumption}
  \label{assumption:convergence-of-confidences}
  Fix environment $i$. As $n \to \infty$,
  \begin{equation*}
    \bl{\law(\tau(X, Y, C) \mid P^i)
      - \law(\tau(X, Y, \what{C}^{-i}) \mid \what{P}^i)} \cas 0,
  \end{equation*} where $\what{C}^{-i} = \{\widehat{C}_\tau^{-i}\}_{\tau \in \mathbb{R}}$ is defined in Algorithm~\ref{alg:nested-multi-env}.
\end{intassumption}
\begin{intassumption}
  \label{assumption:convergence-of-confidences-split}
  Fix any validation environment $E = i$. As $n \rightarrow \infty$,
  \begin{equation*}
    \bl{\law(\tau(X, Y, C) \mid P^i)
      - \law(\tau(X, Y, \what{C}^{D_1}) \mid \what{P}^i)} \cas 0,
  \end{equation*}
  where $\what{C}^{D_1} = \{\what{C}^{D_1}_\tau\}_{\tau \in \mathbb{R}}$ is defined in Algorithm~\ref{alg:nested-multi-env-split}.
\end{intassumption}

\stepcounter{assumption}
\begin{intassumption}
  \label{assumption:continuity-confidence}
  Let $\lambda$ be a measure on $\mc{Y}$, fix $\tau\opt \in \R$, and
  let $\epsilon > 0$. Define the (random) subsets
  \begin{equation*}
    B^i_{n,\tau} \defeq \left\{x \in \mc{X} ~ \mbox{s.t.}~
    \lambda\left(\what{C}_{\tau}^{-i}(x) \setdiff
    C_{\tau\opt}(x)\right) \ge \epsilon\right\},
  \end{equation*}
  indexed by $n \in \N, \tau \in \R$, and $i \le m$, of $\mc{X}$.
  Let $\tau(n)$ be such that
  $\lim_{n \to \infty} \tau(n) = \tau\opt$.
  Then for suitably slowly growing $m = m(n) \to \infty$,
  the $X$-measure of $B_{n,\tau} \defeq
  \cup_{i = 1}^m B_{n,\tau}^i$ satisfies
  $\lim_{n \to \infty} P_X(B_{n,\tau(n)}) = 0$
  with probability 1.
\end{intassumption}
\begin{intassumption}
  \label{assumption:continuity-confidence-split}
  Let $\lambda$ be a measure on $\mc{Y}$, fix $\tau\opt \in \R$, and
  let $\epsilon > 0$. Define the (random) subsets
  \begin{equation*}
    \Bsplit \defeq \left\{x \in \mc{X} ~ \mbox{s.t.}~
    \lambda\left(\what{C}^{D_1}_{\tau}(X) \setdiff
    C_{\tau\opt}(X)\right) \ge \epsilon\right\},
  \end{equation*}
  indexed by $n \in \N$ and $\tau \in \R$, of $\mc{X}$.
  Let $\tau(n)$ be such that
  $\lim_{n \to \infty} \tau(n) = \tau\opt$.
  Then
  the $X$-measure of $\Bsplit$ satisfies
  $\lim_{n \to \infty} P_X(\Bsplitn) = 0$
  with probability 1.
\end{intassumption}
\noindent
These give a type of consistency of the confidence set
mappings $\what{C}^{-i}$ and $\widehat{C}^{D_1}$ to $C$.

\subsubsection{Examples realizing the assumptions}

A few examples may make it clearer that we
expect Assumptions~\ref{assumption:convergence-of-confidences} and
\ref{assumption:continuity-confidence} to hold. Similar arguments apply for Assumptions~\ref{assumption:convergence-of-confidences-split} and
\ref{assumption:continuity-confidence-split}, and thus are omitted.
Throughout the examples, we make the standing assumption that
if $\pi$ is the (prior) distribution
on environments and $P_X = \int P^e_X d\pi(e)$ is the marginal
distribution over $X$, then
\begin{equation*}
  \sup_{e \in \environments} \dchi{P^e_X}{P_X} \le
  \maxdivergence^2
  < \infty.
\end{equation*}
In this case, for any $g : \mc{X} \to \R$, we have
\begin{align*}
  & \left|\int g (dP^e - dP)\right|= \left|\int g \left(\frac{dP^e}{dP} - 1\right) dP\right|\\
  & \qquad \le \sqrt{\E_P[g^2]} \sqrt{\dchi{P^e}{P}}  \le  \sqrt{\E_P[g^2]} \maxdivergence;
\end{align*}
in particular, $|P^e(A) - P(A)|
\le \sqrt{P(A)} \maxdivergence$ for $A \subset \mc{X}$.

We now give three examples: regression, quantile regression, and
(multiclass) logistic regression, and for each we provide a simple
sufficient condition for
Assumptions~\ref{assumption:convergence-of-confidences}
and~\ref{assumption:continuity-confidence} to hold. For all three examples, the sufficient condition assumes the leave-one-out predictions converge uniformly on compact sets.
The arguments are technical and do not particularly impact the main thread
of the paper, so we defer formal proofs to appendices.

\begin{example}[The regression case]
  \label{example:regression-consistency}
  In the case that we perform regression
  as in Section~\ref{sec:basic-method},
  we assume the existence of a population function
  $f(x) = \E[Y \mid X = x]$, where the expectation is taken
  across environments, and $\what{f}_{-i} \to f$ for each $i$.
  We
  assume this convergence is nearly uniform on compact sets:
  for each $\epsilon > 0$, there exists a
  subset $\mc{X}_\epsilon$ such that
  \begin{equation}
    \label{eqn:locally-uniform-f-convergence}
    \sup_{x \in \mc{X}_\epsilon} |\what{f}_{-i}(x) - f(x)|
    \cas 0
    ~~ \mbox{and} ~~
    P(\mc{X}_\epsilon) \ge 1 - \epsilon.
  \end{equation}
  Let $\lambda$ be Lebesgue measure.
  Then the locally uniform convergence
  condition~\eqref{eqn:locally-uniform-f-convergence} implies
  Assumptions~\ref{assumption:convergence-of-confidences}
  and~\ref{assumption:continuity-confidence}
  for $\lambda$;
  see Appendix~\ref{sec:proof-regression-consistency} for
  the argument.
\end{example}

Uniform convergence on compact sets is not a particularly onerous condition
for standard problems. For example, for a linear regression model, if the
estimate for model coefficients converges almost surely to the model
coefficients, then assumption~\eqref{eqn:locally-uniform-f-convergence}
holds trivially.

\begin{example}[Quantile regression]
  \label{example:quantile-consistency}
  In the case of quantile-type regression problems, we recall
  Section~\ref{sec:quantile}, and we assume the
  consistency conditions that
  \begin{equation*}
    l(x) = \quant_{\alpha/2}(Y \mid X = x)
    ~ \mbox{and} ~
    u(x) = \quant_{1 - \alpha/2}(Y \mid X = x),
  \end{equation*}
  and that for each $x$, $Y \mid X = x$ has a positive density, so that $l$
  and $u$ are unique. Then in analogue to
  condition~\eqref{eqn:locally-uniform-f-convergence}, we
  assume that for each $\epsilon > 0$, there exists
  $\mc{X}_\epsilon$ such that
  \begin{equation}
    \label{eqn:locally-uniform-u-l-convergence}
    \sup_{x \in \mc{X}_\epsilon}
    \max\left\{|\what{l}_{-i}(x) - l(x)|,
    |\what{u}_{-i}(x) - u(x)|\right\} \cas 0
  \end{equation} and 
   \begin{equation*}
    P_X(\mc{X}_\epsilon) \ge 1 - \epsilon.
  \end{equation*}
  As in Example~\ref{example:regression-consistency},
  if $\lambda$ is Lebesgue measure, then
  the convergence~\eqref{eqn:locally-uniform-u-l-convergence}
  implies Assumptions~\ref{assumption:convergence-of-confidences}
  and~\ref{assumption:continuity-confidence}
  for $\lambda$. See Appendix~\ref{sec:proof-quantile-consistency}
  for a proof.
\end{example}

\begin{example}[Classification and logistic regression]
  \label{example:classification-consistency}
  We consider a $k$-class logistic regression problem, where we
  take the loss
  \begin{equation*}
    \loss(y, v) = \log\left(\sum_{i = 1}^k e^{v_i - v_y}\right)
    = \log\bigg(1 + \sum_{i \neq y} e^{v_i - v_y}\bigg).
  \end{equation*}
  We assume the predictors $f : \mc{X} \to \R^k$
  are (Bayes) optimal in that they
  satisfy
  \begin{equation*}
    \frac{\exp([f(x)]_y)}{
      \sum_{i = 1}^k \exp([f(x)]_i)}
    = \P(Y = y \mid X = x).
  \end{equation*}
  As the loss $\loss$ is invariant to constant shifts we make
  the standing w.l.o.g.\ assumption that $f(x)$ and $\what{f}(x)$ are always
  mean zero (i.e.\ $\ones^T f(x) = 0$)
  and assume the consistency condition that for each
  $\epsilon > 0$, there exists $\mc{X}_\epsilon \subset \mc{X}$ such that
  \begin{equation}
    \label{eqn:locally-uniform-logreg-convergence}
    \sup_{x \in \mc{X}_\epsilon}
    \norm{\what{f}_{-i}(x) - f(x)} \cas 0
    ~~ \mbox{and} ~~
    P_X(\mc{X}_\epsilon) \ge 1-\epsilon.
  \end{equation}
  In this case, the uniqueness of quantile estimators requires a type of
  continuity condition that was unnecessary for the regression cases, and we
  make the additional continuity assumption that for each $c \in \R_+$ and
  $y \in [k]$, the set $\{x \in \mc{X} \mid \loss(y, f(x)) = c\}$ has
  measure zero.  As a consequence, the sets
  \begin{equation*}
    D_{c,\epsilon} \defeq \left\{x \in \mc{X} \mid
    \mbox{there~exists~} y ~ \mbox{s.t.} ~
    |\loss(y, f(x)) - c| < \epsilon \right\}
  \end{equation*}
  satisfy $\lim_{\epsilon \to 0} P_X(D_{c,\epsilon}) = 0$ for each
  $c$, by continuity of measure.

  An analogous argument to that in Examples~\ref{example:regression-consistency}
  and~\ref{example:quantile-consistency} then shows that
  the conditions above suffice for
  Assumptions~\ref{assumption:convergence-of-confidences}
  and~\ref{assumption:continuity-confidence} to hold with
  counting measure.
  See Appendix~\ref{sec:proof-classification-consistency}.
\end{example}

\subsection{Consistency of the multi-environment Jackknife-minmax}

With the motivating examples in place, we now provide  the
main convergence theorem. We state one final assumption, which
makes the environment-level quantiles sufficiently unique
that identifiability is possible.

\begin{assumption}
  \label{assumption:quantiles-by-environments}
  Define the quantile of the coverage threshold~\eqref{eqn:coverage-threshold}
  on environment $i$ by
  \begin{equation*}
    \quant_{1 - \alpha}(P^i)
    \defeq \inf\left\{t \mid P^i(\tau(X, Y, C) \le t) \ge 1 - \alpha
    \right\}.
  \end{equation*}
  For a given $\delta$, there exists a (unique) $q(\delta)$ such that for any $u > 0$,
  \begin{equation*}
    \P\left(\quant_{1 - \alpha}(P^\env) \le q(\delta) - u\right) < 1 - \delta,
     \end{equation*} and
  \begin{equation*}
    \P\left(\quant_{1 - \alpha}(P^\env) \le q(\delta) + u\right) > 1 - \delta,
  \end{equation*}
  where the probability is taken over $\env$ drawn randomly
  from $\environments$.
\end{assumption}

Given this assumption, we can show that the $X$-measure
of sets where the multi-environment jackknife-minmax and
``true'' confidence set $C_{\tau\opt(\delta,\alpha)}$ make
different predictions converges to zero.

\begin{theorem}
  \label{theorem:general-consistency}
  Let $\lambda$ be a measure on $\mc{Y}$ such that
  Assumptions~\ref{assumption:convergence-of-confidences},
  \ref{assumption:continuity-confidence}, and
  \ref{assumption:quantiles-by-environments} hold. Let $\epsilon > 0$
  and $m = m(n)$ be a sufficiently slowly growing sequence.
  Then the $P_X$ measure of
  \begin{align*}
   & \what{B}(\epsilon) \defeq\\
   & \left\{x \in \mc{X}
    ~~ \mbox{such that} ~~
    \lambda\left(\Cjack(x) \setdiff C_{\tau\opt(\delta, \alpha)}(x)
    \right) \ge \epsilon\right\}
  \end{align*}
  satisfies $P_X(\what{B}(\epsilon)) \cas 0$.
\end{theorem}
\begin{proof}
  The key step in the argument is to recognize that
  Assumptions~\ref{assumption:convergence-of-confidences}
  and~\ref{assumption:quantiles-by-environments} give
  consistency of the estimated threshold $\what{\tau}$:
  \begin{lemma}
    \label{lemma:consistency-of-delta-quantile}
    Let Assumption~\ref{assumption:convergence-of-confidences} hold and
    \ref{assumption:quantiles-by-environments} hold for the choice
    $\delta$. Then
    for all suitably slowly growing sequences
    $m = m(n) \to \infty$, the global
    estimated threshold $\what{\tau} \defeq
    \quantplus[\delta](\{S^i_{1-\alpha}\}_{i=1}^m)$
    in Algorithm~\ref{alg:nested-multi-env} converges almost surely:
    $\what{\tau} \cas q(\delta)$.
  \end{lemma}
  \noindent
  We defer this proof to
  Section~\ref{sec:proof-consistency-of-delta-quantile},
  continuing with the main thread of the theorem.
  
  Let $m = m(n)$ be a sequence tending to $\infty$ but such that the
  conclusions of Lemma~\ref{lemma:consistency-of-delta-quantile} hold.  Let
  $\what{\tau} = \quantplus[\delta](\{S^i_{1-\alpha}\}_{i=1}^m)$ be the
  random $\delta$-quantile in Alg.~\ref{alg:nested-multi-env}.
  Lemma~\ref{lemma:consistency-of-delta-quantile} guarantees that
  $\what{\tau} \cas q(\delta)$. Following a slight variation of the notation
  of Assumption~\ref{assumption:continuity-confidence}, define the sets
  $B_{n}^i = \{x \mid \lambda(\what{C}^{-i}_{\what{\tau}}(x) \setdiff
  C_{\tau\opt(\delta,\alpha)}(x)) \ge \epsilon\}$.
  Then
  \begin{align*}
    \left\{x \mid
    \lambda\left(\Cjack(x) \setdiff C_{\tau\opt(\delta, \alpha)}(x)
    \right) \ge \epsilon
    \right\}
    & \subset \bigcup_{i = 1}^m B_n^i
  \end{align*}
  and Assumption~\ref{assumption:continuity-confidence}
  guarantees that the latter set has $P_X$-measure tending to zero.
\end{proof}

\subsection{Consistency of multi-environment split conformal prediction}

Similarly, we can show that the $X$-measure of sets where the split
conformal and ``true'' confidence set $C_{\tau^{\star}(\delta, \alpha)}$
make different predictions converges to zero.
 
\begin{theorem}
  \label{theorem:general-consistency-split}
  Let $\lambda$ be a measure on $\mc{Y}$ such that
  Assumptions~\ref{assumption:convergence-of-confidences-split},
  \ref{assumption:continuity-confidence-split}, and
  \ref{assumption:quantiles-by-environments} hold. Let $\epsilon > 0$
  and $m = m(n)$ be a sufficiently slowly growing sequence.
  Then the $P_X$ measure of
  \begin{align*}
    &\what{B}(\epsilon)\defeq\\
    & \left\{x \in \mc{X}
    ~~ \mbox{such that} ~~
    \lambda\left(\Csplit(x) \setdiff C_{\tau\opt(\delta, \alpha)}(x)
    \right) \ge \epsilon\right\}
  \end{align*}
  satisfies $P_X(\what{B}(\epsilon)) \cas 0$.
\end{theorem}
\begin{proof}
  As in Lemma~\ref{lemma:consistency-of-delta-quantile}, we let $m = m(n)$
  be a sequence tending to $\infty$ such that the global estimated threshold
  $\widehat{\tau}=\widehat{q}_\delta^{+}(\{S_{1-\alpha}^i\}_{i \in D_2})$ in
  Algorithm~\ref{alg:nested-multi-env-split} converges almost surely to
  $q(\delta)$. Use Assumption~\ref{assumption:continuity-confidence-split}
  to complete the proof.
\end{proof}


\section{Environment-conditional coverage}
\label{sec:environment-conditional-coverage}

As with most predictive inference methods, the coverage guarantees we have
provided are marginal: they hold on average over the draw of the $(m + 1)$th
environment.
One naturally therefore asks to what extent we might achieve conditional
coverage, where we can condition on the newest environment.
Well-known impossibility results~\cite{BarberCaRaTi21a, Vovk12} suggest this
is likely to be challenging while providing distribution-free coverage; for
example, Vovk~\cite{Vovk12} shows that if an inferential method for
prediction $Y \in \R$ from $X$ drawn from a non-atomic measure guarantees $1
- \alpha$ conditional coverage, i.e., $\P(Y \in C(x) \mid X = x) \ge 1 -
\alpha$, then for almost every $x$, $C(x)$ has infinite Lebesgue measure
with probability at least $1 - \alpha$.
Nonetheless, weaker forms of conditional coverage, such as weighted versions
of marginal coverage, are possible~\cite{GibbsChCa23},
and we might expect that these ideas should
translate to multi-environment settings.
Here, we show that both these intuitions hold.

To present the results, we first define environment-conditional
coverage.
Let $\environments$ be the collection of possible environments, $\mc{X}$
the covariate space as usual, and let $\what{C} : \mc{X} \times \environments \toto
\mc{Y}$ a set-valued predictor.
Assume the hierarchical sampling
model~\eqref{eqn:exchangeable-generation}--\eqref{eqn:iid-within-environment}
that environments $E \in \environments$ are drawn i.i.d., and conditional on
$E = e$, we observe $(X_i, Y_i) \simiid P^e$.
Then $\what{C}$ provides \emph{distribution free
environment-conditional $(\alphaenv, \alpha)$-coverage} if
\begin{equation*}
  \P\bigg(\frac{1}{n} \sum_{j = 1}^n \indic{Y_j
    \in \what{C}(X_j, e)} \ge 1 - \alpha \mid E = e\bigg)
  \ge 1 - \alphaenv
\end{equation*}
when $(X_j, Y_j) \simiid P^e$, where the probability is also 
over any sampling used to construct $\what{C}$.
We shall see that this is, essentially, impossible.

\subsection{The impossibility of conditional coverage}

Because of the parallels between our multi-environment split conformal
guarantees and classical split-conformal predictive guarantees---the results
in Theorems~\ref{theorem:main-coverage-split}
and~\ref{theorem:nested-coverage-split} parallel the classical cases---we
expect to obtain certain impossibility results for environment-conditional
coverage.
Non-atomicity plays an essential role in demonstrations that conditional
coverage fails~\cite{BarberCaRaTi21a, Vovk12}, but it is interesting to
consider cases with discrete environments.

Thus, we begin by extending Vovk's impossibility results~\cite{Vovk12}
to non-atomic cases for classical distribution-free prediction.
We focus on the case that $Y \in \R$, as it provides the simplest statements.
Temporarily considering the case
that $(X_i, Y_i)_{i = 1}^n \simiid P$ for a fixed distribution $P$, let
$\what{C}_n : \mc{X} \toto \R$ be a set-valued mapping and function of
sample, so that we tacitly mean
\begin{equation*}
  \what{C}_n(x) = \what{C}_n(x \mid X_1^n, Y_1^n) \subset \R.
\end{equation*}
Then $\what{C}_n$ achieves distribution-free conditional $(1 - \alpha)$
coverage if for $P$-almost-all $x$,
\begin{equation*}
  \P(Y_{n + 1} \in \what{C}_n(X_{n + 1}) \mid X_{n + 1} = x) \ge 1 - \alpha
\end{equation*}

We restrict attention to measures for which the marginal distribution $P_X$
is appropriately non-atomic:
\begin{definition}
  \label{definition:indiscrete}
  A probability measure $P$ on $\mc{X}$ is $(\epsilon,
  \delta)$-indiscrete if for all sets $F \subset \mc{X}$ with $P(F) \ge
  \delta$, there exists $G \subset F$ such that $0 < P(G) \le
  \epsilon$.
\end{definition}
\noindent
For example, if $P$ has no atoms and $\mc{X}$ is a metric space, then it is
$(\epsilon, \delta)$-indiscrete for all $\epsilon, \delta > 0$ (see Lemma 2
in the extended version of the paper~\cite{Vovk12}).
If we let $\lebesgue$ denote Lebesgue measure, we obtain the following
impossibility result. (See Appendix~\ref{sec:proof-indiscrete-impossibility} for
the proof.)
\begin{proposition}
  \label{proposition:indiscrete-impossibility}
  Let $P$ be a probability distribution for which $P_X$
  is $(\epsilon / n, \delta)$-indiscrete. Assume that $\what{C}_n$
  achieves $1 - \alpha$ distribution-free conditional coverage.
  Then for any $F \subset \mc{X}$
  with $P_X$-probability $P_X(F) \ge \delta$, there exists
  $x \in F$ for which
  \begin{equation*}
    P^n(\lebesgue(\what{C}_n(x)) = +\infty) \ge 1 - \alpha - 2 \sqrt{\epsilon}.
  \end{equation*}
\end{proposition}

This result makes some intuitive sense:
if the marginal distribution on $X$ is non-atomic ``enough,'' that is,
at a level that scales as $1/n$, then there are $x$ values we
are unlikely to observe.
For these $x$, we expect that coverage should be hard. The
proposition makes explicit that this hardness is substantial.
We can also recover a slightly stronger
version of~\cite[Proposition 4]{Vovk12}:
\begin{corollary}
  \label{corollary:recover-vovk-impossibility}
  Let $\what{C}_n$ achieve $(1 - \alpha)$-distribution-free
  conditional coverage. Then
  \begin{enumerate}[i.]
  \item if $P_X$ is $(\epsilon/n, \delta)$-indiscrete, then
    if $F \subset \mc{X}$ is measurable and each $x \in F$
    satisfies
    $\P(\lebesgue(\what{C}_n(x)) = +\infty) < 1 - \alpha - 2 \sqrt{\epsilon}$,
    then $F$ has probability $P_X(F) < \delta$.
  \item if $P_X$ is non-atomic and $\mc{X}$ is a metric space,
    \begin{equation*}
      \P(\lebesgue(\what{C}_n(x)) = +\infty) \ge 1 - \alpha
    \end{equation*}
    for $P_X$-almost all $x \in \mc{X}$.
  \end{enumerate}
\end{corollary}
\begin{proof}
  Let $F$ be a measurable set of $x$ such that
  $P(\lebesgue(\what{C}_n(x)) = + \infty) < 1 - \alpha - 2
  \sqrt{\epsilon}$.
  Then if $P_X(F) \ge \delta$,
  Proposition~\ref{proposition:indiscrete-impossibility}
  would imply the existence of some $x \in F$ such that
  $P(\lebesgue(\what{C}_n(x)) = +\infty) \ge 1 - \alpha - 2 \sqrt{\epsilon}$.
  This is a contradiction to the definition of $F$.

  For the second part, recall that any non-atomic
  $P_X$ is $(\epsilon, \delta)$-indiscrete for $\epsilon, \delta > 0$,
  and take limits.
\end{proof}

\subsection{The impossibility of environment-conditional coverage}

We mimic the argument to prove
Proposition~\ref{proposition:indiscrete-impossibility} to extend
it to environment-conditional coverage, where we have
predictors $\what{C} : \mc{X} \times \environments \toto \R$.
To make the statement and proof of the result simpler, we assume that for
each $e \in \environments$ the predictive sets $\what{C}(x, e)$ are
\emph{homoskedastic}, meaning that for each $e$, the Lebesgue measure of
$\what{C}(x, e)$ is constant as $x$ varies.
For example, in homoskedastic regression~\cite{LeiGSRiTiWa18},
predictive intervals are symmetric sets $\what{C}(x, e) = [f(x) -
  \what{\sigma}, f(x) + \what{\sigma}]$, where $\what{\sigma} \ge 0$ may
depend on $e$; we have $\lebesgue(\what{C}(x, e)) = 2 \what{\sigma}$ for all
$x$.

In this case, we can extend
Proposition~\ref{proposition:indiscrete-impossibility}.
We assume the sampling model~\eqref{eqn:iid-within-environment}
that environments $E \in \environments$ are
i.i.d., and conditional on $E = e$, we observe $(X_j, Y_j) \simiid P^e$.
Let $P_\env$ be the marginal distribution over $\env$.
\begin{proposition}
  \label{proposition:no-environment-conditional}
  Let $P$ be a probability distribution for which
  $P_E$ is $(\epsilon/m, \delta)$-indiscrete.
  Assume that $\what{C}$ achieves distribution free $(\alphaenv,
  \alpha)$-environment conditional coverage and is homoskedastic.
  Then for any $F \subset \environments$ with $P_E(F) \ge \delta$, there exists an
  environment $e \in F$ for which
  \begin{equation*}
    P^m(\lebesgue(\what{C}(x, e)) = +\infty) \ge 1 - \alphaenv
    - 2 \sqrt{\epsilon}.
  \end{equation*}
\end{proposition}
\noindent
See Appendix~\ref{sec:proof-no-environment-conditional} for the proof.

We immediately obtain the following corollary via the same argument as that
for Corollary~\ref{corollary:recover-vovk-impossibility}.
\begin{corollary}
  \label{corollary:your-coverage-sucks}
  Let $\what{C}$ be homoskedastic and achieve distribution free $(\alphaenv,
  \alpha)$-coverage.
  \begin{enumerate}[i.]
  \item If $P_E$ is $(\epsilon/m, \delta)$-indiscrete and
    $F \subset \environments$ is measurable
    and each $e \in F$ satisfies
    $\P(\lebesgue(\what{C}(x, e)) = +\infty) < 1 - \alphaenv - 2
    \sqrt{\epsilon}$, then $F$ has probability $P_E(F) < \delta$.
  \item If $P_E$ is non-atomic and the collection $\environments$ of environments
    is metrizable, then for $P_E$-almost all $e \in \environments$,
    \begin{equation*}
      \P(\lebesgue(\what{C}(x, e)) = +\infty) \ge 1 - \alphaenv.
    \end{equation*}
  \end{enumerate}
\end{corollary}

\subsection{Weighted environment-conditional coverage}
\label{sec:weighted-environment-coverage}

In spite of the negative results
Corollary~\ref{corollary:your-coverage-sucks} highlights, as in the case of
standard supervised problems of prediction $Y$ from $X$, we can extend the
split conformal procedure in Alg.~\ref{alg:multi-env-split} to guarantee
weighted forms of coverage, which approximate conditional coverage in certain
limits.
We
adopt a perspective that Gibbs et al.~\cite{GibbsChCa23} and Areces et
al.~\cite{ArecesChDuKu24} advocate, seeking weighted per-environment coverage. We
focus on the general confidence set setting in
Section~\ref{sec:general-confidence}, and for simplicity consider the case
that we have a fixed collection $\{C_\tau\}_{\tau \in \R}$ of confidence set
mappings independent of the samples $(E_i, \{X_j^i, Y_j^i\}_{j =
  1}^{n_i})_{i = 1}^{m + 1}$, which are
exchangeable~\eqref{eqn:exchangeable-arbitrary-sample}.

To set the stage, we revisit Gibbs et al.'s approach in the classic (single
environment) case, then extend it to multi-environment scenarios.
In the single environment case, $C : \mc{X} \toto \mc{Y}$
provides conditional coverage if and only if
\begin{equation*}
  \E[f(X) (\indic{Y \in C(X)} - (1 - \alpha))] \ge 0
\end{equation*}
for all nonnegative integrable $f$, and $\P(Y \in
C(X) \mid X) = 1 - \alpha$ a.s.\ iff
$\E[f(X)\indic{Y \in C(X)}]
= (1 - \alpha) \E[f(X)]$ for all integrable $f$.
By considering collections
$\mc{F}$ of functions $f : \mc{X} \to \R$, Gibbs et al.\ show how to make
$\E[f(X) (\indic{Y \in f(X)} - (1 - \alpha))]$ small relative to $f$.
Their
approach follows: assume that $\{(X_i, \scorerv_i)\}_{i = 1}^{n + 1}$ are
exchangeable random vectors, $\scorerv_i \in \R$, and let $\mc{G} \subset
\mc{X} \to \R$ be a vector space of functions. Let
\begin{equation*}
  \loss_\delta(t) \defeq (1 - \delta) \hinge{t} + \delta \hinge{-t}
\end{equation*}
be the quantile loss,
$\reg : \mc{G} \to
\R_+$ be a convex regularizer, and let $\reg'(g; f) = \lim_{t \downarrow 0}
\frac{1}{t}(\reg(g + t f) - \reg(g)$ be its directional derivative.
Then
\begin{equation}
  \label{eqn:gibbs-cherian-est}
  \what{g} \defeq \argmin_{g \in \mc{G}}
  \left\{\frac{1}{n + 1} \sum_{i = 1}^{n + 1} \loss_\delta(\scorerv_i
  - g(X_i))
  + r(g) \right\}
\end{equation}
satisfies the following:
\begin{corollary}[Gibbs et al.~\cite{GibbsChCa23}, Theorem 3]
  \label{corollary:gibbs-cherian-est}
  Let $(X_i, \scorerv_i)_{i = 1}^{n+1}$ be exchangeable.
  Then for all $f \in \mc{G}$, the empirical minimizer $\what{g}$ in
  Eq.~\eqref{eqn:gibbs-cherian-est} satisfies
  \begin{align*}
    & \E\left[f(X_{n + 1})\left( \indic{\what{g}(X_{n + 1}) > \scorerv_{n + 1}}
      - (1 - \delta) \right)\right] \\
    & \qquad \qquad \qquad \qquad \qquad
    = -\E[\reg'(\what{g}; f)] + \epsint,
  \end{align*}
  where $|\epsint| \le \E[|f(X_i)| \indic{\scorerv_i =
      \what{g}(X_i)}]$ and $\epsint$ is nonnegative if
  $f \ge 0$.
\end{corollary}

An immediate extension applies to the  guarantees of
environment coverage we consider, where the observed samples
are exchangeable~\eqref{eqn:exchangeable-arbitrary-sample}.
In this case, we extend the
definition~\eqref{eqn:coverage-threshold} of the coverage threshold,
and set
\begin{equation}
  \label{eqn:score-by-environment}
  \scorerv_i \defeq \inf\bigg\{\tau \in \R \mid
  \frac{1}{n_i} \sum_{j = 1}^{n_i} \indic{Y_j^i \in C_\tau(X_j^i)}
  > 1 - \alpha \bigg\}
\end{equation}
to be the empirical coverage threshold on environment $i$.  Then the pairs
$(E_i, \scorerv_i)_{i = 1}^{m + 1}$ are exchangeable, and substituting $E_i$
for $X_i$ in the definition~\eqref{eqn:gibbs-cherian-est} allows a coverage
guarantee.
To make the computations concrete
let $\scoreval \in \R$ and $\reg^*(\eta) =
\sup_{g \in \mc{G}} \{ \sum_{i = 1}^{m + 1} \eta_i g(E_i) - (m+1)\reg(g)\}$
be a scaled convex conjugate-like dual to $\reg$.
Then problem~\eqref{eqn:gibbs-cherian-est} (with $E_i$ substituting for $X_i$)
admits the parameterized dual~\cite[Eq.~(4.2)]{GibbsChCa23}
\begin{equation*}
  \eta(\scoreval) \in
  \argmax_{\eta \in \R^{m + 1}}
  \left\{
  \begin{array}{l}
    \sum_{i = 1}^m \eta_i \scorerv_i + \eta_{m + 1} \scoreval
    - \reg^*(\eta) \\
    \textup{s.t.}~ -\delta \le \eta_i \le 1 - \delta,
    ~ i \in [m + 1]
  \end{array}
  \right\}.
\end{equation*}
Gibbs et al.~\cite[Thm.~4]{GibbsChCa23} show that $\scoreval \mapsto
\eta_{m+1}(\scoreval)$ is non-decreasing, so
binary search efficiently finds the value $\scoreval\opt_{m +
  1} \defeq \sup\{\scoreval \mid \eta_{m+1}(\scoreval) < 1 - \delta\}$.
Combining Corollary~\ref{corollary:gibbs-cherian-est}
with~\cite[Prop.~4]{GibbsChCa23}, we arrive at the following algorithm.

\begin{center}
  \algbox{
    \label{alg:empirical-cond-coverage}
    Weighted environment coverage
  }{
    \textbf{Input:} samples $\{X_j^i, Y_j^i\}_{j = 1}^{n_i}$,
    $i \in [m]$, levels $(\alpha, \delta)$,
    nested confidence sets $\{C_\tau\}_{\tau \in \R}
    \subset \{\mc{X} \toto \mc{Y}\}$

    \textbf{Set} $\scorerv_i$ as in Eq.~\eqref{eqn:score-by-environment}

    \textbf{Find} $\what{\tau} \defeq \scoreval\opt_{m+1}(\delta)
    = \sup\{\scoreval \mid \eta_{m+1}(\scoreval)
    < 1 - \delta\}$

    \textbf{Return} confidence set mapping
    \begin{equation*}
      \what{C}(x) \defeq C_{\what{\tau}}(x)
    \end{equation*}
  }
\end{center}

Then we have the following corollary, where in the corollary,
we let $\what{\P}_i$ denote the empirical probabilities associated
with environment $E_i$.
\begin{corollary}
  \label{corollary:environment-conditional}
  The confidence set $\what{C}$ that Alg.~\ref{alg:empirical-cond-coverage}
  returns satisfies
  \begin{align*}
    & \E\left[f(E_{m + 1})
      \indic{\what{\P}_{m+1}(Y \in \what{C}(X)) > 1 - \alpha}
      \right]
    \\
    & \qquad\qquad \ge (1 - \delta)\E[f(E_{m+1})]
    -\E[\reg'(\what{g}; f)]
  \end{align*}
  for all $f \in \mc{G}$ with $f \ge 0$.
\end{corollary}
\begin{proof}
  By the KKT conditions associated with the primal
  problem~\eqref{eqn:gibbs-cherian-est} and its
  dual~\cite[cf.][Prop.~4]{GibbsChCa23},
  we have the inclusions
  \begin{align*}
    \lefteqn{\{\scorerv_{m + 1} < \what{g}(E_{m+1})\}
      \subset \{\eta_{m + 1}(\scorerv_{m + 1}) = -\delta\}} \\
    & \subset  \{\eta_{m + 1}(\scorerv_{m + 1}) < 1-\delta\}
    \subset 
    \{\scorerv_{m + 1} \le \what{g}(E_{m+1})\}.
  \end{align*}
  Because
  $\scoreval \mapsto \eta_{m+1}(\scoreval)$ is non-decreasing,
  the event that $\{\eta_{m+1}(\scorerv_{m+1}) < 1 - \delta\}$ implies
  $\scorerv_{m +1} \le \scoreval\opt_{m+1}(\delta)$,
  that is,
  $\what{\tau} = \scoreval\opt_{m + 1}(\delta)$ satisfies
  $\what{\P}_{m+1}(Y \in C_{\what{\tau}}(X)) > 1 - \alpha$,
  where we use the right continuity of $C_\tau$.
  So if $f \ge 0$, we have
  \begin{align*}
    \lefteqn{\E\left[f(E_{m+1})
        \indic{\what{\P}_{m+1}(Y \in C_{\what{\tau}}(X)) > 1 - \alpha}\right]}
    \\
    & \qquad \ge \E[f(E_{m+1}) \indic{\scorerv_{m + 1} \le \scoreval\opt_{m+1}(\delta)}] \\
    & \qquad \ge \E[f(E_{m+1}) \indic{\eta_{m+1}(\scorerv_{m+1}) < 1 - \delta}] \\
    & \qquad \ge \E\left[f(E_{m+1}) \indic{\what{g}(E_{m+1}) > \scorerv_{m + 1}}
      \right].
  \end{align*}
  Corollary~\ref{corollary:gibbs-cherian-est} then
  implies the desired result.
\end{proof}

If we entertain randomized confidence sets, then we can
extend Corollary~\ref{corollary:environment-conditional} to apply
to arbitrary $f \in \mc{G}$. Indeed, without proof we claim that
if we replace $\what{\tau}$ with the randomized mapping
\begin{equation*}
  \what{\tau}_U \defeq \sup\left\{\scoreval \mid \eta_{m+1}(\scoreval)
  \le U - \delta \right\}
\end{equation*}
where $U \sim \uniform[0, 1]$ independently of the data, then
\begin{align*}
  \lefteqn{\E\left[f(E_{m+1}) \indic{\what{\P}_{m+1}(Y \in C_{\what{\tau}_U}(X))
        > 1 - \alpha}\right]} \\
  & \qquad\qquad\qquad = (1 - \delta) \E[f(E_{m+1})]
  - \E[\reg'(\what{g}; f)]
\end{align*}
for any $f \in \mc{G}$. (See~\cite[Prop.~4.2]{GibbsChCa23}.)  In particular,
in simple cases where $\reg = 0$ and $\mc{G} = \{f(e) = \theta^\top
\phi(e)\}_{\theta \in \R^k}$ for some feature mapping $\phi$,
the randomized confidence set achieves
exact $\mc{G}$-weighted $(\alpha, \delta)$-coverage.




\section{Discussion and conclusions}

The challenge of maintaining predictive validity when distributions change
remains one of the core challenges in statistics and machine learning.  At
the most basic level, any claim of a study's external validity is a claim
that statistical conclusions remain valid in environments distinct---however
slightly---from the study's population~\cite{ImbensRu15}.  A growing
literature in machine learning also highlights the challenges of prediction
across environments. In some cases, those environments are obvious, arising
from distinct experimental conditions or measurement
devices~\cite{KohSaMaXiZhBaHuYaPhBeLeKuPiLeFiLi21} or from changing
populations, such as identifying pathologies from lung scans across
hospitals~\cite{ZechBaLiCoTiOe18}. In others, new environments arise even
when constructing new evaluation datasets that replicate original data
collection as exactly as possible~\cite{RechtRoScSh19, TaoriDaShCaReSc20}.

The approaches this paper outlines to predictive inference across
environments and in other hierarchical data collection scenarios should
therefore see wide application. One might argue that, given that
applications of learning algorithms always involve some distribution shift,
however mild, we should always employ some type of corrective measure to
attempt to maintain validity. Section~\ref{sec:resizing_algos} highlights a
key insight, which we believe is one of the main takeaways of this work:
measuring variance and the scale of uncertainty across environments is
essential for practicable confidence sets and predictions.

Most types of
predictive inference repose on some type of exchangeability~\cite{VovkGaSh05,
  TibshiraniBaCaRa19, ChernozhukovWuZh18, BarberCaRaTi21a, GuptaRo21}---excepting a
recent line of work moves toward coverage guarantees that hold
asymptotically irrespective of the
data~\cite{GibbsCa21,AngelopoulosCaTi23}---as does this work and others on
maintaining validity across populations~\cite{LeeBaWi23, DunnWaRa23}.
Our particular exchangability assumptions should also extend straightforwardly
to risk control across environments~\cite{AngelopoulosBa23}, though we
leave this for future work.
The
optimality and adaptivity of predictive inference procedures, such as the
techniques we develop in Section~\ref{sec:consistency-results}, also rely on
some type of independence and exchangeability.  Future research to identify
more nuanced ways in which data remain exchangeable could thus have
substantial impact, allowing us to enhance the versatility and utility of
conformal prediction, the jackknife, and other predictive inferential
approaches.

\begin{appendix}

\section{Coverage proofs}
\label{sec:technical-proofs}

\subsection{Proof of Theorem~\ref{theorem:main-coverage}}
\label{sec:proof-main-coverage}

We take as inspiration the proof of Barber et al.~\cite[Theorem 3]{BarberCaRaTi21}.  For
the sake of the proof only, to demonstrate the appropriate exchangeability,
we assume we have access to both the features and responses of the test
environment $\{(X^{m+1}_j,Y^{m+1}_j)\}_{j=1}^{n_{m+1}}$.  Then we let
$\wt{f}_{-(i,k)}$ be the predictive function fit on all
environments (training and test) except that environments $i$ and $k$
are removed, i.e.,
\begin{equation*}
  \wt{f}_{-(i,k)}
  = \alg\left(\{X^l, Y^l\}_{l \neq i, l \neq k}\right).
\end{equation*}
With this construction, $\wt{f}_{-(i,k)}=\wt{f}_{-(k,i)}$ for $i \neq k$, and
\begin{equation*}
  \wt{f}_{-(i,m+1)} = \what{f}_{-i}
  ~~ \mbox{for}~ i \in [m],
\end{equation*}
the key identity allowing us to exploit block exchangeability.
Define the matrix
$R \in
\R^{(m+1)\times(m+1)}$ of residual quantiles
\begin{equation}
  \label{eqn:quantile-residual-matrix}
  R_{ik}
  = \begin{cases} \infty & \mbox{if}~ i=k  \\ 
    \quantplus[n_i,\alpha]
    \left(\left\{
    |Y^{i}_j -\tilde{f}_{-(i,k)}(X^i_j)|\right\}_{j=1}^{n_i}\right)
    & \mbox{if~} i \neq k, \end{cases} 
\end{equation}
so $R_{i, m + 1} = S^i_{1 - \alpha}$ in Algorithm~\ref{alg:multi-env}.

Now~\citep[cf.][Proof of Theorem 3]{BarberCaRaTi21},
define the comparison matrix $A \in \{0,1\}^{(m+1)\times(m+1)}$ with entries
\begin{equation*}
  A_{ik} = \indic{\min_{k'} R_{ik'} > R_{ki}},
\end{equation*}
so that the smallest residual $(1 - \alpha)$ quantile when omitting
environment $i$ is larger than the $(1 - \alpha)$ quantile of residuals in
environment $k$ when not including $i$ or $k$.
Define the set
\begin{equation*}
  \mathcal{S}(A) =
  \left\{i \in \{1,\ldots,m+1\}: A_{i, \bullet}\geq (1-\delta)(m+1)\right\},
\end{equation*}
of strange environments,
where $A_{i, \bullet}= \sum_{k=1}^{m+1}A_{ik}$ is the $i$th row sum of $A$,
to be those where environment $i$ typically has ``too small'' residuals.

We identify three steps that together yield the proof.
\begin{description}
\item[Step 1.] Observe that $|\mathcal{S}(A)| \leq \delta (m+1)$.
\item[Step 2.] Using that the environments are exchangeable,
  the probability that the test environment $m+1$ is strange
  (i.e., $m+1 \in \mathcal{S}(A)$) satisfies
  \begin{equation}
    \P(m + 1 \in \mc{S}(A)) \le \delta.
    \label{eqn:strange-probability}
  \end{equation}
\item[Step 3.] Prove the desired coverage guarantee of the
  theorem by showing that if coverage in environment
  $m + 1$ fails, then it is strange, i.e.,
  $m + 1 \in \mc{S}(A)$.
\end{description}

The result in Step 1 immediately follows Barber et al.~\cite[Thm.~3, Step
  1]{BarberCaRaTi21}.  Step~2 similarly follows immediately~\cite[Thm.~3,
  Step 2]{BarberCaRaTi21}.  It remains to prove step 3.

\paragraph{Proof of Step 3:} Suppose that coverage fails,
that is,
\begin{equation}
  \label{eqn:coverage-failure}
  \begin{split}
    & \sum_{j=1}^{n_{m+1}}
    \indic{Y^{m+1}_j \in \Cjack(X^{m+1}_{j})} \\
    &\quad < \ceil{(1-\alpha)(n_{m+1}+1)}.
  \end{split}
\end{equation}
We will show that on the event~\eqref{eqn:coverage-failure},
environment $m + 1$ is strange, so that~\eqref{eqn:strange-probability}
the
failure has probability at most $\delta$.

Before we complete the proof, we provide two technical lemmas.
These lemmas  allow us to transition between
coverage guarantees ``columnwise,'' in the sense of within an
environment, and ``rowwise'' in the sense of across environments.

\begin{lemma}
  \label{lemma:quantile-comparison}
  Let $B \in \R^{n \times m}$ and $c \in \R$. Then
  \begin{equation*}
    \quantplus[n,\alpha]
    \left(\indicBig{c < \min_{k} B_{jk}}_{j = 1}^n\right)
    = 1
  \end{equation*}
  implies
    \begin{equation*}
    c < \min_{k} \quantplus[n,\alpha] \left(\{B_{jk}\}_{j = 1}^n\right).
  \end{equation*}
  That is, if fewer than $\ceil{(1 - \alpha)(n + 1)}$
  indicators $\indics{c < \min_{k} B_{jk}}$ are 0,
  then $c$ is less than each ($k = 1, 2, \ldots, m$)
  of the $(1 - \alpha)$ quantiles
  $\quantplus[n,\alpha] (\{B_{jk}\}_{j = 1}^n)$.
\end{lemma}
\begin{proof}
  We have $\what{q}^+_{n,\alpha} (\indics{c < \min_{k} B_{jk}}_{j = 1}^n) =
  1$ if and only if at least $n - \ceil{(1 -\alpha)(n + 1)}$ values
  in the collection
  $\{\indics{c < \min_k B_{jk}}\}_{j = 1}^n$
  are 1. In this case, there are at least $n -
  \ceil{(1 - \alpha)(n + 1)}$ rows $J_{\textup{dom}} \subset [n]$ in $B$
  satisfying $c < \min_k B_{jk}$ for $j \in J_{\textup{dom}}$.  Then for
  each column $k$, the indices $j \in J_{\textup{dom}}$ satisfy $B_{jk} >
  c$, and as $|J_{\textup{dom}}| \ge n - \ceil{(1 - \alpha)(n + 1)}$, the
  $(1 - \alpha)$ quantile satisfies
  \begin{equation*}
    \quantplus[n,\alpha](\{B_{jk}\}_{j = 1}^n)
    \ge \min_{j \in J_{\textup{dom}}} B_{jk} > c,
  \end{equation*}
  giving the lemma.
\end{proof}

\begin{lemma}
  \label{lemma:min-max-comparison}
  Let $B \in \{0, 1\}^{n \times m}$.
  Assume that there exist rows
  $J \subset [n]$ with $|J| \ge \floor{\alpha(n + 1)}$
  and columns
  $K \subset [m]$ such that
  $|K| \ge \ceil{(1 - \delta) m}$
  such that $b_{jk} = 1$ for $j \in J$ and $k \in K$.
  Then
  \begin{equation*}
    \sum_{k = 1}^m
    \what{q}^+_{n,\alpha}
    \left(\left\{B_{jk}\right\}_{j = 1}^n\right)
    \ge (1 - \delta) m.
  \end{equation*}
\end{lemma}
\begin{proof}
  Fix any column $k \in K$.
  Then at least $\floor{\alpha(n + 1)}$ elements
  of $\{B_{jk}\}_{j = 1}^n$ are 1 (those in $J$),
  so that
  $$\what{q}_{n,\alpha}^+(\{B_{jk}\}_{j = 1}^n)
  = 1$$ for these $k \in K$. Thus
  \begin{align*}
    \sum_{k = 1}^m
    \what{q}^+_{n,\alpha}
    \left(\left\{B_{jk}\right\}_{j = 1}^n\right)
    & \ge \sum_{k \in K}
    \what{q}^+_{n,\alpha}
    \left(\left\{B_{jk}\right\}_{j = 1}^n\right) \\
    & = |K| \ge (1 - \delta) m
  \end{align*}
  as desired.
\end{proof}

We return to the main thread.
On the event~\eqref{eqn:coverage-failure}, there exists a set
$J_\bad$ of indices where coverage
fails and $|J_\bad| \ge \floor{\alpha (n_{m + 1} + 1)}$, that is, such that
\begin{equation*}
  Y_j^{m+1}>\max_{i \in [m]}
  \what{f}_{-i}(X^{m+1}_{j})+ \quantplus[m,\delta](\{S^k_{1-\alpha}\} ),
\end{equation*} or 
\begin{equation*}
    Y_{j}^{m+1} < \min_{i \in [m]}\what{f}_{-i}(X^{m+1}_{j})
    - \quantplus[m,\delta](\{S^k_{1-\alpha}\})
\end{equation*}
for each $j \in J_\bad$.
We can be more precise about
the indices of $\{S^k_{1-\alpha}\}_{k = 1}^m$: by the definitions of the
quantiles $\quantplus[m,\delta]$, there exists an index set $K_\bad \subset
[m]$ such that $|K_\bad| \ge \ceil{(1 - \delta)(m + 1)}$ and for each $k \in
K_\bad$ and $j \in J_\bad$,
\begin{equation*}
  Y_j^{m+1} > \max_{i \in [m]}
  \what{f}_{-i}(X^{m+1}_{j}) + S^k_{1 - \alpha}
\end{equation*} or \begin{equation*}
  Y_{j}^{m+1} < \min_{i \in [m]}\what{f}_{-i}(X^{m+1}_{j})
  - S^k_{1 - \alpha}
\end{equation*}
by taking $K_\bad$ to be the order statistics of $S^k_{1-\alpha}$.

We now show that on event~\eqref{eqn:coverage-failure},
the environment $m + 1$ is strange, that is, $A_{m+1,\bullet}$ is large.
We apply Lemma~\ref{lemma:quantile-comparison},
making the substitutions
$c = R_{k, m + 1}$ and letting $B \in \R^{n_{m + 1} \times m}$ have
entries $B_{jk} = |Y_j^{m + 1} - \what{f}_{-k}(X_j^{m+1})|$:
\begin{align*}
  \lefteqn{\sum_{k=1}^{m+1} A_{m+1,k}
    = \sum_{k=1}^{m+1} \indic{R_{k,m+1}<\min_{k'}R_{m+1,k'}}} \\ 
  & = \sum_{k = 1}^{m + 1} \indic{R_{k,m+1}
    < \min_{k'} \quantplus[n_{m+1}, \alpha]\left(\{B_{jk'}\}_{j = 1}^{n_{m+1}}
    \right)} \\
  & \ge \sum_{k = 1}^{m + 1}
  \quantplus[n_{m+1}, \alpha] \left(\Big[\indicBig{R_{k, m+1}
    < \min_{k'} B_{jk'}}\Big]_{j = 1}^{n_{m+1}}\right).
%
\end{align*}

Finally, recall that by construction of the residual
quantile matrix~\eqref{eqn:quantile-residual-matrix} we have
$R_{k,m+1} = S^k_{1-\alpha}$. Then
$R_{k,m+1} < \min_{k'} |Y_j^{m + 1} - \what{f}_{-k'}(X_j^{m + 1})|$
if and only if
$S^k_{1 - \alpha} < \min_{k'} |Y_j^{m + 1} - \what{f}_{-k'}(X_j^{m + 1})|$,
which in turn occurs if and only if
\begin{equation*}
  Y_j^{m + 1} <  \min_{k'} \what{f}_{-k'}(X_j^{m + 1})
  - S^k_{1 - \alpha},
  \end{equation*}
or
\begin{equation*}
  Y_j^{m + 1} > \max_{k'} \what{f}_{-k'}(X_j^{m + 1})
  + S^k_{1 - \alpha}.
\end{equation*}
Revisiting the sum above, then, for
intervals
\begin{align*}
  C_{jk} \defeq & \Big[\min_{k'} \what{f}_{-k'}(X_j^{m+1}) - S_{1-\alpha}^k, \\
    & \qquad ~ 
    \max_{k'} \what{f}_{-k'}(X_j^{m + 1}) + S_{1-\alpha}^k \Big],
\end{align*}
we have
\begin{align*}
  A_{m+1,\bullet}
  & \ge
  \sum_{k=1}^{m+1}
  \quantplus[n_{m+1},\alpha]
  \left(\left[\indic{Y_j^{m + 1} \not \in C_{jk}}\right]_{j = 1}^{n_{m+1}}\right) \\
  &\geq (1-\delta)(m+1),
\end{align*}
where in the last line we used Lemma~\ref{lemma:min-max-comparison} with the
choice
$B_{jk} = \indics{Y_j^{m + 1} \not \in C_{jk}}$,
recognizing that $B_{jk} = 1$ for all indices $j \in J_\bad$ and
$k \in K_\bad$.
In particular, we have shown
that $m + 1 \in \mc{S}(A)$.
As $\P(m + 1 \in \mc{S}(A)) \le \delta$ by
step 2 (recall Eq.~\eqref{eqn:strange-probability}), we
have the theorem.

\subsection{Proof of Theorems~\ref{theorem:main-coverage-split}
  and \ref{theorem:nested-coverage-split}}
\label{sec:proof-main-coverage-split}

\providecommand{\scoreval}{s}
\providecommand{\scorerv}{S}
\providecommand{\bZ}{\boldsymbol{Z}}

Theorems~\ref{theorem:main-coverage-split}
and~\ref{theorem:nested-coverage-split} will follow
as specializations of a few technical lemmas, which
follow here.
In the lemma,
recall that for a collection
$\{\scoreval_i\}_{i = 1}^m$,
we let $\quantplus[m,\alpha](\{\scoreval_i\}_{i = 1}^m)$ be the
$\ceil{(m + 1) (1 - \alpha)}$th smallest value of the $\scoreval_i$.
Then
the following lemma is standard~\cite{VovkGaSh05, LeiGSRiTiWa18},
though we use a particular form.
\begin{lemma}[Romano et al.~\cite{RomanoPaCa19}, Lemma 2]
  \label{lemma:standard-exchangability}
  Let $\scorerv_i$, $i = 1, \ldots, m, m + 1$, be exchangeable. Then
  \begin{equation*}
    \P\left(\scorerv_{m + 1} \le \quantplus[m,\alpha](\{\scorerv_i\}_{i = 1}^m)
    \right)
    \ge 1 - \alpha,
  \end{equation*}
  and if $\scorerv_i$ are almost surely distinct, then
  \begin{equation*}
    \P\left(\scorerv_{m + 1} \le \quantplus[m,\alpha](\{\scorerv_i\}_{i = 1}^m)
    \right) \le 1 - \alpha + \frac{1}{m}.
  \end{equation*}
\end{lemma}

Now consider our setting of exchangeable samples from different
environments.
Let $\mc{Z}$ be an arbitrary measurable space, and let $\mc{Z}^* \defeq
\{\mc{Z}, \mc{Z}^2, \mc{Z}^3, \ldots\}$ be the collection of its Cartesian
products.
Assume we observe exchangeable $\bZ^i \in \mc{Z}^*$, $i = 1, \ldots, m + 1$,
and let $D_1, D_2$ be sets of indices chosen uniformly from $[m]$ satisfying
$\card(D_1) = m \gamma$ and $\card(D_2) = (1 - \gamma) m$.
Let $f : \mc{Z}^* \times (\mc{Z}^*)^{m\gamma} \to \R$ be measurable, where
for an index set $D \subset [m]$ we use the shorthand
\begin{equation*}
  f(\bZ \mid D)
  = f(\bZ; \{\bZ^i\}_{i \in D}).
\end{equation*}

Assign an arbitrary ordering $\pi_{D_2} : [(1 - \gamma) m] \to [m]$
to the elements of $D_2$, so that $\pi_{D_2}(i)$ is the $i$th index
of $D_2$ (this is just for notational convenience).
Define the score random variables
\begin{equation}
  \label{eqn:arbitrary-score-dist}
  \scorerv^i \defeq f(\bZ^{\pi_{D_2}(i)} \mid D_1),
  ~~ i = 1, \ldots, (1 - \gamma) m,
\end{equation}
and $\scorerv^{m + 1} = f(\bZ^{m + 1} \mid D_1)$.
Then by construction and the uniformity of the index sets $D_1, D_2$,
the random variables $\scorerv^i$ are exchangeable.
The following lemma thus follows as a consequence of
Lemma~\ref{lemma:standard-exchangability}:
\begin{lemma}
  \label{lemma:burrito}
  Let the scores $\scorerv^i$ have
  definition~\eqref{eqn:arbitrary-score-dist}.
  Then for any $\delta \in (0, 1)$,
  \begin{equation*}
    \P\left(\scorerv^{m + 1} > \quantplus[\delta]\left(
    \{\scorerv_i\}_{i = 1}^{(1 - \gamma) m}\right)\right) \le \delta.
  \end{equation*}
  If additionally $\{\scorerv^i\}$ are distinct with probability 1, then
  \begin{equation*}
    \P\left(\scorerv^{m + 1} > \quantplus[\delta]\left(
    \{\scorerv^i\}_{i = 1}^{
      (1 - \gamma) m}\right)\right)
    \ge \delta - \frac{1}{(1 - \gamma) m + 1}.
  \end{equation*}
\end{lemma}

\subsubsection{Proof of Theorem~\ref{theorem:main-coverage-split}}

We specialize Lemma~\ref{lemma:burrito} to the setting of
Theorem~\ref{theorem:main-coverage-split}.
Evidently the scores $\scorerv^i_{1 - \alpha}$ in
Alg.~\ref{alg:multi-env-split} are of the
form~\eqref{eqn:arbitrary-score-dist}.
Then observe that
\begin{equation*}
  Y_j^{m + 1} \in \Csplit(X_j^{m + 1})
\end{equation*}
if and only if the residual
\begin{equation*}
  R_j^{m + 1} \defeq |Y_j^{m + 1} - \what{f}_{D_1}(X_j^{m+1})|
  > \quantplus[\delta](\{\scorerv_{1 - \alpha}^i\}_{i \in D_2}).
\end{equation*}
Then
\begin{align*}
  \frac{1}{n_{m + 1}}
  \sum_{j = 1}^{n_{m + 1}}
  \indic{Y_j^{m + 1} \in \Csplit(X_j^{m+1})} < (1 - \alpha)
\end{align*}
if and only if at least $\alpha n_{m + 1}$ of the residuals $R_j^{m + 1}$
satisfy $R_j^{m + 1} > \quantplus[\delta]$.
But by construction of
$\scorerv_{1 - \alpha}^{m + 1} = \quantplus[\alpha](\{R_j^{m + 1}\}_{j =
  1}^{n_{m+1}})$, this implies that
\begin{equation*}
  \scorerv_{1 - \alpha}^{m + 1} > \quantplus[\delta]
  \left(\{\scorerv_{1 - \alpha}^i\}_{i \in D_2}\right).
\end{equation*}
Lemma~\ref{lemma:burrito} implies that $\P(\scorerv_{1 - \alpha}^{m + 1} >
\quantplus[\delta](\{\scorerv_{1 - \alpha}^i\}_{i \in D_2}))
\le \delta$, which is the first claim of the theorem.
The second (when the scores are distinct) follows from the second part of
Lemma~\ref{lemma:burrito}.

\subsubsection{Proof of Theorem~\ref{theorem:nested-coverage-split}}
\label{sec:proof-nested-coverage-split}

The proof is, \emph{mutatis mutandis}, identical to that of
Theorem~\ref{theorem:main-coverage-split} given Lemma~\ref{lemma:burrito}.
All we require is to notice that in Alg.~\ref{alg:multi-env}, we have
\begin{equation*}
  Y_j^{m + 1} \in \Csplit(X_j^{m + 1})
\end{equation*}
if and only if for $\what{\tau} = \quantplus[\delta](\{\scorerv^i_{1 -
  \alpha}\}_{i \in D_2})$, we have
\begin{equation*}
  Y_j^{m + 1} \in \what{C}_{\what{\tau}}^{D_1}(X_j^{m + 1}),
\end{equation*}
which holds if and only if $R_j^i \le \what{\tau}$, where this last
statement uses the right continuity of $C_\tau(x)$ in $\tau$.  The remainder
of the proof mimics that of Theorem~\ref{theorem:main-coverage-split}.

\subsection{Proof of Theorem~\ref{theorem:nested-coverage}}
\label{sec:proof-nested-coverage}

The proof is quite similar to that of Theorem~\ref{theorem:main-coverage},
with appropriate redefinitions of residual matrices and the $A$ matrix.  We
begin with the analogues of $\wt{f}_{-(i,k)}$ and the $R$
matrix~\eqref{eqn:quantile-residual-matrix}. Define
the collections of confidence set mappings
\begin{equation*}
  \left\{\wt{C}^{-(i,k)}_\tau\right\}_{\tau \in \R}
  = \alg\left(\{X^l, Y^l\}_{l \neq i, l \neq k}\right),
\end{equation*}
so that
\begin{equation*}
  \wt{C}^{-(i,m+1)}_\tau = \what{C}^{-i}_\tau
  ~~ \mbox{for~} i \in [m].
\end{equation*}
For collection $\{C_\tau\}_{\tau \in \R}$ of set-valued mappings $C_\tau :
\mc{X} \toto \mc{Y}$, define the minimal covering value
\begin{equation*}
  \tau(x, y, \{C_\tau\}_{\tau \in \R}) \defeq \inf\left\{
  \tau \mid y \in C_\tau(x) \right\}.
\end{equation*}

We can then define the $(1 - \alpha)$-quantile residual matrix
$R \in \R^{(m + 1) \times (m + 1)}$ with entries
\begin{equation}
  \label{eqn:nested-residual-matrix}
  R_{ik} =
  \begin{cases}
    +\infty & \mbox{if}~ i = k \\
    \quantplus[\alpha]\left(
    \left[\tau(X_j^i, Y_j^i, \{\wt{C}^{-(i,k)}_\tau\}_{\tau \in \R})\right]_{j = 1}^{n_i} \right)
    & \mbox{if}~ i \neq k,
  \end{cases}
\end{equation}
so that again $R_{i,m+1} = S^i_{1 - \alpha}$ in
Alg.~\ref{alg:nested-multi-env}, while $R_{m+1,i} = \quantplus([
\tau(X_j^{m+1}, Y_j^{m+1}, \{\what{C}^{-i}_\tau\})]_{j=1}^{n_{m+1}})$ gives
quantiles for coverage on the new environment $m + 1$.  We define the matrix
$A$ identically as in the proof of Theorem~\ref{theorem:main-coverage},
$A_{ik} = \indics{\min_{k'} R_{ik'} > R_{ki}}$ and the set of strange
environments $\mc{S}(A) = \{i \in [m + 1] \mid A_{i,\bullet} \ge (1 -
\delta)(m + 1)\}$ as before.

We again have that $|\mc{S}(A)| \le \delta (m + 1)$ and that $\P(m + 1 \in
\mc{S}(A)) \le \delta$, as in Eq.~\eqref{eqn:strange-probability}. We show
that if coverage in environment $m + 1$, fails, then environment $m + 1$ is
strange. To that end, suppose that
coverage fails, that is,
\begin{align}
  \label{eqn:nested-coverage-failure}
  &\sum_{j = 1}^{n_{m + 1}} \indic{Y_j^{m + 1}
    \in \Cjack(X_j^{m + 1})} \\
    &< \ceil{(1 - \alpha) (n_{m + 1} + 1)} \nonumber.
\end{align}
Recall that for the threshold $\what{\tau} =
\quantplus[\delta](\{S^i_{1-\alpha}\}_{i = 1}^m)$,
Algorithm~\ref{alg:nested-multi-env} sets $\Cjack(x) = \cup_{i = 1}^m
\what{C}_{\what{\tau}}^{-i}(x)$. Then on the
event~\eqref{eqn:nested-coverage-failure}, there necessarily exists a set
$J_\bad$, $|J_\bad| \ge \floor{\alpha(n_{m + 1} + 1)}$, such that coverage
fails for examples $X_j^{m + 1}$ whose indices $j \in J_\bad$:
\begin{equation*}
  Y_j^{m + 1} \not \in \bigcup_{i = 1}^m \what{C}^{-i}_{\what{\tau}}(X_j^{m + 1})
  ~~ \mbox{for~} j \in J_\bad.
\end{equation*}
By definition of $\what{\tau}$ as the
quantile $\quantplus[\delta](\{S^i_{1-\alpha}\})$, then, we
also see that there exists a set $K_\bad \subset [m]$ with
cardinality $|K_\bad| \ge \ceil{(1 - \delta)(m + 1)}$ and for
which if $k \in K_\bad$ and $j \in J_\bad$, we have
\begin{equation*}
  Y_j^{m + 1} \not \in \bigcup_{i = 1}^m \what{C}^{-i}_{S^k_{1-\alpha}}(X_j^{m + 1}).
\end{equation*}

With these equivalences of failing to cover, we replicate the chain of
inequalities in the proof of Theorem~\ref{theorem:main-coverage}.
For shorthand, define $B_{jk} = \tau(X_j^{m + 1}, Y_j^{m + 1},
\{\what{C}_\tau^{-k}\}_{\tau \in \R})$
for $k = 1, \ldots, m + 1$ to be the minimal values for coverage
of example $(X_j^{m + 1}, Y_j^{m + 1})$ by the predictive sets fit
without environment $k$.
Assuming event~\eqref{eqn:nested-coverage-failure} occurs,
we then have
\begin{align*}
  A_{m+1,\bullet}
  & = \sum_{k = 1}^{m+1} \indic{R_{k,m+1} < \min_{k'} R_{m+1,k'}} \\
  & = \sum_{k = 1}^{m + 1} \indic{R_{k, m+1} < \min_{k'}
    \quantplus[\alpha]([B_{jk'}]_{j = 1}^{n_{m + 1}})
    } \\
  & \ge
  \sum_{k = 1}^{m + 1}
  \quantplus[\alpha]\left(\left[
    \indic{R_{k, m + 1} < \min_{k'} B_{jk'}}\right]_{j = 1}^{n_{m+1}}
  \right),
\end{align*}
where the inequality follows from Lemma~\ref{lemma:quantile-comparison}.
But of course, by the construction~\eqref{eqn:nested-residual-matrix} of the
residual matrix, we have $R_{k,m+1} = S^k_{1-\alpha}$, and $S^k_{1-\alpha} <
\min_{k'} \inf\{\tau \mid Y_j^i \in \what{C}^{-k'}_\tau(X_j^{m+1})\}$
if and only if
\begin{equation*}
  Y_j^{m + 1} \not \in \what{C}_{S^k_{1-\alpha}}^{-k'}(X_j^{m + 1})
  ~~ \mbox{for~any}~ k'
\end{equation*}
by the assumed nesting (and right continuity)
property of the confidence sets $\what{C}_\tau$.
We therefore obtain
\begin{align*}
  &A_{m + 1,\bullet}\\
  & \ge \sum_{k = 1}^{m + 1}
  \quantplus[\alpha]\left(
  \left[\indic{Y_j^{m + 1} \not \in \cup_{i = 1}^m
      \what{C}_{S^k_{1-\alpha}}^{-i}(X_j^{m + 1})}\right]_{j = 1}^{n_{m+1}}
  \right)\\
 & \ge (1 - \delta)(m + 1),
\end{align*}
where the final inequality uses Lemma~\ref{lemma:min-max-comparison}
and that the index sets $J_\bad$ and $K_\bad$ have
cardinalities $|K_\bad| \ge \ceil{(1 - \delta)(m + 1)}$
and $|J_\bad| \ge \floor{\alpha(n_{m + 1} + 1)}$.

On the event~\eqref{eqn:nested-coverage-failure}, the environment $m + 1$ is
thus strange, and so applying
the probability bound~\eqref{eqn:strange-probability}
gives Theorem~\ref{theorem:nested-coverage}.

\section{Proofs for Section~\ref{sec:environment-conditional-coverage}}

\subsection{Proof of Proposition~\ref{proposition:indiscrete-impossibility}}
\label{sec:proof-indiscrete-impossibility}

Let $\gamma \ge 0$ to be specified,
and assume for the sake of contradiction that for some $L < \infty$, there
exists a set $F_L$ with marginal probability $P_X(F_L) \ge \delta$ for which
\begin{equation*}
  P^n(\lebesgue(\what{C}_n(x)) \le L) \ge \gamma ~~ \mbox{for~all~} x \in
  F_L.
\end{equation*}
We will show that if
$\gamma > \alpha + 2 \sqrt{\epsilon}$, we have a contradiction.
Said differently, this implies that for any set $F$ with marginal
probability $P_X(F) \ge \delta$, for each $L < \infty$ there exists some
$x \in F$ for which
\begin{equation*}
  P^n(\lebesgue(\what{C}_n(x)) \le L) \le \alpha + 2 \sqrt{\epsilon},
\end{equation*}
i.e.,
\begin{equation*}
  P^n(\lebesgue(\what{C}_n(x)) = +\infty) \ge 1 - \alpha - 2 \sqrt{\epsilon}.
\end{equation*}

Because by assumption the marginal $P_X$ is $(\epsilon/n,
\delta)$-indiscrete, we can obtain a set $G \subset F_L$ such that $0 <
P_X(G) \le \frac{\epsilon}{n}$.
Now, for a value $\kappa$ to be chosen, define the alternative distribution
$Q$ on $(X, Y) \in \mc{X} \times \R$ by
\begin{equation}
  \label{eqn:alt-y-uniform-dist}
  Y \mid X = x \sim
  \begin{cases} P(Y \in \cdot \mid X = x) & \mbox{if}~
    x \in \mc{X} \setminus G \\
    \uniform[-\kappa, \kappa] &
    \mbox{otherwise},
  \end{cases}
\end{equation}
so that under $Q$, $Y$ is uniform on $[-\kappa, \kappa]$ when
$x \in G$, and let the marginal $Q_X = P_X$.

Observe that
\begin{align*}
  \tvnorm{P - Q} & = \half \int dQ(x) |q(y \mid x) - p(y \mid x)| dy \\
  & \le P(X \in G) \le \frac{\epsilon}{n},
\end{align*}
so that by the standard relationship between Hellinger and
variation distance~\cite[Lemma 2.3]{Tsybakov09}
\begin{equation*}
  \tvnorm{P^n - Q^n} \le \sqrt{2} \sqrt{1 - (1 - \epsilon/n)^n}
  \le \sqrt{2 \epsilon}.
\end{equation*}
As a consequence, we see that for each $x \in G$,
\begin{align*}
  \lefteqn{Q^n(\lebesgue(\what{C}_n(x)) \le L)} \\
  & \ge
  P^n(\lebesgue(\what{C}_n(x)) \le L)
  - \tvnorm{P^n - Q^n}
  \ge \gamma - \sqrt{2 \epsilon}.
\end{align*}
Integrating against $Q_X = P_X$, Fubini's theorem thus implies that
\begin{align*}
  & Q^n(\lambda(\what{C}_n(X_{n+1})) \le L, X_{n + 1} \in G) \\
  & \qquad\qquad \ge \left(\gamma - \sqrt{2 \epsilon}\right) Q(X_{n + 1} \in G).
\end{align*}
Recalling our definition of $Q$, which made $Y$ uniform on
$[-\kappa, \kappa]$ when $X \in G$,
we see that whenever the Lebesgue measure $\lambda(\what{C}_n(X_{n+1}))$ is
finite and $\kappa$ is large enough, the probability
that $Y_{n + 1} \in \what{C}_n(X_{n+1})$ is small, and in particular,
there is some $\kappa < \infty$ such that
\begin{align*}
  & Q^n(Y_{n + 1} \not\in \what{C}_n(X_{n+1}), X_{n + 1} \in G) \\
  & \qquad\qquad\qquad \ge \left(\gamma - 2 \sqrt{\epsilon}\right)
  Q(X_{n + 1} \in G).
\end{align*}
Dividing by $Q(X_{n + 1} \in G)$, we obtain that
\begin{equation*}
  Q(Y_{n + 1} \not \in \what{C}_n(X_{n + 1}) \mid X_{n + 1} \in G)
  \ge \gamma - 2 \sqrt{\epsilon}.
\end{equation*}
But we assumed that $\what{C}_n$ had $(1 - \alpha)$-conditional coverage,
and this is a contradiction:
\begin{equation*}
  \alpha \ge Q(Y_{n + 1} \not \in \what{C}_n(X_{n + 1}) \mid X_{n + 1} \in G)
  \ge \gamma - 2 \sqrt{\epsilon},
\end{equation*}
that is, $\alpha \ge \gamma - 2 \sqrt{\epsilon}$.

\subsection{Proof of Proposition~\ref{proposition:no-environment-conditional}}
\label{sec:proof-no-environment-conditional}

The proof is essentially identical to that of
Proposition~\ref{proposition:indiscrete-impossibility}, with minor
modifications to address environments.
Assume for the sake of contradiction that for some $L < \infty$ and $\gamma
> \alphaenv + 2 \sqrt{\epsilon}$, there exists a set $F_L$ with marginal
probability $P_{E}(F_L) \ge \delta$ for which
\begin{equation*}
  P^m(\lebesgue(\what{C}(x, e)) \le L) \ge \gamma
  ~~ \mbox{for~all~} x \in \mc{X}, e \in F_L.
\end{equation*}
Because by assumption the marginal $P_{E}$ is $(\epsilon/m,
\delta)$-indiscrete, we can obtain a set $G \subset F_L$ such that $0 <
P_{E}(G) \le \frac{\epsilon}{m}$. As in our
definion~\eqref{eqn:alt-y-uniform-dist}, for a value $\kappa$ to be chosen,
define the alternative distribution $Q$ to have marginals $Q_{X,E} = P_{X,
  E}$, and define its $Y$ conditional
\begin{equation}
  \label{eqn:alt-y-env-uniform}
  Y \mid X, E \sim
  \begin{cases} P(Y \in \cdot \mid X, E)
    & \mbox{if~} E \in \mc{E} \setminus G \\
    \uniform[-\kappa, \kappa]
    & \mbox{otherwise}.
  \end{cases}
\end{equation}
Thus, for environments $e \in G$, $Y$ is independent of $X$ and
uniform on $[-\kappa, \kappa]$.

Given this construction, the proof from here is essentially identical to
that of Proposition~\ref{proposition:indiscrete-impossibility}. We
have $\tvnorm{Q - P} \le \frac{\epsilon}{m}$, and so
$\tvnorms{Q^m - P^m} \le \sqrt{2 \epsilon}$. Using that
$\lebesgue(\what{C}(x, e))$ is independent of $x$ by assumption,
we obtain
\begin{align*}
  \lefteqn{Q^m(\lebesgue(\what{C}(x, E_{m + 1})) \le L, E_{m+1} \in G)} \\
  & \qquad \qquad \ge \left(\gamma - \sqrt{2 \epsilon}\right) Q(E_{m + 1} \in G).
\end{align*}
Now we recall our definition~\eqref{eqn:alt-y-env-uniform} of the distribution
$Q$, which leaves $Y$ uniform on $[-\kappa, \kappa]$ whenever $E \in G$.
Then we can take $\kappa < \infty$ large enough that
the $Q$ probability that $Y_j^{m + 1} \in \what{C}(X_j^{m + 1}, E_{m+1})$
conditional on $E_{m + 1} \in G$
is arbitrarily small. Consequently, if we let
$\what{P}_n^E(\cdot)$ denote the empirical measure of $n$
examples $(X, Y)$ drawn i.i.d.\ conditional on $E$, we have
\begin{align*}
  & Q^m\left(\what{P}^{E_{m+1}}_{n_{m+1}}(Y \not\in \what{C}(X, E_{m+1}))
  < 1 - \alpha, E_{m+1} \in G\right) \\
  & \qquad \ge  (\gamma - 2 \sqrt{\epsilon}) Q(E_{m+1} \in G)
\end{align*}
for suitably large $\kappa < \infty$.
Dividing by $Q(E_{m+1} \in G) > 0$, we obtain that
\begin{equation*}
  Q(\what{P}^{E_{m+1}}_{n_{m+1}}(Y \not \in \what{C}(X, E_{m+1}))
  \mid E_{m+1} \in G)
  \ge \gamma - 2 \sqrt{\epsilon}.
\end{equation*}
But we assumed that $\what{C}$ had distribution-free
$(\alphaenv, \alpha)$-environment
conditional coverage,
and this contradicts $\alphaenv < \gamma - 2 \sqrt{\epsilon}$.



\section{Consistency Proofs}

\subsection{Proof of Lemma~\ref{lemma:quantiles-from-bl}}
\label{sec:proof-quantiles-from-bl}

Let $t = \quant_\alpha(Q)$ and $u > 0$ be otherwise arbitrary,
and let $v > 0$ be such that $Q(z \le t - u/2) \le \alpha - v$ and
$Q(Z \le t + u/2) \ge \alpha + v$.
We show that for small enough $\epsilon > 0$, if
$\bl{P - Q} \le \epsilon$, then
\begin{equation*}
  P(Z \le t - u) \le \alpha - \frac{v}{2}
  ~~ \mbox{and} ~~
  P(Z \le t + u) \ge \alpha + \frac{v}{2}.
\end{equation*}
As $\quant_\alpha(P) = \inf\{t' \mid P(Z \le t') \ge \alpha\}$,
we then immediately see that
$t - u \le \quant_\alpha(P) \le t + u$, and as $u$ is otherwise arbitrary,
this proves the lemma.

To see the first claim, let $0 < \delta \le u/2$, and
define the $1/\delta$-Lipschitz continuous and bounded function
\begin{equation*}
  f_\delta(z) \defeq \begin{cases} 1 & \mbox{if}~ z \le t \\
    1 - z/\delta & \mbox{if}~ t \le z \le t + \delta \\
    0 & \mbox{if~} t + \delta \le z,
  \end{cases}
\end{equation*}
which approximates the threshold
$\indic{z \le t}$.
Then we have
\begin{align*}
  & P(Z \le t - u)
  \le P f_\delta(Z + u)
  \stackrel{(\star)}{\le} Q f_\delta(Z + u) + \frac{\epsilon}{\delta}\\
&  \le Q(Z + u \le t + \delta) + \frac{\epsilon}{\delta},
\end{align*}
where inequality~$(\star)$ follows because
$\bl{P - Q} \le \epsilon$.
As $\delta \le u/2$, we have
\begin{equation*}
  Q(Z \le t + \delta - u)
  \le Q(Z \le t - u/2)
  \le \alpha - v,
\end{equation*}
and so we have
\begin{equation*}
  P(Z \le t - u)
  \le \alpha - v + \frac{\epsilon}{\delta}.
\end{equation*}
Any $\epsilon < v \delta/2$ thus guarantees
$P(Z \le t - u) \le \alpha - v/2$.
A completely similar argument gives
$P(Z \le t + u) \ge \alpha + v/2$ for
small enough $\epsilon$.

\subsection{Proof of Example~\ref{example:regression-consistency}}
\label{sec:proof-regression-consistency}

To see how Assumption~\ref{assumption:convergence-of-confidences}
follows, note that
\begin{align*}
 & \left|\tau(x, y, \what{C}^{-i}) - \tau(x, y, C)\right|
  = \left||\what{f}_{-i}(x) - y| - |f(x) - y|\right|\\
  & \le \left|\what{f}_{-i}(x) - f(x)\right|, 
\end{align*}
so
\begin{equation*}
  \sup_y \sup_{x \in \mc{X}_\epsilon} \left|\tau(x, y, \what{C}^{-i})
  - \tau(x, y, C)\right| \cas 0
  \end{equation*}
for any $\epsilon > 0$.
Let $\what{\tau} = \tau(\cdot, \cdot, \what{C}^{-i})$ for shorthand
and $\tau = \tau(\cdot, \cdot, C)$. Then
we claim that
\begin{align}\label{eqn:bl-triangle-inequality}
 & \bl{\law(\what{\tau} \mid \what{P}^i)
    - \law(\tau \mid P^i)}\\ \nonumber
  & \le \bl{\law(\what{\tau} \mid \what{P}^i)
    - \law(\tau \mid \what{P}^i)}
  + \bl{\law(\tau \mid \what{P}^i)
    - \law(\tau \mid P^i)}. 
\end{align}
We consider the two terms in turn.  For the first, let $\eta > 0$ and
consider the event that $|\tau(x, y, \what{C}^{-i}) - \tau(x, y, C)| \le
\eta$ for $x \in \mc{X}_\epsilon$, which occurs eventually (with
probability 1).  Note that for any function $h$ with $\linf{h} \le 1$ and
$\lipnorm{h} \le 1$, we have
\begin{align*}
  &\int \left[h(\what{\tau}(x,y)) - h(\tau(x,y))\right] d\what{P}^i(x,y)\\
  & =
  \int_{\mc{X}_\epsilon}
  \left[h(\what{\tau}(x,y))
    - h(\tau(x,y)) \right] d\what{P}^i(x,y)\\
  &+ \int_{\mc{X}_\epsilon^c}
  \left[h(\what{\tau}) - h(\tau)\right] d\what{P}^i \\
  & \le \what{P}^i(\mc{X}_\epsilon)
  \sup_{x \in \mc{X}_\epsilon, y}|\what{\tau}(x, y) - \tau(x, y)|
  +
  2 \what{P}^i(\mc{X}_\epsilon^c) \\
  & \le \eta + 2 \what{P}^i(\mc{X}_\epsilon^c).
\end{align*}
The final term converges a.s.\ to $P^i(\mc{X}_\epsilon^c) \le
P_X(\mc{X}_\epsilon^c) + \sqrt{P_X(\mc{X}_\epsilon^c)} \maxdivergence \le
\epsilon + \maxdivergence \sqrt{\epsilon}$.  For the second term
in~\eqref{eqn:bl-triangle-inequality}, $\tau$ is fixed and so standard
bounded Lipschitz convergence~\cite{VanDerVaartWe96} guarantees its
a.s.\ convergence to 0.  Then with probability 1, for any $\epsilon > 0$ and
$\eta > 0$, we have
\begin{equation*}
  \limsup_n
  \bl{\law(\what{\tau} \mid \what{P}^i) - \law(\tau \mid P^i)}
  \le \eta + 2 (\epsilon + \maxdivergence \sqrt{\epsilon}),
\end{equation*}
which gives Assumption~\ref{assumption:convergence-of-confidences}.

For Assumption~\ref{assumption:continuity-confidence}, let
$\lambda$ be Lebesgue measure, and recognize that for any
$\tau_0, \tau$, we have
\begin{equation*}
  \what{C}^{-i}_{\tau_0}(x) \setdiff
  C_\tau(x)
  = [\what{f}_{-i}(x) \pm \tau_0] \setdiff
  [f(x) \pm \tau],
\end{equation*}
so $\lambda(\what{C}^{-i}_{\tau_0}(x) \setdiff C_\tau(x)) \le 2 |f(x)
- \what{f}_{-i}(x)| + 2 |\tau - \tau_0|$.  For any $\epsilon > 0$, if
$|\tau\opt - \tau| \le \epsilon/4$, the sets $B_{n,\tau}$ in
Assumption~\ref{assumption:continuity-confidence} satisfy
\begin{align*}
 & B^i_{n,\tau} = \left\{x \mid
  \lambda(\what{C}^{-i}_{\tau}(x) \setdiff
  C_{\tau\opt}(x)) \ge \epsilon\right\}\\
 & \subset
  \left\{x \in \mc{X} \mid |\what{f}_{-i}(x) - f(x)| \ge \epsilon/2\right\}.
\end{align*}
The conditions~\eqref{eqn:locally-uniform-f-convergence}
guarantee that for any large enough $n$
such that $\tau = \tau(n)$ satisfies
$|\tau(n) - \tau\opt| \le \epsilon/4$
and any $m \in \N$,
\begin{align*}
 & P_X\bigg(\bigcup_{i = 1}^m B_{n,\tau}^i\bigg)\\
 & \le P_X\bigg(\bigcup_{i = 1}^m
  \left\{x \mid |\what{f}_{-i}(x) - f(x)| \ge \epsilon/2\right\}
  \bigg) \to 0
\end{align*}
as $n \to \infty$. This then must occur for any slowly
enough growing sequence $m(n)$.

\subsection{Proof of Example~\ref{example:quantile-consistency}}
\label{sec:proof-quantile-consistency}

The argument to justify
Assumption~\ref{assumption:convergence-of-confidences} is similar to that
in Example~\ref{example:regression-consistency}
(see Section~\ref{sec:proof-regression-consistency}): as $C_\tau(x) = [l(x) -
  \tau, u(x) + \tau]$ and $\tau(x, y, C) = \max\{l(x) - y, y - u(x)\}$,
we have
\begin{align*}
  &|\tau(x, y, C) - \tau(x, y, \what{C}^{-i})|\\
  & = \big|
  \max\{l(x) - y, y - u(x)\}\\
  & - \max\{\what{l}_{-i}(x) - y, y - \what{u}_{-i}(x)\}
  \big| \\
  & \le \big||l(x) - y| - |\what{l}_{-i}(x) - y| \big|\\
 & +
  \big||y - u(x)| - |y - \what{u}_{-i}(x)|\big| \\
  & \le \big|\what{l}_{-i}(x) - l(x)\big|
  + \big|\what{u}_{-i}(x) - u(x) \big|.
\end{align*}
In particular, as in Example~\ref{example:regression-consistency},
we have $$\sup_y \sup_{x \in \mc{X}_\epsilon}
|\tau(x, y, C) - \tau(x, y, \what{C}^{-i})| \cas 0, $$
and so the bounded Lipschitz
convergence argument there applies and
Assumption~\ref{assumption:convergence-of-confidences} follows.

To obtain the conditions in Assumption~\ref{assumption:continuity-confidence},
recognize that for Lebesgue measure $\lambda$ an application
of the triangle inequality gives
\begin{align*}
&  \lambda\left(\what{C}^{-i}_{\tau_0}(x)
  \setdiff C_\tau(x)\right)\\
&  \le 2\left(\big|\what{l}_{-i}(x) - l(x)\big|
  + \big|\what{u}_{-i}(x) - u(x)\big|
  + |\tau - \tau_0|\right)
\end{align*}
for any $\tau, \tau_0$. The remainder of the argument is,
\emph{mutatis mutandis}, identical to that in
Example~\ref{example:regression-consistency}.

\subsection{Proof of Example~\ref{example:classification-consistency}}
\label{sec:proof-classification-consistency}

We present an analogous argument to that we use in
Example~\ref{example:regression-consistency},
Section~\ref{sec:proof-regression-consistency} to show how
Assumptions~\ref{assumption:convergence-of-confidences}
and~\ref{assumption:continuity-confidence} follow from
the convergence~\eqref{eqn:locally-uniform-logreg-convergence}.
In this case, the loss $\loss$ is Lipschitz continuous, and
so as
\begin{equation*}
  C_\tau(x) = \{y \in [k] \mid \loss(y, f(x)) \le \tau\}
\end{equation*}
and $\tau(x, y, C) = \loss(y, f(x))$, we have
\begin{equation*}
  \max_{y \in [k]} \sup_{x \in \mc{X}_\epsilon}
  \left|\tau(x, y, \what{C}^{-i}) - \tau(x, y, C)\right|
  \cas 0
\end{equation*}
under the convergence~\eqref{eqn:locally-uniform-logreg-convergence}.
Then exactly as in the proof of
Example~\ref{example:regression-consistency} in
Section~\ref{sec:proof-regression-consistency},
we have
\begin{equation*}
  \bl{\law(\tau(X, Y, \what{C}^{-i}) \mid \what{P}^i) -
    \law(\tau(X, Y, C) \mid P^i)} \cas 0,
\end{equation*}
implying Assumption~\ref{assumption:convergence-of-confidences} holds.

Consider Assumption~\ref{assumption:continuity-confidence}.
Let $\tau\opt = \tau\opt(\delta,\alpha)$ for shorthand and
$x \not \in D_{\tau\opt,\epsilon}$,
so there is no $y$ such that $|\loss(y, f(x)) - \tau\opt| < \epsilon$.
Then
\begin{align*}
  &C_{\tau\opt}(x) = \{y \mid \loss(y, f(x)) \le \tau\opt\}\\
  &= \{y \mid \loss(y, f(x)) \le \tau\opt - \epsilon\}
  = \{y \mid \loss(y, f(x)) \le \tau\opt + \epsilon\}
\end{align*}
by definition of $D_{\tau\opt,\epsilon}$. The
Lipschitz continuity of $v \mapsto \loss(y, v)$ implies there
exists $\eta > 0$ such that if $$\norms{\what{f}^{-i}(x) - f(x)} \le \eta, $$
we have $$|\loss(y, \what{f}^{-i}(x)) - \loss(y, f(x))| \le \epsilon/4,$$
and so
if $|\tau - \tau\opt| < \epsilon/4$ and $\norms{\what{f}^{-i}(x) - f(x)}
\le \eta$, then $\loss(y, f(x)) \le \tau\opt - \epsilon$
implies $\loss(y, \what{f}^{-i}(x)) \le \tau\opt - \epsilon/4
\le \tau - \epsilon/2$, and similarly,
$\loss(y, f(x)) > \tau\opt + \epsilon$ implies that
$\loss(y, \what{f}^{-i}(x)) > \tau\opt + 3\epsilon/4
\ge \tau + \epsilon/2$. That is,
\begin{equation*}
  C_{\tau\opt}(x) = \what{C}^{-i}_{\tau}(x).
\end{equation*}
Recalling the notation of
Assumption~\ref{assumption:continuity-confidence}, the
sets $B_{n,\tau}$ then satisfy
\begin{align*}
 & (B^i_{n,\tau})^c = \left\{x \mid
  \what{C}^{-i}_\tau(x) = C_{\tau\opt}(x) \right\}\\
 & \supset \bigg\{x \in \mc{X} \mid
  \norms{\what{f}_{-i}(x) - f(x)} \le \eta,\\
&  \hspace{6mm} |\tau - \tau\opt| \le \epsilon/4,
  x \not\in D_{\tau\opt,\epsilon}\bigg\}.
\end{align*}

As by assumption the sets $\mc{X}_\epsilon$ on which
$\what{f}^{-i}$ uniformly converges have $P_X(\mc{X}_\epsilon) \le
\epsilon$, the convergence~\eqref{eqn:locally-uniform-logreg-convergence}
yields that if $\tau = \tau(n) \to \tau\opt$, then
for any fixed $m \in \N$,
\begin{equation*}
  \lim_{n \to \infty} P_X\bigg(
  \bigcup_{i = 1}^m B^i_{n,\tau}\bigg) \to 0.
\end{equation*}
Once again, this must occur for any sequence
$m(n)$ growing slowly enough to $\infty$.

\subsection{Proof of Lemma~\ref{lemma:consistency-of-delta-quantile}}
\label{sec:proof-consistency-of-delta-quantile}

We show the argument in a few steps, first showing
that $S^i_{1-\alpha}$ and $\quant_{1-\alpha}(P^i)$ are quite close,
then using Assumption~\ref{assumption:quantiles-by-environments}
to show that the $1 - \delta$ quantile of $\quant_{1-\alpha}(P^i)$
converges.

First, we leverage Lemma~\ref{lemma:quantiles-from-bl}. Recalling
that in Algorithm~\ref{alg:nested-multi-env},
the residuals $R_j^i = \tau(X_j^i, Y_j^i, \what{C}^{-i})
= \inf\{\tau \mid Y_j^i \in \what{C}^{-i}_\tau(X_j^i)\}$,
we see that the empirical distribution
$\what{P}_R^i = \frac{1}{n_i} \sum_{j = 1}^{n_i} 1_{R^i_j}$
satisfies
\begin{equation*}
  \bl{\what{P}_R^i - \law(\tau(X, Y, C) \mid P^i)} \cas 0
\end{equation*}
by Assumption~\ref{assumption:convergence-of-confidences}.
In particular, the assumption that
$\tau(X, Y, C)$ has a density under $(X, Y) \sim P^i$ in a neighborhood
of
\begin{equation*}
  \quant_{1 - \alpha}(P^i)
  \defeq \inf \{\tau \mid P^i(\tau(X, Y, C)
  \le \tau) \ge 1 - \alpha\}
\end{equation*}
then guarantees, via Lemma~\ref{lemma:quantiles-from-bl} and the continuous
mapping theorem, that as $n_i \to \infty$ the quantile $S^i_{1 - \alpha} =
\quantplus[\alpha](\{R_j^i\}_{j=1}^{n_i})$ satisfies
\begin{equation*}
  S^i_{1 - \alpha} - \quant_{1 - \alpha}(P^i)
  \cas 0.
\end{equation*}

As an immediate consequence, we obtain
\begin{equation*}
  \max_{i \le m} |\quant_{1 - \alpha}(P^i) - S^i_{1 - \alpha}|
  \cas 0
\end{equation*}
for any fixed $m$, and hence a sequence $m(n)$ growing
slowly enough as $m(n) \to \infty$
as $n \to \infty$.
From this convergence,
an application of the triangle inequality
gives that if $\law(\quant_{1-\alpha}(P^\env))$ denotes the
induced probability law over $\quant_{1-\alpha}(P^\env)$ by
sampling $\env \in \environments$ and
$\law(\{S^i_{1-\alpha}\}_{i=1}^m)$ denotes the empirical
law of the $S^i$, then
\begin{equation*}
  \bl{\law(\quant_{1 - \alpha}(P^\env))
    - \law(\{S^i_{1-\alpha}\}_{i=1}^{m(n)})} \cas 0
\end{equation*}
as $n \to \infty$. Combining
Assumption~\ref{assumption:quantiles-by-environments}
and Lemma~\ref{lemma:quantiles-from-bl}
yields that
\begin{equation*}
  \quantplus[\delta]\left(\{S^i_{1-\alpha}\}_{i=1}^m\right)
  \cas q(\delta),
\end{equation*}
where $q(\delta)$ in the assumption is the unique
$1 - \delta$ quantile of $\quant_{1-\alpha}(P^\env)$
over random $\env$.

\section{Performance of Multi-environment Jackknife+ Quantile}\label{app:jkquantile}

Algorithm~\ref{alg:multi-env-quantile} presents a less conservative
jackknife algorithm than Algorithm~\ref{alg:multi-env}.
We apply Algorithms \ref{alg:multi-env},  \ref{alg:multi-env-split}, and \ref{alg:multi-env-quantile} to the neurochemical sensing data introduced in Section \ref{exp:neurochemical_sensing}. We set $\alpha = 0.05$, and the split ratio to be 0.5 for Algorithm \ref{alg:multi-env-split}. During each experiment, we vary the values of $\delta$, and record the empirical $1-\alpha$, $1- \delta$, and set length. We repeat the experiment 100 times, and display the results in Figure \ref{fig:vary_delta_with_quantile}.
Although Algorithm~\ref{alg:multi-env-quantile} outputs less conservative confidence intervals than the other two, it fails to provide valid coverage except when the input $1 - \delta$ is large.

\begin{center}
  \algbox{
    \label{alg:multi-env-quantile}
    \textbf{Multi-environment Jackknife+ Quantile:} the regression case
  }{
    \textbf{Input:} samples $\{X^i_j, Y^i_j\}_{j = 1}^{n_i}$,
    $i = 1, \ldots, m$, confidence levels $\alpha, \delta$ \\
    \textbf{For} $i = 1, \ldots, m$,

    \hspace{1em} \textbf{set}
    $\what{f}_{-i}$ to the leave-one-out predictor~\eqref{eqn:loo-predictor}

    \hspace{1em} \textbf{construct} residuals
    \begin{equation*}
      R^i_j = |Y_j^i - \what{f}_{-i}(X_j^i)|,
      ~~ j = 1, \ldots, n_i,
    \end{equation*}
    \hspace{1em} and quantiles
    \begin{equation*}
      S^i_{1 - \alpha}
      = \widehat{q}_{n_i, \alpha}^{+} \left(R_1^i, R_2^i,
      \ldots, R_{n_i}^i\right)
    \end{equation*}
    \textbf{Return} confidence interval mapping
    \begin{align*}
      \what{C}_{m, \alpha, \delta}(x) :=
      \Big[& \widehat{q}_{m, \delta}^{-}\left(\what{f}_{-i}(x) - 
        \{S^i_{1-\alpha}\}_{i=1}^m\right),\\
       & \widehat{q}_{m, \delta}^{+}\left(\what{f}_{-i}(x) + 
        \{S^i_{1-\alpha}\}_{i=1}^m\right)
        \Big].
    \end{align*}
  }
\end{center}

\begin{figure}
\centering
\includegraphics[width=9cm]{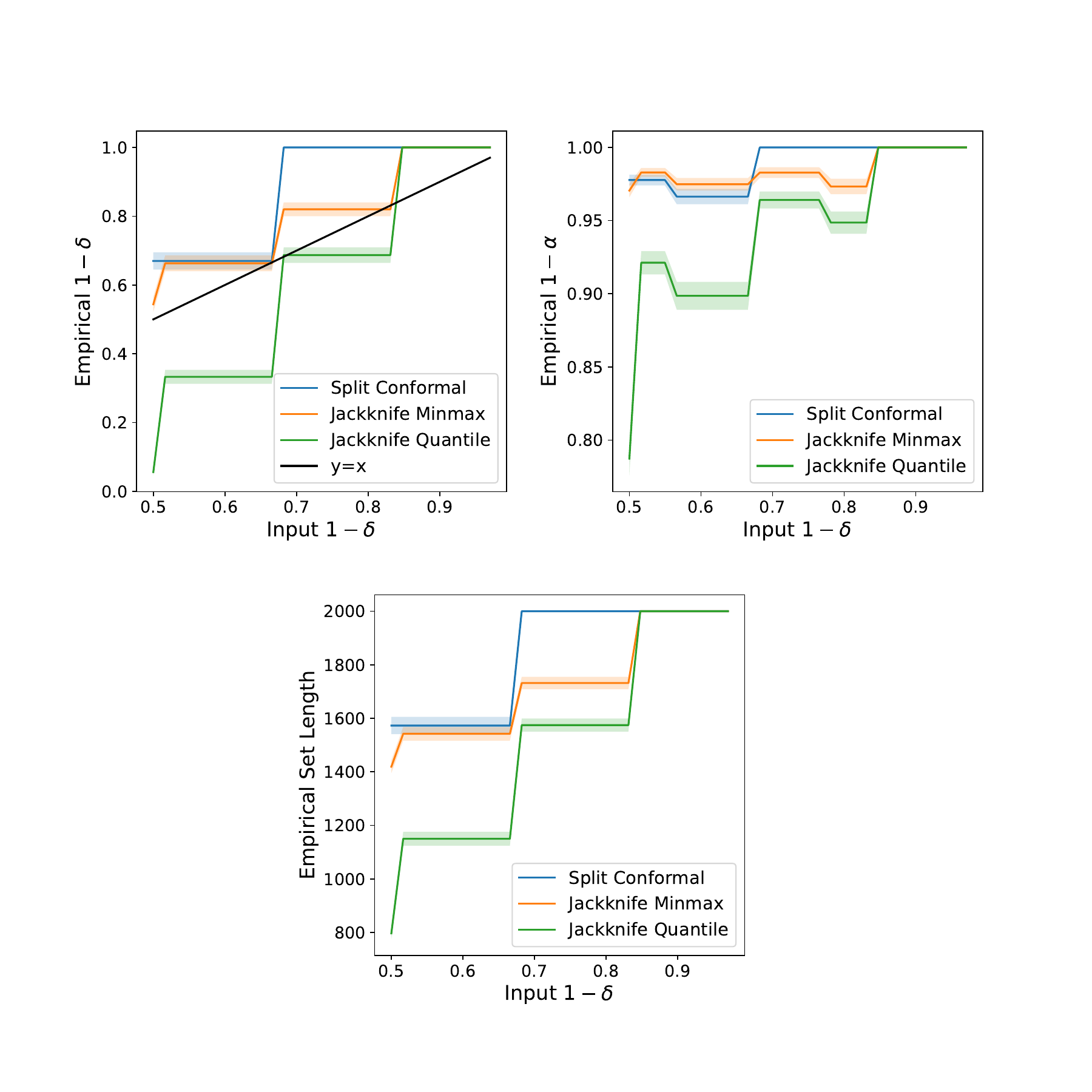}
\vspace{-1cm}
\caption{\label{fig:vary_delta_with_quantile} Influence of input $\delta$ on the performance of Algorithms \ref{alg:multi-env},  \ref{alg:multi-env-split}, and \ref{alg:multi-env-quantile} applied to the neurochemical sensing data,
  with $\alpha = .05$. The plots show the empirical $1-\delta$, empirical $1-\alpha$, and empirical set length for both the split conformal and jackknife-minmax algorithms with various input $\delta$.}
\end{figure}

\end{appendix}

\begin{acks}[Acknowledgments]
This work began when SG was a student at Stanford. JCD and SG were partially
supported by the National Science Foundation grant RI-2006777 and Office of
Naval Research grant N00014-22-1-2669. PS and KJ were partially funded by the National Science Foundation grant 
DMS-2113426. PS was also partially funded by the Eric and Wendy Schmidt Fund for Strategic Innovation. 
\end{acks}

\bibliographystyle{imsart-number} 
\bibliography{bib}

\end{document}